\documentclass[10pt]{article} 
\usepackage[preprint]{tmlr}

\usepackage{amsmath,amsfonts,bm}

\def\eqref#1{equation~\ref{#1}}

\def\1{\bm{1}}

\DeclareMathAlphabet{\mathsfit}{\encodingdefault}{\sfdefault}{m}{sl}
\SetMathAlphabet{\mathsfit}{bold}{\encodingdefault}{\sfdefault}{bx}{n}

\newcommand{\E}{\mathbb{E}}

\usepackage[utf8]{inputenc} 
\usepackage[T1]{fontenc}    
\usepackage{booktabs}       
\usepackage{amsfonts}       
\usepackage{nicefrac}       
\usepackage{xcolor}         
\usepackage{graphicx}
\usepackage{todonotes}
\usepackage{subcaption}
\usepackage{algorithm}
\usepackage{algpseudocode}
\usepackage{amsmath,amssymb}
\usepackage{amsthm}
\usepackage{soul}
\usepackage{multirow}
\usepackage{colortbl}
\usepackage{wrapfig}
\usepackage{placeins}
\usepackage{longtable}
\usepackage{etoc}
\usepackage{geometry}
\usepackage{array}
\usepackage{hyperref}
\usepackage{url}
\usepackage[nameinlink,capitalize,noabbrev]{cleveref}

\renewcommand{\hl}[1]{#1}

\theoremstyle{plain}
\newtheorem{theorem}{Theorem}[section]
\newtheorem{lemma}[theorem]{Lemma}
\newtheorem{proposition}[theorem]{Proposition}
\newtheorem{corollary}[theorem]{Corollary}
\newtheorem{assumption}[theorem]{Assumption}

\theoremstyle{definition}
\newtheorem{definition}[theorem]{Definition}

\theoremstyle{remark}
\newtheorem{remark}[theorem]{Remark}

\crefname{theorem}{Theorem}{Theorems}
\Crefname{theorem}{Theorem}{Theorems}
\crefname{lemma}{Lemma}{Lemmas}
\Crefname{lemma}{Lemma}{Lemmas}
\crefname{proposition}{Proposition}{Propositions}
\Crefname{proposition}{Proposition}{Propositions}
\crefname{corollary}{Corollary}{Corollaries}
\Crefname{corollary}{Corollary}{Corollaries}
\crefname{assumption}{Assumption}{Assumptions}
\Crefname{assumption}{Assumption}{Assumptions}

\newcommand{\cH}{\mathcal{H}}
\newcommand{\cS}{\mathcal{S}}
\newcommand{\cT}{\mathcal{T}}

\DeclareMathOperator{\tr}{tr}

\newcommand{\ip}[2]{\left\langle #1, #2 \right\rangle}
\newcommand{\norm}[1]{\left\lVert #1 \right\rVert}
\newcommand{\opnorm}[1]{\left\lVert #1 \right\rVert_{\mathrm{op}}}

\title{The Universal Weight Subspace Hypothesis}

\author{\name Prakhar Kaushik\thanks{Corresponding author: \texttt{prakhark2@gmail.com}} \email pkaushi1@jhu.edu \\
      \addr Department of Computer Science\\
      Johns Hopkins University, Baltimore, MD, USA
      \AND
      \name Shravan Chaudhari\thanks{Equal contribution.} \email schaud35@jhu.edu \\
      \addr Department of Computer Science\\
      Johns Hopkins University, Baltimore, MD, USA
      \AND
      \name Ankit Vaidya\footnotemark[2] \email avaidya7@jhu.edu \\
      \addr Department of Computer Science\\
      Johns Hopkins University, Baltimore, MD, USA
      \AND
      \name Rama Chellappa \email rchella4@jhu.edu \\
      \addr Department of Computer Science\\
      Johns Hopkins University, Baltimore, MD, USA
      \AND
      \name Alan Yuille \email ayuille1@jhu.edu \\
      \addr Department of Computer Science\\
      Johns Hopkins University, Baltimore, MD, USA}

\def\month{MM}  
\def\year{YYYY} 
\def\openreview{\url{https://openreview.net/forum?id=XXXX}} 

\begin{document}
\etocdepthtag.toc{main}
\maketitle
\maketitle

\begingroup
\renewcommand\thefootnote{}
\footnotetext{\href{https://toshi2k2.github.io/unisub/}{Project page}}
\addtocounter{footnote}{-1}
\endgroup
\begin{abstract}
We show that many deep neural networks trained across diverse tasks exhibit remarkably similar low-dimensional parametric subspaces.
We provide the first large-scale empirical evidence that demonstrates that many neural networks systematically converge to shared spectral subspaces regardless of initialization, task, or domain. Through mode-wise spectral analysis of over 1200 models - including 500 Mistral-7B LoRAs, 500 ViTs, and 50 LLaMA models - we identify universal subspaces capturing information in just a few principal directions. Via spectral decomposition of the weight matrices of various architectures trained on a wide range of tasks and datasets, we identify sparse, joint subspaces that are consistently exploited within shared architectures across diverse tasks and datasets. Our findings offer new insights into the intrinsic organization of information within deep networks and raise important questions about the possibility of discovering these universal subspaces without the need for extensive data and computational resources. Furthermore, this inherent structure has significant implications for model reusability, multi-task learning, model merging, and the development of training- and inference-efficient algorithms, potentially reducing the carbon footprint of large-scale neural models.
\end{abstract}
\begin{figure}[hbt]
\centering
\includegraphics[width=.95\textwidth]{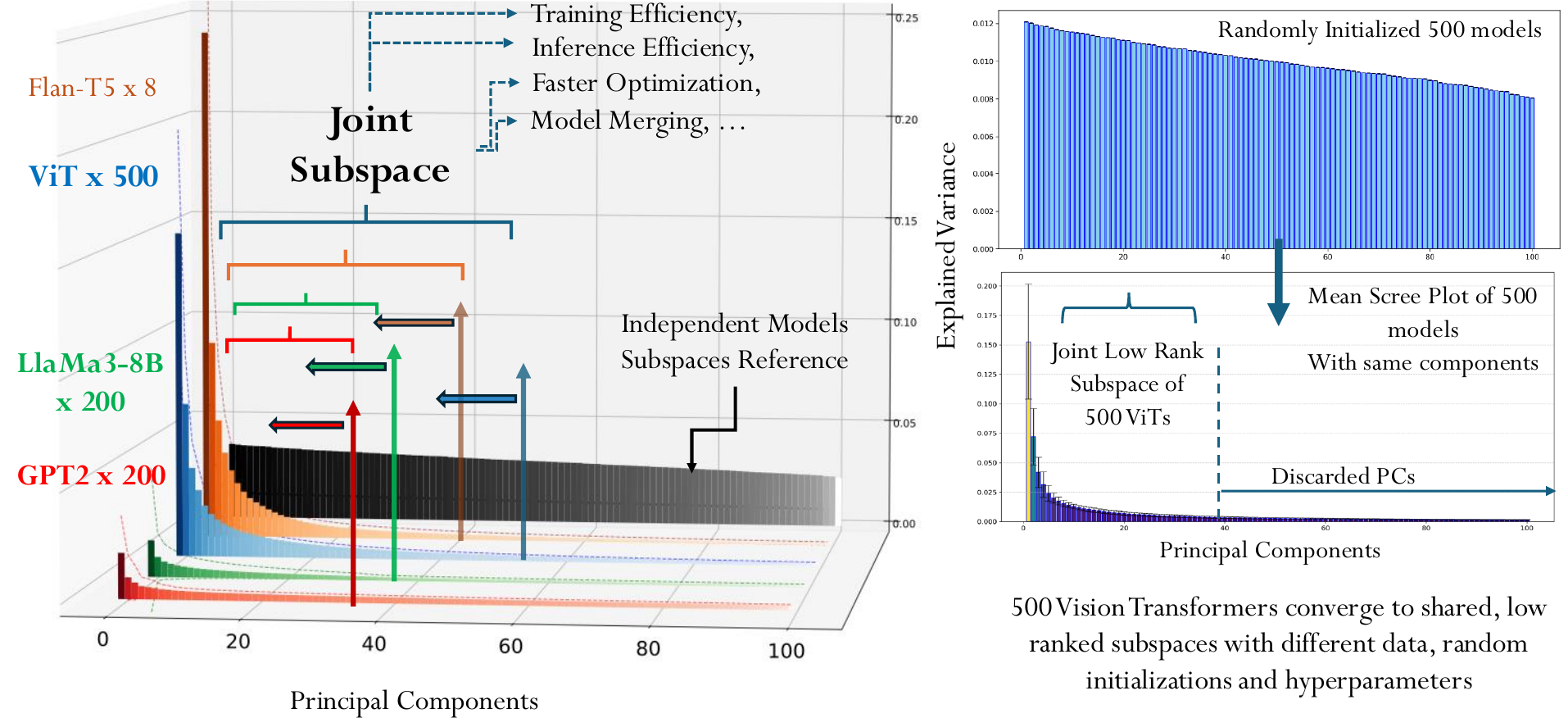}
\caption{ \textbf{Deep Networks Exhibit Shared, Low-Rank (Universal) Subspaces.} Across distinct architectures and modalities, neural networks systematically learn to operate within remarkably similar low-dimensional parameter subspaces. \small{\textbf{Left:} PCA of  $\sim$200 GPT2, 500 ViTs, 50 LLaMA-8B, \& 8 Flan-T5 models reveals consistent sharp spectral decay - strong evidence that a few weight directions capture dominant variance despite vast differences in training data, objectives, and initialization. The black baseline (independent subspaces reference) represents the naive expectation that models would learn distinct directions; our empirical findings contradict this. \textbf{Right:} Strikingly, 500 distinct ViT models exhibit a shared
low-rank subspace across diverse datasets,
training configurations, and domains. This emergent structure unlocks powerful applications: parameter-efficient adaptation, efficient model merging, compressed storage, and accelerated training and inference. Further discussion in \cref{sec:appendix}.}}
\label{fig:teaser}
\vspace{-8pt}
\end{figure}
\section{Introduction}
\label{sec:introduction}
\hl{We show that many backpropagated neural networks trained on a variety of datasets - which could be disjoint and unrelated, diverse hyper-parameter settings, initializations and regularization methods, often learn an architecture-specific, layer-wise similar, low-rank joint subspace (we refer to this as the Universal Subspace)}. We provide the first large-scale empirical analysis - across a diverse set of models - that neural networks tend to converge to these joint subspaces, largely independent of their initialization or the specific data used for training. Our study encompasses different model architectures trained on a variety of datasets, even with different loss functions and tasks. Our spectral subspace analysis of the weights of these models~(Fig.\ref{fig:teaser}) suggests that although individual tasks appear to induce distinct subspaces, individually, they are all part of an unusually low-ranked joint subspace. Our work extends our understanding of what neural networks learn. This universality could explain several puzzling neural properties: why overparameterized models with millions more parameters than training samples still generalize; how different initializations converge to similar representations; and why techniques like weight sharing and parameter-efficient fine-tuning succeed across architectures. If networks indeed learn within shared subspaces, this would provide a supporting explanation for implicit regularization, transferability, and the effectiveness of sparse training methods, while also opening up avenues for applications like merging, new optimization techniques, faster and more efficient learning and inference.

Some prior works have hinted at phenomena consistent with our joint (universal) subspace hypothesis. For example, in the infinite - width limit, training dynamics are governed by a limiting kernel determined by the architecture and initialization distribution~\citep{jacot2018ntk}. Similarly, research into mechanistic interpretability - specifically its own universality hypothesis~\citep{olah2020zoom, chughtai2023toymodeluniversalityreverse} - has uncovered recurring circuits and patterns within specific layers of toy or vision networks, lending indirect support to the concept of universality. Other phenomena, such as the lottery ticket hypothesis~\citep{frankle2019lottery} and mode connectivity~\citep{garipov2018mode}, provide further evidence of structure and redundancy in neural-network parameter spaces. Notably, \cite{Krizhevsky2012ImageNet} observed that the first layer of convolutional networks consistently learns Gabor-like filters on ImageNet. More recently, studies by \citet{guth2024universalityneuralencodingscnns,guth24} have shown initial evidence of recurring eigendirections in certain layers of convolutional neural networks trained on natural images. 
While \cite{Mao_2024} study low-dimensional structure in prediction space, we examine shared subspaces in model weights.
Importantly, in contrast to the abstract or speculative notions of universality presented in earlier works like~\cite{chughtai2023toymodeluniversalityreverse, olah2020zoom} - which focus primarily on representations and are arguably easier to demonstrate due to direct data dependency - ours is the first work to provide rigorous empirical evidence and a clear universal hypothesis at the \textit{neural parameter} level.

In our analysis, we present compelling empirical evidence for the existence of universal subspaces within model architectures across different modalities and tasks. We initially focus on LoRA adapters due to their ease of training and the ability to collect a large number of adapters for diverse tasks, models, and datasets, which enables robust evaluation of our hypothesis. E.g., we demonstrate the emergence of a universal subspace across approximately 500 LoRA adapters for the Mistral-7B~\citep{jiang2023mistral7b} model. We further extend our investigation to the full weight space, where we observe similar universality, extracting sparse, low-rank universal subspaces from about 500 Vision Transformer models and 50 LLaMA3-8B models, each trained on different datasets and initializations. Additional experiments with a trained from scratch ResNet-50, StyleGAN2, and 3D generative models further verify the hypothesis.
Although the underlying causes and broader implications of this universal property remain an open area of investigation, \hl{even an initial understanding of parameter subspace universality has profound implications for neural network efficiency and interpretability. Shared subspaces could enable: (1) massive model compression by storing only subspace coefficients rather than full weights; (2) rapid adaptation to new tasks within learned subspaces; (3) theoretical insights into generalization bounds and optimization landscapes; and (4) environmental benefits through reduced computational requirements for training and inference.}

This geometric universality offers a novel vantage point for investigating fundamental neural properties, including generalization, catastrophic forgetting, and dataset efficiency. Crucially, the pervasive nature of this universal property may underscore the relative primacy of model architecture over other factors in shaping the learned parameter space. We leave open the question of cross-architectural comparison: how do the universal subspaces of distinct architectures differ, and can we explicitly design architectures to optimize the geometry of this subspace? Furthermore, a fundamental question emerges from our results regarding the implications of convergence: if neural networks systematically collapse into the same subspace - thereby inheriting shared biases, capabilities, and failure modes - is this lack of diversity a fundamental bottleneck, and should we develop methods specifically designed to break this convergence?
The remainder of this paper is organized as follows. Section \ref{sec:analysis_method} explains our analysis methodology and Section \ref{sec:results} presents the comprehensive empirical evidence of the Universal subspaces, including their application to new tasks in Section \ref{sexc:newtasks}. Section \ref{sec:notations} then gives a conditional analysis of when a shared subspace can be estimated consistently. Section \ref{sec:experiments} briefly discusses the analysis, providing useful insights and addressing the fundamental questions raised in the introduction. We discuss related work in appendix \ref{sec:related_work} and discuss limitations and scope for future work in Section \ref{sec:limitation}.
Our primary contributions include: (1) We empirically demonstrate the existence of a lower-dimensional shared universal subspace in backpropagated neural networks and give a conditional analysis of when such a subspace can be estimated consistently. (2) We illustrate how to learn an approximate low-dimensional shared subspace from available tasks and introduce a data-free metric for assessing the suitability of available models and layers. (3) We show that the universal subspace enables efficient and faster learning, model scaling, and model compression.

\section{Universal Subspace Analysis}\label{sec:analysis}
This analysis constitutes the core contribution of our work. We demonstrate that architecture-specific, layer-wise \textbf{universal subspaces} consistently emerge across a diverse array of neural models. In our experiments, adherence to the universal subspace hypothesis appears robust, showing no significant deviation regardless of whether models are trained from scratch, fully finetuned, or adapted via low-rank methods. Furthermore, this phenomenon remains invariant across differing initialization strategies, modalities, data formats, and dataset contents. Notably, however, the fidelity of the extracted subspace correlates with the quantity and quality of the available models. This observation leads us to postulate the existence of an "ideal" universal subspace intrinsic to each architecture, towards which individual model instances converge. We hypothesize that superior algorithms, cleaner data, and more effective optimization strategies enable models to approximate this ideal more closely. While we do not formally verify the "ideal universal subspace" hypothesis in the present work, leaving it for future investigation, we suggest that this subspace represents the most stable configuration for contemporary neural networks trained via backpropagation. Consequently, exceptions to this rule may offer fertile ground for further research. 

\textbf{Analysis methodology}\label{sec:analysis_method}
We introduce a data-free heuristic to assess whether the available models support a reliable universal subspace approximation. Following \citet{martin2025setolsemiempiricaltheorydeep}, we use the layer-wise heavy-tail exponent $\alpha$ as a proxy for training quality. Our layer score $q_l$ combines the number of models, the fraction with $2\leq\alpha\leq6$, and the variability of $\alpha$ across models. Geometric means over all and selected layers give the architecture-level scores $q_{\mathrm{all}}$ and $q_{\mathrm{arch}}$, respectively. Higher scores favor larger, better-trained, spectrally consistent collections. These scores guide layer selection and flag potentially unreliable approximations; details are in \cref{sec:metric_apx}.

\hl{Since there is no current method that enables us to compare subspaces of models with different architectures, we focus on large number of models trained on the same architecture.} To this end, we perform analysis using Low rank adapters~\citep{hu2021lora} (LoRA) as well as classical weights of transformer and CNN (Convolutional Neural Network) architectures. For all our experiments, unless stated otherwise, we perform Order 1-2 HOSVD~\citet{lathauwer2000multilinear} only (i.e., $|\mathcal{P}|\in\{1,2\}$ in \cref{alg:trunc-hosvd}), to ensure that our methodology works even in the simplest case. \cref{alg:trunc-hosvd} provides the algorithm we implement. Refer to \cref{sec:analysis_apx} for discussion regarding secondary subspace and how to choose the number of top components. Throughout the paper, \emph{Universal Subspace} refers to the shared linear
subspace of the zero-centered weights. In the original parameter coordinates,
the corresponding reconstruction set is the affine space obtained by adding
the mean $\boldsymbol{\mu}$ of \cref{alg:trunc-hosvd}, broadcast to the weight shape.

\newcommand{\modeprod}{\times}
\newcommand{\nmode}[1]{\mathbin{\modeprod_{#1}}}

\begin{algorithm}
\caption{\hl{Truncated Zero-Centered Shared HOSVD}}
\label{alg:trunc-hosvd}
\begin{algorithmic}[1]

\Require
$M$ aligned model weight tensors $\{W^{(m)}\}_{m=1}^{M}$,
$W^{(m)}\in\mathbb{R}^{I_1\times\cdots\times I_P}$;
set of decomposed parameter modes $\mathcal{P}\subseteq\{1,\ldots,P\}$;
explained-variance threshold $\tau$.

\Ensure
Mean $\boldsymbol{\mu}$; shared factors $U^{(p)}\in\mathbb{R}^{I_p\times\hat r_p}$
(orthonormal columns) for $p\in\mathcal{P}$; model-specific cores
$\{\mathcal{C}^{(m)}\}_{m=1}^{M}$.

\State Stack the models: $\mathcal{X}\in\mathbb{R}^{M\times I_1\times\cdots\times I_P}$.

\State \textbf{Zero-centering:}
       $\boldsymbol{\mu}\gets$ mean of $\mathcal{X}$ over the model axis and all modes $j\notin\mathcal{P}$
       \Comment{feature-wise mean}

\State $X_c^{(m)} \gets W^{(m)}-\boldsymbol{\mu}$ for $m=1,\ldots,M$
       \Comment{broadcast $\boldsymbol{\mu}$ to $I_1\times\cdots\times I_P$}

\For{$p\in\mathcal{P}$}
    \State $X_{c,(p)} \gets$ mode-$p$ unfolding of the centered stacked tensor
           \Comment{$X_{c,(p)}\in\mathbb{R}^{I_p\times M\prod_{j\ne p} I_j}$}

    \State Compute thin SVD:
           $X_{c,(p)} = \tilde U^{(p)} \Sigma^{(p)} \tilde V^{(p)\top}$

    \State $\hat r_p\gets$ smallest rank whose cumulative explained variance is at least $\tau$

    \State $U^{(p)} \gets \tilde U^{(p)}(:,\, 1{:}\hat r_p)$
\EndFor

\For{$m=1$ \textbf{to} $M$}
    \State $\mathcal{C}^{(m)} \gets X_c^{(m)} \nmode{p} U^{(p)\top}$ for all $p\in\mathcal{P}$
           \Comment{project onto the shared factors}
\EndFor

\State \Return
       $\boldsymbol{\mu},\, \{U^{(p)}\}_{p\in\mathcal{P}},\, \{\mathcal{C}^{(m)}\}_{m=1}^{M}$

\end{algorithmic}
\end{algorithm}
A model is reconstructed as $\widehat W^{(m)}=\boldsymbol{\mu}+\mathcal{C}^{(m)}\nmode{p}U^{(p)}$ for all $p\in\mathcal{P}$, where $\nmode{p}$ denotes mode-$p$ tensor--matrix multiplication and modes not in $\mathcal{P}$ are left uncompressed.
The model/task axis is used only to estimate the shared factors and is not itself compressed.
With a single decomposed mode ($|\mathcal{P}|=1$), the procedure is PCA on the mode-$p$ fibers of all models (\cref{sec:lora_appx}).

\section{Conditional Analysis of Shared-Subspace Estimation}
This section gives conditional theoretical motivation for our methodology; it does not prove universality for the networks we study, which is an empirical claim (\cref{sec:results}). Following kernel analyses of deep networks~\citep{ortiz-jimenez2023taskarithmetic, wei2019regularization, chen2021deep, belfer2024spectral, bietti2019kernelperspectiveregularizingdeep}, we model each task's ground-truth predictor $f_t^\star$ as an element of a Hilbert space, e.g., a reproducing kernel Hilbert space (RKHS). We ask when the dominant subspace of the task second-moment operator $\cS$ can be estimated consistently from predictors $\hat f_t$ learned on finitely many tasks with finitely many samples each, and then when this function-space object corresponds to what \cref{alg:trunc-hosvd} computes on weights.
\label{sec:notations}

\paragraph{Setup.}
Let $(\mathcal H,\ip{\cdot}{\cdot})$ be a separable Hilbert space with norm $\norm{\cdot}=\norm{\cdot}_{\mathcal H}$.
For $a,b\in\mathcal H$, the rank-one operator $a\otimes b:\mathcal H\to\mathcal H$ is
$(a\otimes b)g=\ip{b}{g}\,a$; in particular $\opnorm{a\otimes b}=\norm{a}\,\norm{b}$.
Tasks $t\in\{1,\dots,T\}$ are drawn i.i.d. from a distribution $\cT$, and each task dataset $S_t=\{(x_{t,i},y_{t,i})\}_{i=1}^{n_t}$ with $n_t$ samples is drawn independently from $D_t$.
Let $f_t^\star\in\mathcal H$ denote the (unknown) ground-truth predictor for task $t$ and $\hat f_t\in\mathcal H$ be the learned predictor for the task.

\begin{definition}[Task second-moment operator]\label{def:task_operators}
The \emph{population}, \emph{true empirical}, and \emph{learned empirical} task second-moment operators are respectively,
\[
\cS := \E_{t\sim\cT}[\,f_t^\star\otimes f_t^\star\,],\qquad
\hat \cS := \frac{1}{T}\sum_{t=1}^{T} f_t^\star\otimes f_t^\star,\qquad
\tilde \cS := \frac{1}{T}\sum_{t=1}^{T} \hat f_t\otimes \hat f_t .
\]
where $\cS, \hat \cS, \tilde \cS$ are self-adjoint and positive semi-definite such that tr$(\cS)<\infty$.
Its top-$k$ eigenspace $\mathcal H_k^\star$ is the population
rank-$k$ \emph{shared subspace} of tasks. 
\begin{remark}
    \cref{alg:trunc-hosvd} centers the weights; all results below extend to the centered covariance $\cS-m\otimes m$, $m:=\E f_t^\star$, at the same rates (Corollary~\ref{cor:centered}).
\end{remark}

Let $\lambda_1\ge\lambda_2\ge\cdots$ be the eigenvalues of $\cS$ with orthonormal eigenvectors $\{\phi_i\}_{i\ge1}$.
Write $P_k=\sum_{i=1}^k \phi_i\otimes \phi_i$ for the projector onto the population top-$k$ subspace
$\mathcal H_k^\star=\mathrm{span}\{\phi_1,\dots,\phi_k\}$, and let $\tilde P_k$ be the projector onto the
top-$k$ eigenspace of $\tilde \cS$ (the learned shared subspace). Define the eigengap
$\gamma_k:=\lambda_k-\lambda_{k+1}>0$.
\end{definition}

\begin{assumption}[Realizability and bounded second moment]\label{ass:realizability}
For a constant $B>0$ and for all tasks, $f_t^\star\in\mathcal H$ almost surely and $\norm{f_t^\star}\le B$ a.s.
In particular, $\E_{t\sim\cT}\norm{f_t^\star}^2=\tr(\cS)\le B^2<\infty$.
\end{assumption}

\begin{assumption}[Per-task estimation accuracy in $\mathcal H$]\label{ass:pertask}
For each task $t$ and confidence level $\delta_t\in(0,1)$, the learning procedure satisfies
\[
\Pr\!\left(\norm{\hat f_t - f_t^\star}\le \eta_t(n_t,\delta_t)\right)\ \ge\ 1-\delta_t,
\]
where the probability is over the draw of $S_t$. We write $\eta_t:=\eta_t(n_t,\delta_t)$.
\end{assumption}
We keep $\eta_t$ abstract; its rate in $n_t$ depends on the learning procedure (see the remark in \cref{sec:theoretical_analysis}).

The main result combines two ingredients: concentration of the empirical task operator across tasks, and a Davis--Kahan perturbation bound that transfers operator error to subspace error (proofs in \cref{sec:theoretical_analysis}).

\begin{lemma}[Concentration of the empirical second-moment operator]\label{lem:concentration}
Under Assumption~\ref{ass:realizability}, for any $\delta_T\in(0,1)$, with probability at least $1-\delta_T$,
\[
\opnorm{\hat \cS - \cS}\ \le\ \frac{B^2}{\sqrt T}\left(1+\sqrt{2\ln(1/\delta_T)}\right).
\]
\end{lemma}

\begin{lemma}[Davis--Kahan, sin-$\Theta$]\label{lem:dk}
Let $\gamma_k>0$. Then
\[
\opnorm{\tilde P_k - P_k}\ \le\ \frac{2}{\gamma_k}\,\opnorm{\tilde \cS - \cS}.
\]
using definition of $\gamma_k$ from definition \ref{def:task_operators}.
\end{lemma}

\begin{theorem}[Two-level convergence to the shared subspace]\label{thm:twolevel}
Assume \ref{ass:realizability}--\ref{ass:pertask}. There exist absolute constants $c_1, c_2>0$ (e.g., $c_1=3$, $c_2=2$) such that the following holds.
For any $\delta\in(0,1)$, choose $\delta_t=\delta/(2T)$ and set $\delta_T=\delta/2$.
With probability at least $1-\delta$ (over tasks and all per-task samples),
\begin{equation}\label{eq:op-main}
\opnorm{\tilde \cS - \cS}
\ \le\ 
c_1 B^2 \sqrt{\frac{\log(c_2/\delta)}{T}}
\ +\ (2B\bar \eta+\overline{\eta^2})
\end{equation}
If moreover $\gamma_k>0$, then
\begin{equation}\label{eq:subspace-main}
\opnorm{\tilde P_k - P_k}
\ \le\ \frac{2}{\gamma_k}\!
\left(
c_1 B^2 \sqrt{\frac{\log(c_2/\delta)}{T}}
+ (2B\bar \eta+\overline{\eta^2})\right).
\end{equation}
where $\bar \eta = \frac{1}{T}\sum^{T}_{t=1}\eta_t$, $\overline{\eta^2} = \frac{1}{T}\sum^{T}_{t=1}\eta^2_t$, and $\eta_t$ is as in Assumption~\ref{ass:pertask}.
\end{theorem}
\cref{thm:twolevel} (proof in \cref{sec:theoretical_analysis}) separates across-task error, which decays as $T^{-1/2}$, from per-task error ($\bar\eta$, $\overline{\eta^2}$); a larger eigengap $\gamma_k$ stabilizes recovery. The theorem concerns estimation only: it does not imply that the spectral tail $\sum_{i>k}\lambda_i$ is small, i.e., that a low-rank shared subspace exists. That is the empirical question of \cref{sec:results}; Corollary~\ref{cor:risk} keeps the tail as irreducible error.

Since \cref{alg:trunc-hosvd} centers the weights, we also state the analogous guarantee for the centered covariance.

\begin{corollary}[Centered covariance]\label{cor:centered}
Let $m:=\E_{t\sim\cT}f_t^\star$, $\mathcal{C}:=\E_{t\sim\cT}[(f_t^\star-m)\otimes(f_t^\star-m)]=\cS-m\otimes m$, and
$\tilde{\mathcal{C}}:=\frac1T\sum_{t=1}^T(\hat f_t-\bar f)\otimes(\hat f_t-\bar f)$ with $\bar f:=\frac1T\sum_{t=1}^T\hat f_t$.
Under Assumptions~\ref{ass:realizability}--\ref{ass:pertask}, for any $\delta\in(0,1)$ and with $\delta_t=\delta/(2T)$, with probability at least $1-\delta$,
\[
\opnorm{\tilde{\mathcal{C}}-\mathcal{C}}\ \le\ \left(9B^2+3B\bar\eta\right)\sqrt{\frac{\ln(4/\delta)}{T}}\ +\ 4B\bar\eta+2\,\overline{\eta^2}.
\]
Lemma~\ref{lem:dk}, applied to $\mathcal{C}$ with its own eigengap, then bounds the distance between the top-$k$ eigenspaces of $\tilde{\mathcal{C}}$ and $\mathcal{C}$ at the same rate.
\end{corollary}

\textbf{From functions to weights.} The theorem concerns predictors, whereas \cref{alg:trunc-hosvd} decomposes weights. Proposition~\ref{prop:bridge} shows that the two coincide in an idealized regime matching collections fine-tuned from shared pretrained weights $\theta_0$: for models linearized around $\theta_0$, with $\cH$ the RKHS of the tangent kernel at $\theta_0$, the centered covariance of the predictors is unitarily equivalent to the covariance of the function-relevant part of the weights, which is all of it under gradient training from $\theta_0$. PCA of model-wise centered weights then estimates the centered operator of \cref{thm:twolevel}, and centering removes $\theta_0$. \cref{alg:trunc-hosvd} departs from this regime in three ways: it decomposes each mode separately, which captures the top eigenspace of the full weight covariance only under extra structure (e.g., separability across modes); it centers feature-wise; and several of our collections are trained from scratch, without a shared $\theta_0$. We therefore use the analysis as conditional motivation, not as a proof of universality.

\begin{proposition}[Weight-space correspondence in the linearized regime]\label{prop:bridge}
Let all models share reference weights $\theta_0\in\mathbb R^p$ and be linearized around them,
$f_\theta(x)=f_{\theta_0}(x)+\ip{\phi(x)}{\theta-\theta_0}$ with $\phi(x):=\nabla_\theta f_\theta(x)|_{\theta=\theta_0}$.
Let $\cH$ be the RKHS of the tangent kernel $K(x,x'):=\ip{\phi(x)}{\phi(x')}$, let $V:=\mathrm{span}\{\phi(x):x\in\mathcal X\}\subseteq\mathbb R^p$ with orthogonal projector $P_V$, and define $\Phi:\mathbb R^p\to\cH$ by $(\Phi u)(x):=\ip{\phi(x)}{u}$. Then:
\begin{enumerate}
\item[(i)] $\norm{\Phi u}_{\cH}=\norm{P_Vu}$ for all $u\in\mathbb R^p$, and $\Phi$ restricted to $V$ is an isometry onto $\cH$.
\item[(ii)] For weights $\theta_1,\ldots,\theta_T$ with mean $\bar\theta$, and $\bar f:=\frac1T\sum_t f_{\theta_t}$,
\[
\frac1T\sum_{t=1}^T(f_{\theta_t}-\bar f)\otimes(f_{\theta_t}-\bar f)=\Phi\,\hat\Sigma_V\,\Phi^*,\qquad
\hat\Sigma_V:=\frac1T\sum_{t=1}^T P_V(\theta_t-\bar\theta)(\theta_t-\bar\theta)^\top P_V,
\]
and the same identity holds for the population covariances. Hence the two have the same nonzero eigenvalues, and $u\mapsto\Phi u$ maps the corresponding eigenvectors of $\hat\Sigma_V$ onto those of the predictor covariance.
\item[(iii)] If $\theta_t-\theta_0\in V$ for all $t$, then $\hat\Sigma_V=\frac1T\sum_t(\theta_t-\bar\theta)(\theta_t-\bar\theta)^\top$ is the ordinary covariance of the weights. This holds, e.g., for gradient-based training of the linearized model from $\theta_0$ without weight decay.
\end{enumerate}
\end{proposition}

Finally, the learned subspace is near-optimal for representing tasks, up to the irreducible spectral tail.

\begin{definition}[Population projection risk]
For a $k$-dimensional subspace $\mathcal U\subset\mathcal H$ with orthogonal projector $P_{\mathcal U}$, define
\[
\mathcal R(\mathcal U):=\E_{t\sim\cT}\norm{f_t^\star - P_{\mathcal U}f_t^\star}^2 .
\]
\end{definition}

\begin{corollary}[Excess projection risk of the learned subspace]\label{cor:risk}
Let $\tilde{\mathcal H}_k$ denote the range of $\tilde P_k$. Under the event of Theorem~\ref{thm:twolevel},
\[
\mathcal R(\tilde{\mathcal H}_k)
\ \le\ 
\sum_{i>k}\lambda_i
\ +\ \frac{2\,\tr(\cS)}{\gamma_k}\!
\left(
c_1 B^2 \sqrt{\frac{\ln(c_2/\delta)}{T}}
+ 2B\,\bar\eta + \overline{\eta^2}
\right).
\]
\end{corollary}


\textbf{Why do these universal subspaces emerge?} While the precise mechanisms driving this phenomenon remain an open area of investigation, several theoretical factors likely contribute to the emergence of these shared structures. First, spectral analyses of neural tangent kernels show polynomially decaying eigenvalues and a preference for learning low-frequency functions in certain architectures~\citep{belfer2024spectral, bietti2019kernelperspectiveregularizingdeep}. Second, modern architectures impose strong inductive biases that constrain the solution space: convolutional structures inherently favor local, Gabor-like patterns~\citep{Krizhevsky2012ImageNet, guth24}, while mechanistic interpretability studies identify recurring computational circuits~\citep{olah2020zoom, chughtai2023toymodeluniversalityreverse}. Third, the ubiquity of gradient-based optimization - governed by kernels that are largely invariant to task specifics in the infinite-width limit~\citep{jacot2018ntk} - inherently prefers smooth solutions, channeling diverse learning trajectories toward shared geometric manifolds~\citep{garipov2018mode}. If these hypotheses hold, the universal subspace likely captures fundamental computational patterns that transcend specific tasks, potentially explaining the efficacy of transfer learning and why diverse problems often benefit from similar architectural modifications. Under explicit spectral-selection, covariance-concentration, eigengap, and layer-stability assumptions, \cref{thm:weight-subspace-agreement} formalizes one version of this argument for models trained on a common data distribution (\cref{sec:existence_uws}).
\section{Results across Modalities, Tasks, Architectures, and Training Mechanisms}
\label{sec:results}
\hl{We present empirical results using method shown in} Section~\ref{sec:analysis_method},\hl{ extracting our layer wise universal subspace approximations using thousands of publicly available models for most of our experiments. This keeps the additional training cost low.} Spectral analysis relies on efficient spectral decomposition libraries, and can even be run on CPUs. We run all our analysis and experiments on one Nvidia A5000 GPU. The presented large scale empirical results forms the crux of our work and provide strong evidence for the presence of such low ranked joint subspaces across a wide range of task, architecture and modalities. 
In summary, \hl{we present \textbf{ten} sets of analyses and applications spanning image classification, language understanding and generation, text-to-image and 3D generation, model merging, and adaptation to new tasks.}

\begin{figure}
    \centering

    {\scriptsize
    \resizebox{0.8\linewidth}{!}{%
        \begin{tabular}{lcccccc}
            \toprule
            \textbf{Method} &
            \textbf{ImageNet} &
            \textbf{EuroSAT} &
            \textbf{Cifar10} &
            \textbf{Cifar100} &
            \textbf{Pets} &
            \textbf{Avg.} \\
            \midrule
            ResNet50
            & 80.86 & 98.96 & 97.35 & 83.82 & 93.48 & 90.89 \\
            Universal R50
            & 77.89 & 98.83 & 95.89 & 81.49 & 83.81 & 87.58 \\
            \bottomrule
        \end{tabular}%
    }}

    \vspace{3pt}

    \includegraphics[
        width=0.8\linewidth
    ]{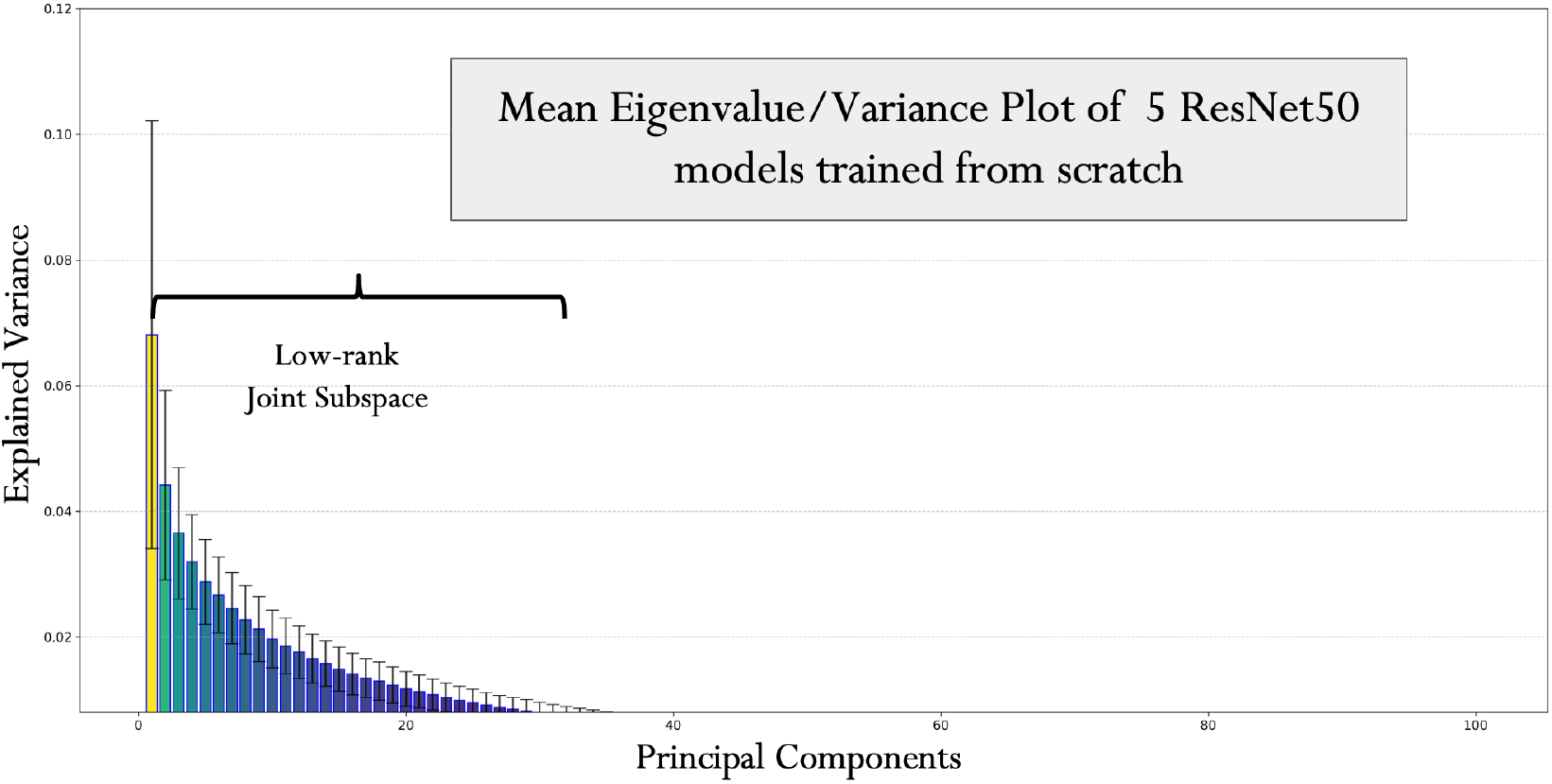}


    \begin{minipage}{0.96\linewidth}
        \scriptsize
        \textit{Note.} The eigenvalue plot summarizes all model weights,
        averaged across all layers of five ResNet50 models. ``Mean'' denotes
        this layer-wise averaging. The vertical axis shows explained variance
        across all models and the horizontal axis shows principal components.
        We use this plotting convention throughout the paper. We refer to the
        low-rank shared subspace as the \emph{Universal} subspace and a model
        constructed from its extracted basis as a \emph{Universal variant}.
    \end{minipage}

    \caption{\footnotesize \textbf{Universal subspaces in CNNs.}
    Decomposing 5 ResNet50 models trained on different tasks reveals a
    low-rank shared subspace, with most information concentrated in top
    fewer directions across layers.}
    \label{fig:short-wrap}
\end{figure}
These experiments cover architectures including CNNs and transformers; vision, language, and multimodal data; discriminative and generative objectives; full model weights and low-rank adapters; and models trained from scratch or finetuned. Together, this diversity provides strong empirical evidence for shared low-rank subspaces within architectures. Further controlled experiments could isolate the effects of individual factors, that remains future work.



\paragraph{Lower-rank joint subspaces in CNNs, LoRA and Finetuned models}
\label{ssec:lora}
In smaller and conventional architectures such as CNNs, evidence for universal structure has been more limited but suggestive. Early work observed that the first convolutional layer often learns Gabor-like filters on ImageNet \citep{Krizhevsky2012ImageNet}. More recently, works report recurring eigendirections in certain CNN layers trained on natural images \citep{guth24,guth2024universalityneuralencodingscnns}.
We extend these observations and examine whether a shared low-rank joint subspace emerges across tasks. Specifically, we use a mix of finetuned and scratch-trained ResNet-50 models for image classification on five disjoint datasets (CIFAR-10, CIFAR-100, ImageNet, Oxford-IIIT Pets, and EuroSAT). The small number of available models leads to a poor approximation of the joint subspace, but we still observe a low-rank structure in the eigenvalue spectrum of the weight matrices (\autoref{fig:short-wrap}). We also find that projecting the original weights into this subspace preserves competitive performance across all five datasets, providing evidence for a shared subspace. 

We further analyze nine ResNet-50 models trained independently from scratch on different datasets. Despite the low suitability score ($q_{\mathrm{all}}=0.0183$), a shared subspace retains 77.91\% mean top-1 accuracy on OOD tasks (compared with 81.00\% for the original models) (\autoref{tab:resnet_scratch}) - it is a significant result for a relatively poor-quality universal space approximation and without any fine-tuning. More details and results are given in \cref{sec:resnet50_apx}.
To test the hypothesis in CNNs beyond image classification, we analyze 20 StyleGAN2~\cite{karras2020analyzingimprovingimagequality} generative models. Despite a low suitability score ($q_{\mathrm{all}}=0.0572$), reconstruction using the universal shared basis preserves comparable image quality (\autoref{fig:stylegan_results}). This holds for both IID and OOD(leave-one-model-out) reconstruction, where the basis is estimated from the other 19 models, without finetuning (\autoref{tab:stylegan_iid_ood}), providing strong evidence for our hypothesis. Further details in \cref{sec:stylegan_apx}.

A recurring limitation of the above experiments is the relatively small number of independently trained models available for testing our hypothesis. To address this, we turn to LoRA~\cite{hu2021lora}, for which large collections of task-specific adapters are publicly available. LoRA updates are particularly suitable for this analysis because they encode task-specific directions and exhibit limited alignment with the pretrained weights~\cite{hu2021lora}, providing a meaningful alternative parameterization for testing whether a shared low-dimensional structure still emerges. 
\begin{figure}
    \centering
    \caption{\small Lots of LoRA model size vs.\ performance.}
    \includegraphics[width=.68\textwidth]{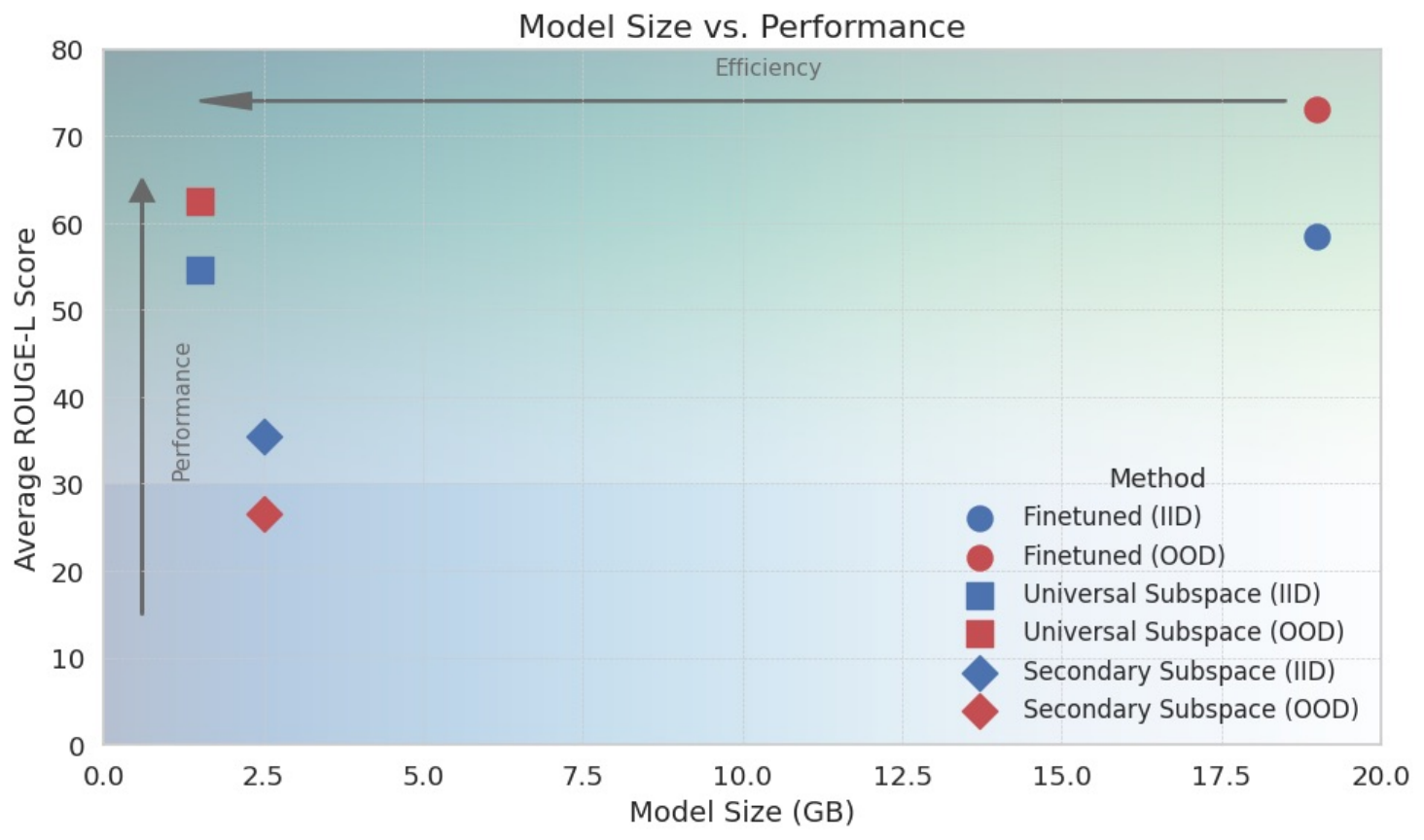}
    
    \label{fig:lola_perf}
\end{figure}
Following Section~\ref{sec:analysis_method}, we spectrally decompose each LoRA submatrix after concatenating the corresponding matrices across all available adapters, and retain the top-$k$ spectral bases. The scale and task diversity of this setting provide a substantially stronger stress test of the Universal Subspace hypothesis.
\begin{figure}[t]
    \centering
    \begin{subfigure}[t]{0.485\linewidth}
        \centering
        \includegraphics[width=\linewidth]{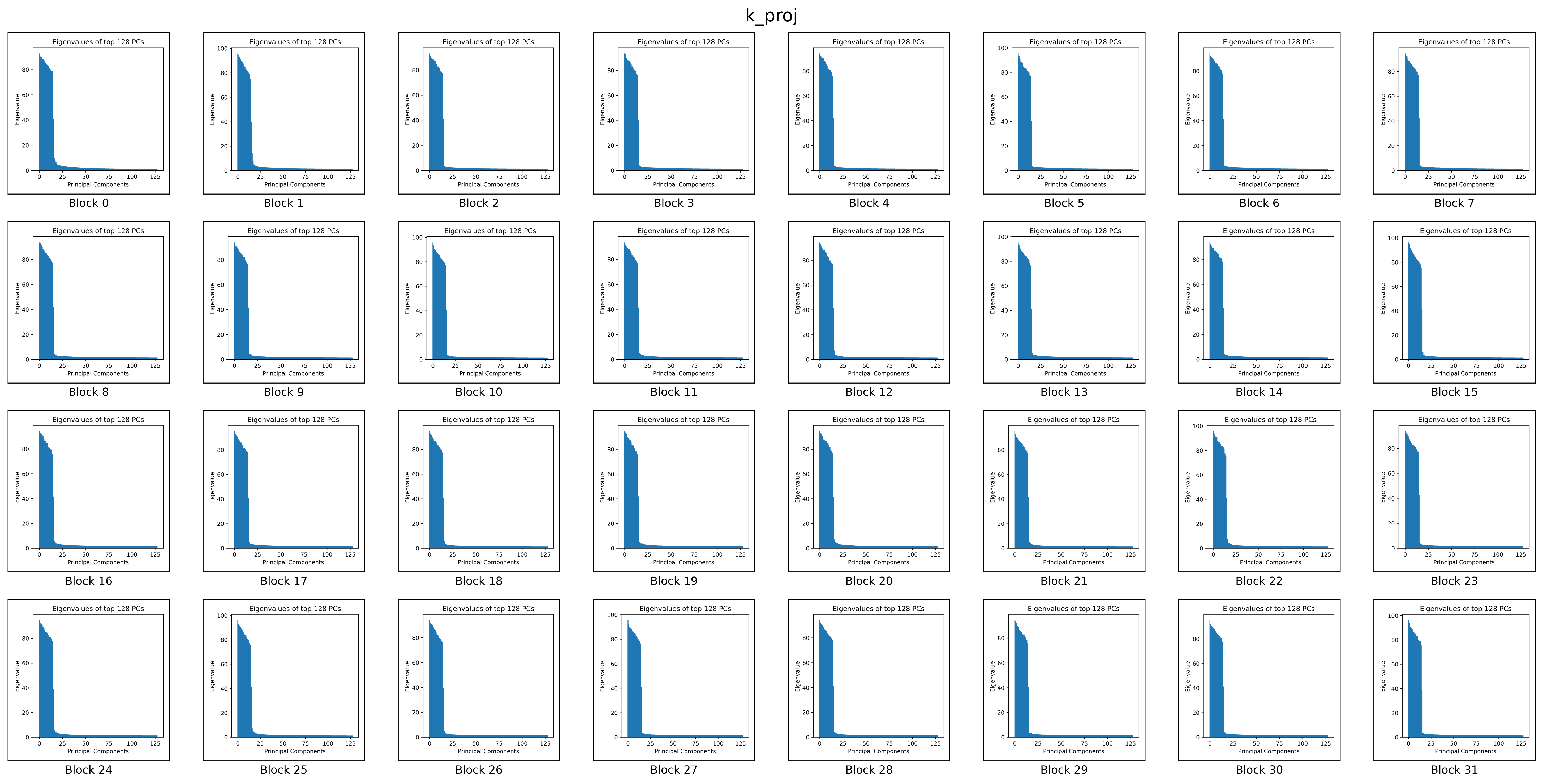}
        \caption{Eigenvalue/variance plot for spectral components across 500 Mistral-7B LoRAs.}
        \label{fig:short-a}
    \end{subfigure}
    \hfill
    \begin{subfigure}[t]{0.485\linewidth}
        \centering
        \includegraphics[width=\linewidth]{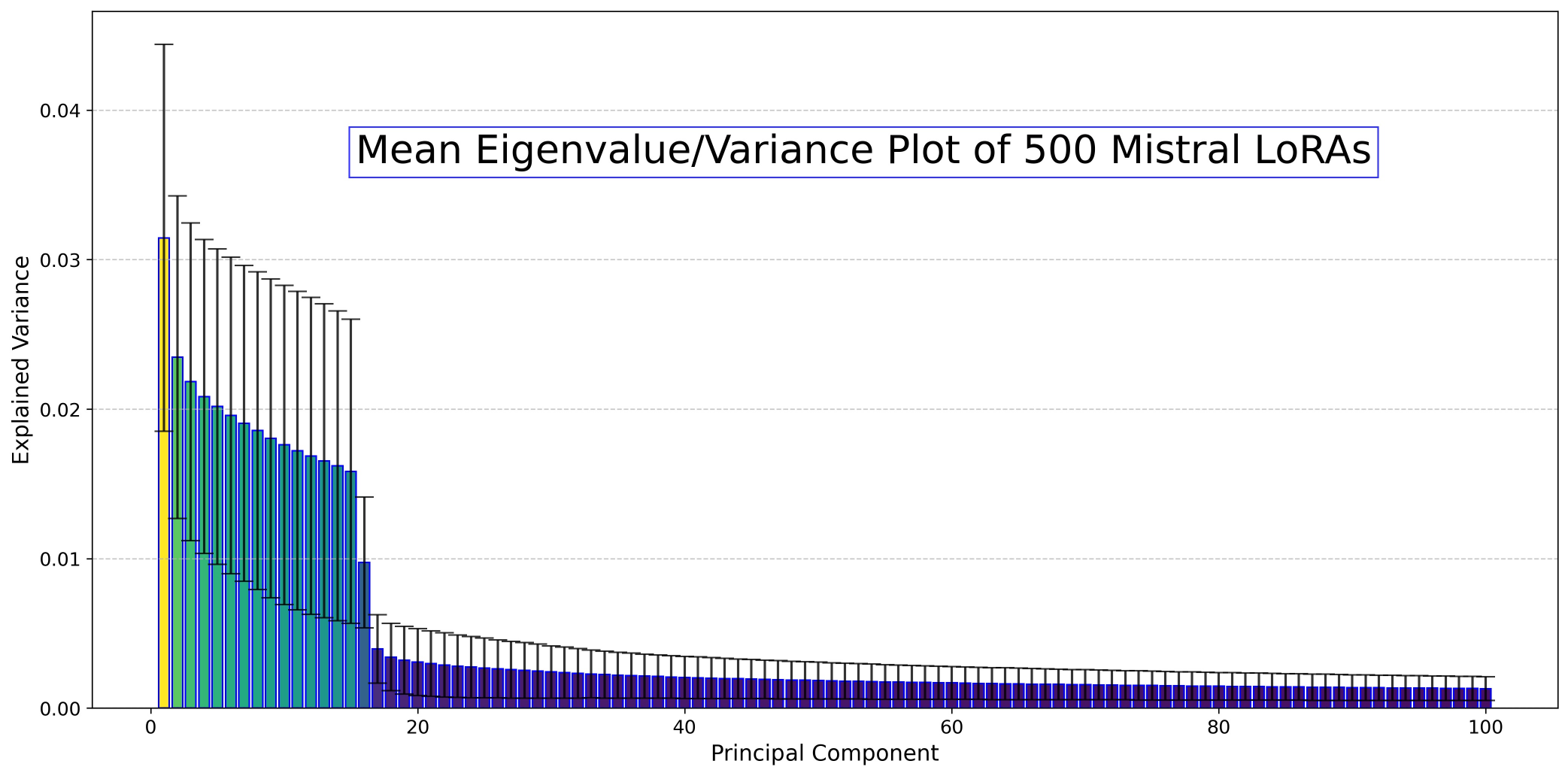}
        \caption{Summarized eigenvalue plot across all 31 layers and 500 LoRAs.}
        \label{fig:short-b}
    \end{subfigure}

    \caption{\textbf{Proving existence of universal subspaces in deep networks.}
    Decomposing 500 LoRAs trained on different tasks using Mistral-7B reveals
    a low-rank universal subspace, with most information concentrated in 16
    or fewer directions across all layers. Additional layer-wise plots are
    provided in \cref{sec:lora_appx}.}
    \label{fig:short}
    \vspace{-10pt}
\end{figure}
In our first experiment, we use LoRA models trained on 500 natural instruction tasks~\citep{wang-etal-2022-super} using Mistral-7B-Instruct-v0.2~\citep{jiang2023mistral7b} as the base~\citep{brüelgabrielsson2024compressserveservingthousands}. \autoref{fig:short} shows the results of our analysis. \autoref{fig:short-a} shows the subspace analysis of individual layers of the Mistral model. Each bar corresponds to the Eigenvalue or explained variance of every unique component (unique for all models, trained in disjoint/different datasets), showing that the parameters of all 500 models can be well approximated by a finite low-rank subspace for all layers. Note that for visualization purposes, the graphs only show the top 100 components, but there exists a significantly larger number of basis. \cref{fig:short-b} gives a summary of this analysis, (summarized) for all of the layers of the 500 Mistral models, further validating the idea of a layer-wise universal subspace.


To test the expressiveness of this universal subspace, we analytically reconstruct the LoRA parameters of randomly chosen seen (IID) and unseen (OOD) tasks by projecting it onto the universal subspace. \autoref{fig:lola_perf} shows the result of our experiment, and as can be seen, that the Universal subspace model perform robustly for both seen and unseen cases. In order to verify the importance of this chosen subspace, we redo the experiment with the \textit{leftover} components from our spectral decomposition (called \textit{Secondary Subspace}). As can be seen from the results, the performance of this model lags drastically behind our chosen subspace. We note that our universal subspace model is \textbf{$19\times$ more memory efficient} since it is no longer necessary to save all 500 LoRAs. 


We extend our analysis to \textbf{text-to-image generation} using Stable Diffusion-XL~\citep{sdxl}. A universal subspace is extracted from publicly available LoRAs on HuggingFace~\citep{von-platen-etal-2022-diffusers}. When projecting individual LoRAs into this subspace, the resulting generations preserve visual quality and style (\autoref{fig:diffusion}). CLIP-based evaluations (\autoref{tab:clip_sdxl}) show that the universal subspace even outperforms individual LoRAs in some cases, possibly due to denoising effects similar to those observed in~\citep{sharma_laser_2023}.

\begin{table}[!h]
\centering
\caption{CLIP scores (higher is better) of images generated using SDXL.}
\resizebox{\textwidth}{!}{%
\begin{tabular}{cccccccccccc}
\toprule
Method & Style 1 & Style 2 & Style 3 & Style 4 & Style 5 & Style 6 & Style 7 & Style 8 & Style 9 & Style 10 & Avg \\
\midrule
LoRA   & 21.95   & 15.59   & 22.18   & 18.84   & 16.65   & 17.99   & 24.66   & 17.47   & 22.07   & 19.93   & 19.73   \\
Universal SDXL LoRA   & 21.96   & 16.07   & 22.07   & 18.79   & 16.68   & 17.99   & 24.66   & 17.56   & 22.46   & 20.09   & \textbf{19.83}   \\
\bottomrule
\end{tabular}
}
\label{tab:clip_sdxl}
\vspace{-5pt}
\end{table}

\begin{figure}[hbt]
  \centering
  \includegraphics[width=1.0\textwidth]{figures/sdxl_lora.pdf}
  \caption{\hl{Text-to-Image Generation Results for Individual models vs. Universal Subspace model. We observe no significant reduction in visual quality despite huge reduction in total model size.}}
  \label{fig:diffusion}
\end{figure}


\hl{In order to test the ability of condensing many models into a single universal subspace, we compare our method with model merging/combination methods in} \autoref{tab:merge_accuracies}. We compare our universal subspace inspired combination approach against six gradient-free baselines: RegMean~\citep{jin2023dataless}, Task Arithmetic (TA)~\citep{ilharco2023editing}, TIES~\citep{yadav2023tiesmerging}, DARE-TIES~\citep{YuDare}, KnOTS-TIES, and KnOTS-DARE-TIES~\citep{knots}. RegMean merges layer weights through a least-squares objective using precomputed layer-input Gram matrices. TA combines task vectors obtained by subtracting the common pretrained weights from each finetuned model. DARE-TIES randomly drops and rescales task-vector entries before applying TIES. KnOTS-TIES and KnOTS-DARE-TIES further apply SVD-based subspace alignment before merging, but still inherit the need for coefficient or pruning hyperparameter selection. In contrast, our universal subspace method, analytically computes the merging coefficients based solely on the geometry of a shared, low-rank universal subspace identified across models, requiring no iterative tuning or validation data-although optional finetuning is possible if data is available. Furthermore, because our subspace is intrinsically low-rank, the merged universal model (plus task specific coefficients) contains significantly fewer parameters than any individual models, offering both computational efficiency and alignment quality. Empirically, our approach achieves higher average accuracy, while reducing parameter count.

In summary, \hl{these experiments provide strong empirical support for our universal subspace hypothesis and demonstrate its practical advantages in terms of memory efficiency, model merging, model reusability, and scalable deployment across diverse tasks and modalities.}
\paragraph{Low rank shared universal subspaces in classical weights}
\begin{figure}[thb]
    \vspace{-10pt}
    \centering

    \includegraphics[
        width=\linewidth
    ]{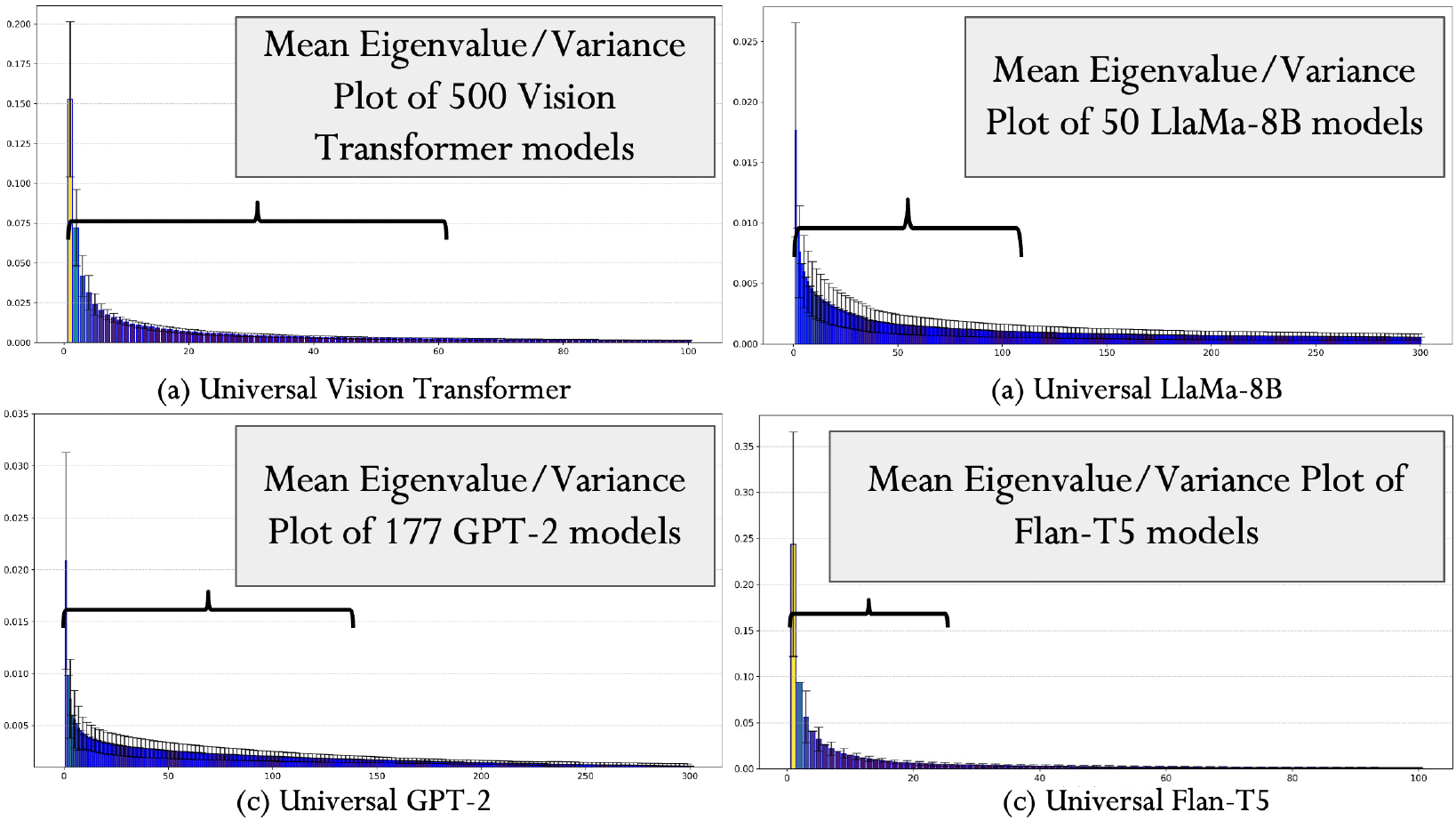}

    \caption{
    \textbf{Universal Subspaces in Classical Weights.}
    Spectral decomposition of weights from
    (1) $\sim$500 Vision Transformers,
    (2) 50 LLaMA-8B models,
    (3) 177 GPT-2 models, and
    (4) GLUE Flan-T5 models
    reveals a consistent low-rank structure, with most variance captured by
    the top few spectral bases. Despite diverse tasks, datasets, and training
    configurations, learned weights consistently align with a shared
    low-dimensional subspace. Only a subset of bases is shown; extended plots
    are provided in \cref{sec:classical_apx}.}
    \label{fig:mainw}

    \vspace{-10pt}
\end{figure}
While aforementioned experiments on CNNs trained from scratch, and LoRAs, provide strong evidence for the presence of the joint subspace, we further rigorously test on large scale MLP and transformer models (500 pretrained ViT, 50 LLaMA3-8B models, 177 GPT-2 and Flan-T5).
%



%
First, we collect $\sim$500 pretrained ViT from HuggingFace, spanning diverse domains - medical imaging, satellite data, and synthetic - and trained with varying losses, optimizers, and initializations. Importantly, we did not filter or curate these models in any way and had no access to the original training data, ensuring our analysis reflects real-world diversity. Details are provided in \cref{sec:classical_apx}. We spectrally decompose (\cref{sec:analysis_method}) all layers (excluding first and last) and observe, in \autoref{fig:mainw}, that the majority of variance is captured by the top few spectral components, revealing a highly compressible, shared subspace across layers. Only the top 100 components are visualized for clarity.

\begin{wraptable}{r}{0.4\columnwidth}
    \centering
    \caption{Image Classification}
    \label{tab:vit_Res}
    \small
    \begin{tabular}{@{}lcc@{}}
        \toprule
        \textbf{Method} & \textbf{IID} & \textbf{OOD} \\
        \midrule
        Full Training & 94.4 \scriptsize{$\pm$ 1.7} & 91.3 \scriptsize{$\pm$ 2.1} \\
        Universal ViT & 94.1 \scriptsize{$\pm$ 2.0} & 87.8 \scriptsize{$\pm$ 1.5} \\
        \bottomrule
    \end{tabular}
    \vspace{-10pt}
\end{wraptable}
To further validate the universality of this subspace, we selected five additional, previously unseen pretrained ViT models for which we had access to evaluation data. These models, considered out-of-domain relative to the initial set, had all their weights reconstructed by projecting onto the identified 16-dimensional universal subspace. We then evaluate their classification accuracy. As shown \autoref{tab:vit_Res}, the
Universal ViT retains performance close to full training, achieving
$94.1\%$ versus $94.4\%$ in IID evaluation and $87.8\%$ versus $91.3\%$
in OOD evaluation.  

\hl{A major outcome of this experiment is that we can replace these 500 ViT models with a single Universal Subspace model.} Ignoring the task-variable first and last layer (weight matrices vary due to different number of categories and input size and formats), \hl{we observe a requirement of \textbf{100$\times$ less memory}, and these savings are prone to increase as the number of trained models increases. We note that we are, to the best of our knowledge, the first work, to be able to \textit{merge} 500 (and theoretically more) ViT into a single universal subspace model. This result implies that hundreds of ViTs can be represented using a single subspace model - excluding task-specific layers - yielding up to \textbf{100$\times$ memory reduction}. To our knowledge, this is the first demonstration of merging over 500 ViTs into a single universal representation.}
We further extend this analysis to 50 finetuned LLaMA3-8B models, 177 GPT-2 models, and Flan-T5 models (trained on GLUE~\cite{glue} datasets) again sourced from HuggingFace without filtering. \hl{As shown in} \autoref{fig:mainw}\hl{, a small number of directions capture dominant structure across models spanning diverse and distinct datasets and tasks}. Further details in \cref{sec:classical_apx}. 
We also test the hypothesis on 3D generation using 20 Geometry Distributions models~\citep{Zhang_2025_ICCV}, each trained from scratch to represent a single shape. These models use a coordinate-denoising MLP and generate point clouds by iterative sampling, extending our analysis to a different architecture, modality, and inference procedure. Despite a low suitability score ($q_{\mathrm{all}}=0.1501$), reconstruction through the shared subspace
preserves the learned geometry without any fine-tuning (\cref{fig:geomdist_reconstructions,tab:geomdist_projection_apx}). Details are in \cref{sec:geomdist_apx}.
\hl{This is, to our knowledge, the first instance of compressing such a large and diverse collection of foundation models into a unified subspace, highlighting its potential for large-scale model reuse and environmental efficiency.}

\subsubsection{Finding universal subspaces and applying them to future tasks}
\label{sexc:newtasks}
\begin{table}
    
    \centering
    \caption{\footnotesize Performance on the GLUE Benchmark.}
    \label{tab:glue_performance}

    \scriptsize

        \begin{tabular}{@{}lcccccccc@{}}
            \toprule
            \textbf{Method} &
            \textbf{Speedup} &
            \textbf{CoLA} &
            \textbf{MRPC} &
            \textbf{RTE} &
            \textbf{QNLI} &
            \textbf{SST-2} &
            \textbf{STS-B} &
            \textbf{Avg.} \\
            \midrule
            LoRA
            & $1\times$ & 59.56 & 86.76 & 77.61 & 92.53 & 94.72 & 90.81 & 83.67 \\
            Universal order-2
            & $2\times$ & 61.82 & 87.25 & 77.62 & 92.71 & 94.15 & 90.48 & 84.01 \\
            Universal order-3
            & $1.8\times$ & 62.06 & 86.52 & 75.81 & 92.98 & 94.26 & 90.39 & 83.67 \\
            \bottomrule
        \end{tabular}%

\end{table}
In this section, the low-rank shared subspaces estimated from a set of available tasks are leveraged to adapt to new, previously unseen tasks. While we do not make theoretical guarantees about reuse on unseen tasks, our experiments show that the approximate shared subspace is empirically reusable across a wide range of practical settings. Concretely, \hl{we reuse the shared principal directions and learn only their task-specific coefficients for the new task. Learning these low-rank coefficients is substantially cheaper than optimizing full-rank weights of size, reducing both computation and memory.} The resulting trainable parameter counts are reported in Table~\ref{tab:vision_models}.
\hl{We find our universal subspace models can have significant impact on the carbon footprint issues of large AI models by making the training, inference and scaling of these models efficient and cheap. As shown in the previous section, we can effectively recycle and replace available pretrained models with a universal subspace model with every individual being represented by a sparse set of coefficients.}
In this section, we show a set of experiments where we utilize the universal subspaces to learn new tasks by freezing the components and simply learning the coefficients using gradient descent. We find that since we are only learning the coefficients, it drastically cuts down the number of parameters required to train the new models. Because only the coefficients are optimized, substantially fewer trainable parameters are required.
        

We present two experiments - Image Classification using ViT-base and Natural Language Understanding using GLUE benchmark~\cite{glue} with RoBERTa-base model. Both involve creating a universal subspace using publicly available models. Details are in \cref{sec:newtask_apx}. For the GLUE benchmark, we follow the same setup as ~\citep{kopiczko_vera_2023} -considering the 6
tasks - CoLA, MRPC, SST-2, QNLI, RTE and STS-B while omitting the time-intensive MNLI and
QQP tasks. We initialize our universal subspace using a leave-one-out-setup, where the subspace is calculated using components of all but one model for which the coefficients are learned. We also compare second- and third-order HOSVD to evaluate how decomposition order affects adaptation performance and training speed for the GLUE experiment. For image classification, we utilize publicly available models to extract our shared subspaces taking care that the data any of these pretrained LoRAs have not seen the data we will be training our coefficients on.
\begin{table}
    \centering
    \caption{\footnotesize Image Classification with Vision Transformer.}
    \label{tab:vision_models}

        \begin{tabular}{lcccccc}
            \toprule
            & \textbf{Params}
            & \textbf{CIFAR100}
            & \textbf{Food101}
            & \textbf{Flowers}
            & \textbf{CIFAR10}
            & \textbf{Pets} \\
            \midrule
            Full Training
            & 86M & 92.8 & 90.7 & 98.82 & 99.0 & 91.2 \\
            Universal ViT
            & 10K & 90.1 & 89.1 & 90.1 & 96.7 & 89.4 \\
            \bottomrule
        \end{tabular}%

\end{table}
\autoref{tab:vision_models} and \autoref{tab:glue_performance} \hl{show that our universal subspace enables significantly more efficient and effective learning since only compact coefficients are trained. The storage required to save all these models is also drastically reduced. The ViTs require 150 GB and LLaMA models require 1.6TB of memory total. Our universal subspace reduces that memory requirement by more than \textbf{100}$\times$.}

\section{Discussion}
\label{sec:experiments}

\hl{This work provides, to the best of our knowledge, the first large-scale, cross-domain analysis showing that many neural networks trained across diverse tasks, modalities, initializations, and hyperparameters consistently exhibit an architecture-specific shared low-rank universal subspace at the layer level.} Concretely, by performing layer-wise spectral decompositions and retaining only the leading principal directions, an accurate approximation of these universal subspaces can be extracted. \hl{Empirically, this behavior emerges broadly: in fully finetuned models and LoRA-based adapters, in models trained from scratch, in both generative and discriminative settings, and in multimodal configurations, as shown in our experiments. Moreover, the approximated subspaces generalize to out-of-distribution tasks, where projecting models and learning only a small set of coefficients suffices to recover strong performance. This enables adapting to new tasks without retraining or storing full weights, and supports robust multi-task learning~\citep{han2026flame}, scalable fine-tuning, and principled model merging within a single unifying framework.} The practical implications are substantial. By reusing a common set of layer-wise principal directions and learning only lightweight coefficients per task, large models can be extended and served with dramatically reduced computational, memory, and engineering overhead. This directly lowers both the financial and environmental costs of training and deployment, aligning with the broader goals of sustainable and accessible AI.

\nocite{cifar100,food101,flowers102,cars,cimpoi2013describingtextureswild,helber2019eurosatnoveldatasetdeep,Stallkamp-IJCNN-2011,lecun2010mnist,Cheng_2017,Xiao:2010,svhn,parkhi12a,schuerholt2024sane,grattafiori2024llama3herdmodels,10.1007/3-540-47969-4_30,1467315,dravid2023rosettaneuronsminingcommon,doi:10.1126/sciadv.adw7697, kaushik2021understandingcatastrophicforgettingremembering, Kaushik_2024_CVPR,Kaushik2024SourceFreeAI, share} 

\bibliography{main}
\bibliographystyle{tmlr}

\clearpage
\appendix
\etocdepthtag.toc{appendix}
\begingroup
\hypersetup{linktoc=all}
\etocsettagdepth{main}{none}
\etocsettagdepth{appendix}{subsubsection}
\etocsettocstyle{\section*{Appendix contents}}{}
\tableofcontents
\endgroup
\clearpage
\section{Appendix A}\label{sec:appendix}
\begin{figure}[ht]
  \centering
  \includegraphics[width=\textwidth]{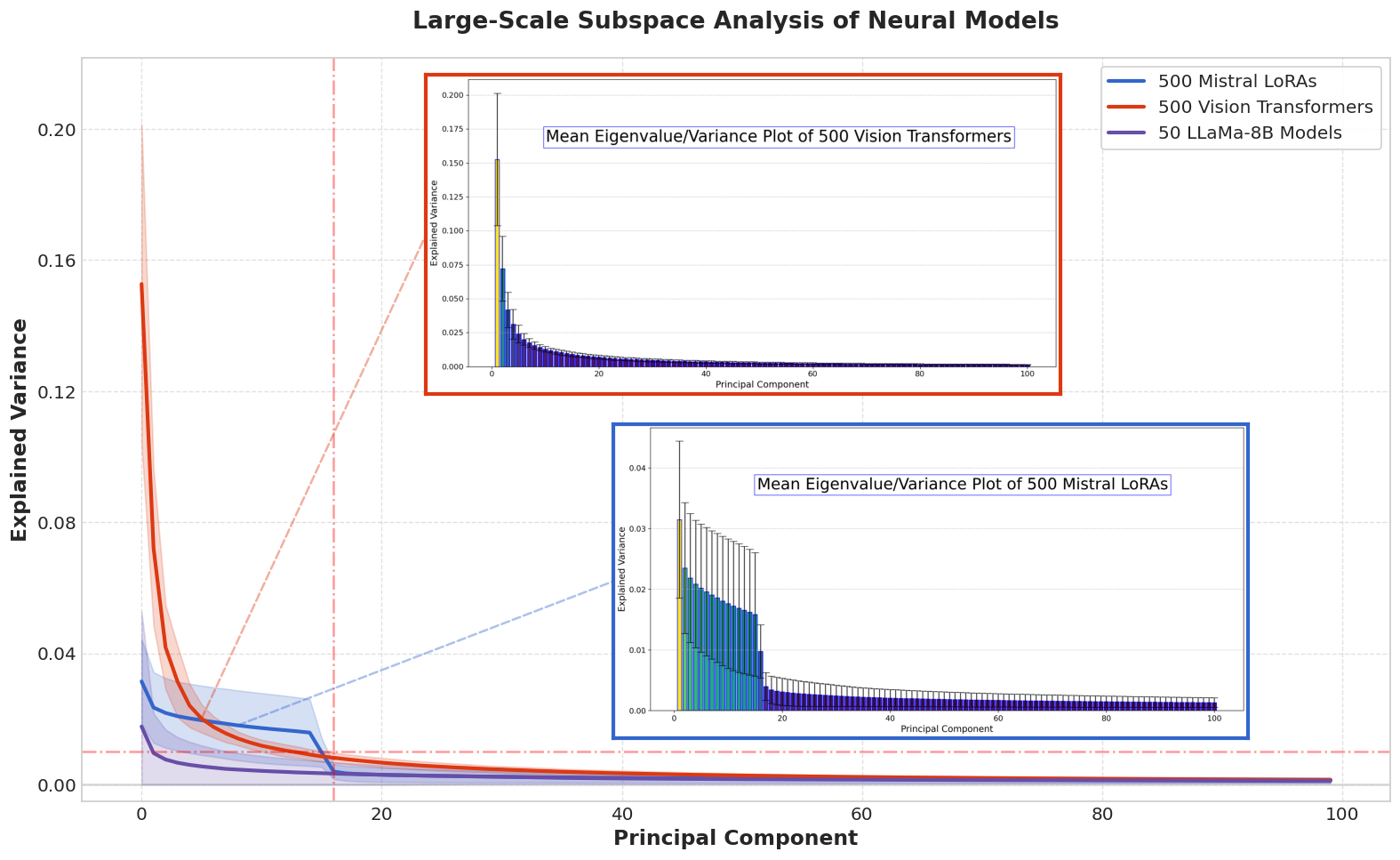}
 \caption{ \textbf{Empirical Evidence for (Universal) Joint Weight Subspaces.}
    This figure illustrates the existence of joint low-dimensional subspaces across models trained on diverse tasks. We plot the average explained variance of the top few principal components of weight matrices from 500 Mistral-7B LoRAs, 500 Vision Transformers, and 50 LLaMA-8B models. Despite differences in modality, data, and training objective, all models exhibit rapid spectral decay - indicating that a small number of directions dominate across layers and settings. This consistent structure provides strong evidence for the presence of joint/universal subspaces, supporting our hypothesis that deep networks systematically reuse a common representational basis. Often, this shared subspace can be seen distinctly. The presence of the subspace has significant implications for deep learning. Not only can large number of models be compressed into a single, lighter Universal model with difference represented as lightweight coefficients, training on future tasks simply becomes tuning those coefficients. Since the basis are fixed, training becomes simpler and quicker. However, this convergence to similar subspace raises few important questions - is it possible to recover the "true" Universal Subspace without learning with huge amounts of data? Is this lack of diversity a bottleneck from current family of deep models?}
  \label{fig:teaserappx}
\end{figure}
\clearpage

\begin{table}[t]
\centering
\caption{Notation reference.}
\begin{tabular}{ll}
\toprule
\textbf{Notation} & \textbf{Description} \\
\midrule
$\mathcal{H}$ & Separable Hilbert space with inner product 
$\langle\cdot,\cdot\rangle$, norm $\|\cdot\|$. \\
$a \otimes b$ & Rank-one operator $g \mapsto \langle b,g\rangle a$, 
$\|a\otimes b\|_{\mathrm{op}} = \|a\|\,\|b\|$. \\
$T$ & Number of tasks. \\
$\mathcal{T}$ & Distribution over tasks. \\
$D_t$ & Data distribution for task $t$. \\
$S_t=\{(x_{t,i},y_{t,i})\}_{i=1}^{n_t}$ & Dataset of size $n_t$ for task $t$. \\
$f_t^\star \in \mathcal{H}$ & Ground-truth predictor for task $t$. \\
$\hat f_t \in \mathcal{H}$ & Learned predictor for task $t$. \\
$B$ & Uniform bound: $\|f_t^\star\|\le B$ almost surely. \\
$\eta_t$ & Per-task error bound: $\Pr(\|\hat f_t-f_t^\star\|\le\eta_t)\ge 1-\delta_t$ (Assumption~\ref{ass:pertask}). \\
$\bar\eta$ & Average error: $\tfrac{1}{T}\sum_{t=1}^T \eta_t$. \\
$\overline{\eta^{2}}$ & Average squared error: $\tfrac{1}{T}\sum_{t=1}^T \eta_t^2$. \\
$\cS$ & Population operator: $\cS=\mathbb{E}_{t\sim\mathcal{T}}[f_t^\star\otimes f_t^\star]$. \\
$\hat \cS$ & Empirical operator (true predictors): 
$\tfrac{1}{T}\sum_{t=1}^T f_t^\star\otimes f_t^\star$. \\
$\tilde \cS$ & Empirical operator (learned predictors): 
$\tfrac{1}{T}\sum_{t=1}^T \hat f_t\otimes \hat f_t$. \\
$\lambda_1 \ge \lambda_2 \ge \dots$ & Eigenvalues of $\cS$. \\
$\phi_i$ & Orthonormal eigenvectors of $\cS$. \\
$P_k$ & Projector onto top-$k$ eigenspace of $\cS$. \\
$\tilde P_k$ & Projector onto top-$k$ eigenspace of $\tilde \cS$. \\
$\gamma_k$ & Eigengap: $\gamma_k := \lambda_k - \lambda_{k+1} > 0$. \\
$\|A\|_{\mathrm{op}}$ & Operator (spectral) norm. \\
$\|A\|_{HS}$ & Hilbert–Schmidt norm. \\
$\delta, \delta_t, \delta_T$ & Failure probabilities (global, per-task, across-task). \\
\bottomrule
\end{tabular}
\end{table}

\subsection{Theoretical Analysis}\label{sec:theoretical_analysis}
We apply a standard generalization bound over the squared error between the task function and its projection onto the shared subspace:
\[
\ell(f_t, x) = \|f_t(x) - f_{t,k}(x)\|^2
\]

To justify the application of PAC-style bounds, we verify that this loss is bounded. We assume that each task predictor \(f_t\) lies in a Reproducing Kernel Hilbert Space (RKHS) with norm bounded by \(B\), i.e., \(\|f_t\|_{\mathcal{H}} \leq B\), and that the projection \(f_{t,k}\) onto the learned shared subspace \(\hat{\mathcal{H}}_k\) also satisfies \(\|f_{t,k}\|_{\mathcal{H}} \leq B\).

Using the reproducing property and assuming a kernel bound \(\kappa^2 = \sup_{x \in \mathcal{X}} \|\phi(x)\|^2\), we have for any \(x\):
\[
\|f_t(x)\| \leq \kappa B \quad \text{and} \quad \|f_{t,k}(x)\| \leq \kappa B
\]

Thus, the pointwise squared loss is bounded as:
\[
\|f_t(x) - f_{t,k}(x)\|^2 \leq (\|f_t(x)\| + \|f_{t,k}(x)\|)^2 \leq (2\kappa B)^2 = 4\kappa^2 B^2
\]

Therefore, the loss function is bounded in \([0, 4\kappa^2 B^2]\), satisfying the conditions required for PAC-style generalization bounds to hold.

\noindent Statements of Lemmas~\ref{lem:concentration}--\ref{lem:dk}, Corollaries~\ref{cor:centered} and~\ref{cor:risk}, and Proposition~\ref{prop:bridge} are given in \cref{sec:notations}; we provide their proofs here.

\begin{proof}[Proof of Lemma~\ref{lem:concentration}]
Let $Y_t:=f_t^\star\otimes f_t^\star$, so that $\hat \cS-\cS=\frac1T\sum_{t=1}^T(Y_t-\cS)$ with $\cS=\E[Y_t]$.
Since $\norm{f_t^\star}\le B$ a.s., $\norm{Y_t}_{HS}=\norm{f_t^\star}^2\le B^2$, so the $Y_t$ and $\cS$ lie in the Hilbert space of Hilbert--Schmidt operators on $\cH$.

\noindent\emph{Expectation.} The $Y_t-\cS$ are i.i.d.\ and mean-zero, so the cross terms vanish and
\[
\E\norm{\hat \cS-\cS}_{HS}^2
=\frac1T\,\E\norm{Y_1-\cS}_{HS}^2
=\frac1T\left(\E\norm{Y_1}_{HS}^2-\norm{\cS}_{HS}^2\right)
\le\frac{B^4}{T}.
\]
By Jensen's inequality, $\E\norm{\hat \cS-\cS}_{HS}\le B^2/\sqrt T$.

\noindent\emph{Bounded differences.} Let $g(f_1^\star,\ldots,f_T^\star):=\norm{\hat \cS-\cS}_{HS}$.
Replacing $f_t^\star$ by any $f_t'$ with $\norm{f_t'}\le B$ changes $\hat \cS$ by $(f_t^\star\otimes f_t^\star-f_t'\otimes f_t')/T$, whose Hilbert--Schmidt norm is at most $2B^2/T$; by the triangle inequality, $g$ changes by at most $2B^2/T$.
McDiarmid's inequality~\citep{mcdiarmid1989method} then gives, for every $\epsilon>0$,
\[
\Pr\big(g-\E g\ge\epsilon\big)
\le\exp\!\left(-\frac{2\epsilon^2}{T\,(2B^2/T)^2}\right)
=\exp\!\left(-\frac{T\epsilon^2}{2B^4}\right).
\]
Setting the right-hand side equal to $\delta_T$, with probability at least $1-\delta_T$,
\[
\norm{\hat \cS-\cS}_{HS}\le\frac{B^2}{\sqrt T}+B^2\sqrt{\frac{2\ln(1/\delta_T)}{T}}.
\]
The claim follows since $\opnorm{A}\le\norm{A}_{HS}$.
\end{proof}

\begin{theorem}[Restating Two-level convergence to the shared subspace theorem]\label{thm:twolevel_app}
Assume \ref{ass:realizability}--\ref{ass:pertask}. There exist absolute constants $c_1, c_2>0$ (e.g., $c_1=3$, $c_2=2$) such that the following holds.
For any $\delta\in(0,1)$, choose $\delta_t=\delta/(2T)$ and set $\delta_T=\delta/2$.
With probability at least $1-\delta$ (over tasks and all per-task samples),
\begin{equation}\label{eq:op-app}
\opnorm{\tilde \cS - \cS}
\ \le\ 
c_1 B^2 \sqrt{\frac{\ln(c_2/\delta)}{T}}
\ +\ (2B\bar \eta+\overline{\eta^2})
\end{equation}
If moreover $\gamma_k>0$, then
\begin{equation}\label{eq:subspace-app}
\opnorm{\tilde P_k - P_k}
\ \le\ \frac{2}{\gamma_k}\!
\left(
c_1 B^2 \sqrt{\frac{\ln(c_2/\delta)}{T}}
+ (2B\bar \eta+\overline{\eta^2})\right).
\end{equation}
where $\bar \eta = \frac{1}{T}\sum^{T}_{t=1}\eta_t$, $\overline{\eta^2} = \frac{1}{T}\sum^{T}_{t=1}\eta^2_t$, and $\eta_t$ is as in Assumption~\ref{ass:pertask}.
\end{theorem}

\begin{proof}[Proof of Theorem~\ref{thm:twolevel}]
\textbf{(i) Triangle split.}
$\opnorm{\tilde \cS - \cS}\le \opnorm{\tilde \cS - \hat \cS} + \opnorm{\hat \cS - \cS}$.

\noindent\textbf{(ii) Within-task term.}
We know that,
\begin{align*}
\opnorm{\hat f_t\otimes \hat f_t - f^\star_t\otimes f^\star_t}
&\le \norm{\hat f_t-f^\star_t}\,(\norm{\hat f_t}+\norm{f^\star_t})\\
&\le \|\hat f_t - f^\star_t\|(\|\hat f_t\|+\|f^\star_t\|)\\
&\le \eta_t\,(2B+\eta_t) \qquad \text{(since $\|\hat f_t\|\le \|f_t^\star\|+\|\hat f_t-f_t^\star\|\le B+\eta_t$)}\\
&= 2B\,\eta_t + \eta_t^2 .
\end{align*}

Averaging and using the triangle inequality for operator norms,
\begin{align*}
\opnorm{\tilde \cS - \hat \cS} &\le 2B\,\bar\eta + \overline{\eta^2}\\
\end{align*}
This holds on the event $\bigcap_{t=1}^T\{\norm{\hat f_t-f_t^\star}\le \eta_t\}$, whose probability is at least $1-\sum_t\delta_t=1-\delta/2$.

\noindent\textbf{(iii) Across-task term.}
Lemma~\ref{lem:concentration} with $\delta_T=\delta/2$ yields, with probability at least $1-\delta/2$,
\[
\opnorm{\hat \cS - \cS}
\le \frac{B^2}{\sqrt T}\left(1+\sqrt{2\ln(2/\delta)}\right)
\le 3B^2\sqrt{\frac{\ln(2/\delta)}{T}},
\]
where the last step uses $\ln(2/\delta)\ge\ln 2$, so that $1+\sqrt{2\ln(2/\delta)}\le\big(1/\sqrt{\ln 2}+\sqrt2\big)\sqrt{\ln(2/\delta)}\le 3\sqrt{\ln(2/\delta)}$.
This is the across-task term of \eqref{eq:op-main} with $c_1=3$ and $c_2=2$.

\noindent\textbf{(iv) Union bound and Davis--Kahan.}
Combining (ii)--(iii) with a union bound gives \eqref{eq:op-main}.
Lemma~\ref{lem:dk} then implies \eqref{eq:subspace-main}.
\end{proof}
Substituting explicit per-task bounds $\eta_t(n_t,\delta_t)$ from Assumption~\ref{ass:pertask} gives the corresponding
finite-sample rate through $\bar{\eta}$ and $\overline{\eta^2}$. In particular,
when the per-task estimation errors vanish, the remaining across-task error
decreases as $O(T^{-1/2})$ up to logarithmic factors.

\begin{proof}[Proof of Corollary~\ref{cor:centered}]
Since $\tilde{\mathcal{C}}=\tilde\cS-\bar f\otimes\bar f$ and $\mathcal{C}=\cS-m\otimes m$,
\[
\opnorm{\tilde{\mathcal{C}}-\mathcal{C}}\le\opnorm{\tilde\cS-\hat\cS}+\opnorm{\hat\cS-\cS}+\opnorm{\bar f\otimes\bar f-m\otimes m}.
\]
On the per-task event of step (ii) in the proof of Theorem~\ref{thm:twolevel} (probability at least $1-\delta/2$), the first term is at most $2B\bar\eta+\overline{\eta^2}$.
Let $r:=\sqrt{\ln(4/\delta)/T}$. By Lemma~\ref{lem:concentration} with $\delta_T=\delta/4$ and the argument of step (iii) (now using $\ln(4/\delta)\ge\ln 4$), the second term is at most $3B^2r$ with probability at least $1-\delta/4$.
For the third term, $\bar f\otimes\bar f-m\otimes m=\bar f\otimes(\bar f-m)+(\bar f-m)\otimes m$, so it is at most $\norm{\bar f-m}(\norm{\bar f}+\norm{m})\le(2B+\bar\eta)\norm{\bar f-m}$, using $\norm{\hat f_t}\le B+\eta_t$ and $\norm{m}\le B$.
Moreover, $\norm{\bar f-m}\le\bar\eta+\norm{\bar f^\star-m}$ with $\bar f^\star:=\frac1T\sum_t f_t^\star$. The proof of Lemma~\ref{lem:concentration}, applied to the vectors $f_t^\star\in\cH$ instead of the operators $Y_t$ (now $\E\norm{\bar f^\star-m}^2\le B^2/T$ and the bounded differences are $2B/T$), gives $\norm{\bar f^\star-m}\le 3Br$ with probability at least $1-\delta/4$.
Collecting terms, using $\bar\eta^2\le\overline{\eta^2}$, and taking a union bound over the three events gives the claim.
\end{proof}

\begin{proof}[Proof of Proposition~\ref{prop:bridge}]
(i) $\Phi u=0$ iff $u\perp\phi(x)$ for all $x$, so $\ker\Phi=V^\perp$ and $\Phi u=\Phi u'$ iff $P_Vu=P_Vu'$. Hence $\ip{\Phi u}{\Phi v}':=\ip{P_Vu}{P_Vv}$ is a well-defined inner product on $\Phi(\mathbb R^p)$, under which $\Phi|_V$ is an isometry onto $\Phi(\mathbb R^p)$. Since $K(\cdot,x)=\Phi\phi(x)$ and $\phi(x)\in V$, we have $\ip{\Phi u}{K(\cdot,x)}'=\ip{P_Vu}{\phi(x)}=\ip{u}{\phi(x)}=(\Phi u)(x)$. Thus $K$ reproduces this space, which is therefore $\cH$ by uniqueness of the RKHS of $K$.

(ii) The reference terms cancel: $f_{\theta_t}-\bar f=\Phi(\theta_t-\bar\theta)=\Phi u_t$ with $u_t:=P_V(\theta_t-\bar\theta)$. For $g\in\cH$, $(\Phi u_t\otimes\Phi u_t)g=\ip{\Phi u_t}{g}_{\cH}\,\Phi u_t=\Phi\,u_tu_t^\top\Phi^*g$; averaging over $t$ gives the identity, and the population case is identical with expectations. By (i), $\Phi|_V$ is unitary from $V$ onto $\cH$, $\Phi^*$ maps $\cH$ into $V$ and inverts $\Phi|_V$, and $\hat\Sigma_V$ vanishes on $V^\perp$. Hence $\Phi\hat\Sigma_V\Phi^*$ is unitarily similar to $\hat\Sigma_V|_V$, which gives the spectral statement.

(iii) If $\theta_t-\theta_0\in V$ for all $t$, then $\theta_t-\bar\theta\in V$, on which $P_V$ is the identity. For the linearized model, the gradient of any loss $\sum_i\ell(f_\theta(x_i),y_i)$ is $\sum_i\partial_1\ell(f_\theta(x_i),y_i)\,\phi(x_i)\in V$, so gradient steps from $\theta_0$ stay in $\theta_0+V$.
\end{proof}

\begin{proof}[Proof of Corollary~\ref{cor:risk}]
Optimality of $P_k$ gives $\mathcal R(\mathcal H_k^\star)=\sum_{i>k}\lambda_i$.
Moreover,
\[
\mathcal R(\tilde{\mathcal H}_k)-\mathcal R(\mathcal H_k^\star)
=\E\ip{f_t^\star}{(P_k-\tilde P_k)f_t^\star}
\le \opnorm{\tilde P_k-P_k}\,\E\norm{f_t^\star}^2
= \tr(\cS)\,\opnorm{\tilde P_k-P_k}.
\]
Apply \eqref{eq:subspace-main}.
\end{proof}

\begin{remark}[Instantiating the per-task bound]
Assumption~\ref{ass:pertask} is a property of the learning procedure used for each task, and we do not derive it here.
Rademacher-complexity bounds~\citep{barlette2003rademacher} control the excess prediction risk of $\hat f_t$, which does not by itself bound the $\mathcal H$-norm distance $\norm{\hat f_t-f_t^\star}$; obtaining such a bound requires additional assumptions on the procedure and on $f_t^\star$.
For any procedure with an available bound $\eta_t(n_t,\delta_t)$, substituting it into $\bar\eta$ and $\overline{\eta^2}$ makes the rate in Theorem~\ref{thm:twolevel} explicit in the per-task sample sizes.
\end{remark}

\subsection{A sufficient condition for weight-subspace agreement}
\label{sec:existence_uws}

We give a sufficient-condition result that formalizes one possible mechanism behind universal weight subspaces. The result does not claim that general nonlinear networks necessarily satisfy these conditions. Rather, it isolates the assumptions under which independently trained networks must select nearby layer-wise subspaces.

Let two networks $A$ and $B$ have the same architecture and be trained on independent samples of size $N$ from a common distribution $\mathcal{D}_{\mathrm{world}}$. For layer $\ell$, let $\widehat{\Sigma}^{(\ell)}_M$ denote the empirical covariance of the layer inputs in model $M\in\{A,B\}$, and let $\Sigma^{(\ell)}_M$ be its population counterpart. Write $\mathcal{U}^{(\ell)}_M$ for the dominant $k$-dimensional weight subspace selected at that layer. For two $k$-dimensional subspaces $\mathcal{U}$ and $\mathcal{V}$, define
\[
d(\mathcal{U},\mathcal{V})
:=\|P_{\mathcal{U}}-P_{\mathcal{V}}\|_{\mathrm{op}},
\]
where $P_{\mathcal{U}}$ and $P_{\mathcal{V}}$ are the corresponding orthogonal projectors.

\begin{assumption}[Layer-wise spectral selection]\label{ass:spectral-selection}
For each layer $\ell$ and model $M$, the training dynamics select a dominant weight subspace $\mathcal{U}^{(\ell)}_M$ equal to the top-$k$ eigenspace of $\widehat{\Sigma}^{(\ell)}_M$.
\end{assumption}

This is the substantive modeling assumption. The piecewise-affine structure of ReLU networks gives an exact local linear representation~\citep{balestriero2018mad}, while results for matrix factorization and deep linear networks show that gradient-based optimization can favor low-rank, mode-wise structure in such regimes~\citep{gunasekar2017implicit,arora2019implicit,saxe2013exact}. These results motivate Assumption~\ref{ass:spectral-selection}, but do not establish it for arbitrary nonlinear networks.

\begin{assumption}[Concentration, eigengap, and stability]\label{ass:covariance-stability}
For each layer $\ell$ and model $M$, suppose that, with probability at least $1-\eta$,
\[
\|\widehat{\Sigma}^{(\ell)}_M-\Sigma^{(\ell)}_M\|_{\mathrm{op}}
\leq \varepsilon_\ell(N,\eta).
\]
Assume that both population covariances have a top-$k$ eigengap of at least $\delta_\ell>0$. For $\ell\geq2$, assume further that
\[
\|\Sigma^{(\ell)}_A-\Sigma^{(\ell)}_B\|_{\mathrm{op}}
\leq L_\ell\,
d(\mathcal{U}^{(\ell-1)}_A,\mathcal{U}^{(\ell-1)}_B)
\]
for some $L_\ell\geq0$.
\end{assumption}

For a fixed feature map and independent samples, standard matrix
concentration results apply under appropriate boundedness or tail
conditions~\citep{tropp2012user}. For learned features, we assume
the stated concentration bound. The final inequality requires the layer-input covariance to change smoothly with the preceding dominant subspace; it is stated explicitly because this property need not hold for every architecture or training regime.

\begin{theorem}[Layer-wise weight-subspace agreement]\label{thm:weight-subspace-agreement}
Under Assumptions~\ref{ass:spectral-selection}--\ref{ass:covariance-stability}, for each layer $\ell\geq2$, the following bound holds with probability at least $1-2\eta$:
\begin{equation}
d(\mathcal{U}^{(\ell)}_A,\mathcal{U}^{(\ell)}_B)
\leq
\frac{2}{\delta_\ell}
\left(
2\varepsilon_\ell(N,\eta)
+L_\ell d(\mathcal{U}^{(\ell-1)}_A,\mathcal{U}^{(\ell-1)}_B)
\right)
\label{eq:layerwise-subspace-recursion}
\end{equation}
At the first layer, where the two models have the same population input covariance, the following bound holds with the same probability:
\[
d(\mathcal{U}^{(1)}_A,\mathcal{U}^{(1)}_B)
\leq \frac{4\varepsilon_1(N,\eta)}{\delta_1}.
\]
Consequently, for any fixed depth, if $\varepsilon_\ell(N,\eta)\to0$ as $N\to\infty$, then the corresponding dominant layer-wise subspaces of the two independently trained networks agree asymptotically. If $\rho_\ell:=2L_\ell/\delta_\ell<1$, the recursion is contractive up to the finite-sample error floor.
\end{theorem}

\begin{proof}
Let $\mathcal{E}^{(\ell)}_M$ be the top-$k$ eigenspace of $\Sigma^{(\ell)}_M$. By the population-gap form of the Davis--Kahan theorem~\citep{yu2015useful}, covariance concentration gives
\[
d(\mathcal{U}^{(\ell)}_M,\mathcal{E}^{(\ell)}_M)
\leq \frac{2\varepsilon_\ell(N,\eta)}{\delta_\ell}
\]
for each model. A second application of Davis--Kahan and Assumption~\ref{ass:covariance-stability} gives
\[
d(\mathcal{E}^{(\ell)}_A,\mathcal{E}^{(\ell)}_B)
\leq \frac{2L_\ell}{\delta_\ell}
d(\mathcal{U}^{(\ell-1)}_A,\mathcal{U}^{(\ell-1)}_B).
\]
The triangle inequality for orthogonal projectors yields \cref{eq:layerwise-subspace-recursion}; a union bound gives probability at least $1-2\eta$. At the first layer, $\Sigma^{(1)}_A=\Sigma^{(1)}_B=\Sigma_{\mathrm{world}}$, so the middle population-subspace term vanishes. Iterating the recursion proves the final statements.
\end{proof}

The theorem separates three effects: finite-sample covariance error, the spectral separation $\delta_\ell$, and the propagation of earlier subspace differences through $L_\ell$. It provides a qualified explanation for why common data statistics and architecture-specific transformations can lead independent training runs to nearby dominant weight subspaces. Establishing Assumption~\ref{ass:spectral-selection} for broad classes of nonlinear networks remains an open problem.

\clearpage
\section{Appendix B}\label{sec:appendix_b}
\subsection{A data-free suitability metric for universal subspace approximations}\label{sec:metric_apx}
Since a good approximation to the universal subspace requires a large, diverse, and well-trained set of models (in our case) equivalently, an ideal algorithm and an ideal set of data to train the models on, we propose a data-free suitability metric to assess the quality of the source models before fitting the HOSVD basis. This metric is based on the heavy-tail exponent of the weight matrices, which has been shown to correlate with generalization and model quality. The fundamental idea is that we want to figure out how close the universal subspace is to an ideal universal subspace for a specific architecture. However, this is a difficult problem to solve because we do not have access to the said ideal universal surface. So instead, we come up with a metric which tells us relatively how good an extracted universal space is from a set of available pre-trained models. The higher the metric, the better the relative quality is the universal space approximation. We use this score as a heuristic indicator of source-model
suitability, rather than a direct measure of subspace approximation error.

For a layer $l$ and source model $m$, we estimate the heavy-tail exponent $\alpha_{l,m}$ by fitting a power law to the tail of the empirical spectral density of $W_{l,m}^{\top}W_{l,m}$, following the Heavy-Tailed Self-Regularization analysis of neural weights~\citep{martin21,martin2025setolsemiempiricaltheorydeep}. This statistic provides a data-free description of the correlations in a trained weight matrix. It does not measure subspace alignment directly. We use it to assess the source models before fitting the HOSVD basis.

For $N_l$ valid estimates at layer $l$, let $\sigma_l$ be their standard deviation. We define
\begin{equation}
q_l=
\left(1-e^{-N_l/\kappa}\right)
\left(
\frac{1}{N_l}\sum_{m=1}^{N_l}
\mathbf{1}[2\leq\alpha_{l,m}\leq6]
e^{-\eta(\alpha_{l,m}-2)}
\right)e^{-\lambda\sigma_l}.
\label{eq:alpha-layer-score}
\end{equation}
The three factors respectively account for the number of available models, the fraction and position of exponents in the preferred range, and agreement across models. Unless stated otherwise, we use $\kappa=20$, $\eta=0.5$, and $\lambda=1$. We first require the layer mean to lie in $[2,6]$, its standard deviation to be at most $1$, and at least half of its valid estimates to lie in $[2,6]$; among these layers, we would like to retain those with $q_l\geq0.2$. Any experiment-specific relaxation of these thresholds is reported with that experiment.

We summarize a family before filtering by
\begin{equation}
q_{\mathrm{all}}=
\exp\left(\frac{1}{L}\sum_{l=1}^{L}\log\max(q_l,\epsilon)\right),
\end{equation}
and summarize the retained set $\mathcal{L}_{\mathrm{keep}}$ by
\begin{equation}
q_{\mathrm{arch}}=
\exp\left(\frac{1}{|\mathcal{L}_{\mathrm{keep}}|}
\sum_{l\in\mathcal{L}_{\mathrm{keep}}}\log\max(q_l,\epsilon)\right),
\end{equation}
where $\epsilon=10^{-12}$; we set $q_{\mathrm{arch}}=0$ if no layer is retained. The geometric mean prevents a few high-scoring layers from obscuring consistently weak layers. The scores are coarse diagnostics of whether the available model collection is sufficiently large and of good quality for estimating a shared subspace approximation.

Importantly, this line of research may help investigate the idea of the ideal neural architecture, but this remains out of scope for current work and is open to future research. 

\subsection{Additional Details for ResNet50 Universal-Subspace Experiments}
\label{sec:resnet50_apx}

For the ResNet-50 experiment, all models are independently trained from scratch with random initialization. We use nine models trained on ImageNet-1K, ImageNet-21K, Places365, iNaturalist, RadImageNet, VGGFace2, CIFAR-10, CIFAR-100, and SVHN. These datasets are largely disjoint and span substantially different visual domains, including natural images, scenes, fine-grained species, medical radiology, faces, objects, and digits, providing a strong test of whether a shared subspace emerges across independently learned representations.

Table~\ref{tab:resnet_scratch} reports \textbf{out-of-domain (OOD)} results on five datasets: ImageNet-1K, Places365, CIFAR-10, CIFAR-100, and SVHN. For each OOD evaluation, the Universal Subspace is constructed using all models \emph{except} the model trained on the target dataset being evaluated. Thus, the reported performance measures how well a subspace learned entirely from other models can approximate the held-out target model. For a convolution tensor
$W_l\in\mathbb R^{C_{\mathrm{out}}\times C_{\mathrm{in}}\times k\times k}$, we reshape each
spatial kernel into a row of length $k^2$ and stack all $C_{\mathrm{out}}C_{\mathrm{in}}$ kernels
from all source models before fitting a shared SVD/PCA basis. 

For each shared weight tensor, we compute a heavy-tail exponent $\alpha$
and the layer suitability score described in \cref{sec:metric_apx}. Across the 52 common convolution tensors,
$q_{\mathrm{all}}=0.0183$. So this analysis experiment essentially shows how a very small number of models, which are often trained from scratch, may not be sufficient for a universal subspace analysis. We however keep the
experiment as a stress test for our hypothesis: the metric is shown in
\autoref{fig:resnet_scratch_alpha}, and the reconstruction results should be interpreted under
this low-suitability caveat, as more like a lower bound for the performance of the universal space model. The fact that a model which is a relatively poor approximation of the universal space for this specific architecture still performs adequately well and gets very close to the performance of the original fully fine-tuned models, suggests that the universal subspace hypothesis is indeed valid and can be extended to other architectures and model families.

\begin{figure}[h]
  \centering
  \includegraphics[width=0.72\linewidth]{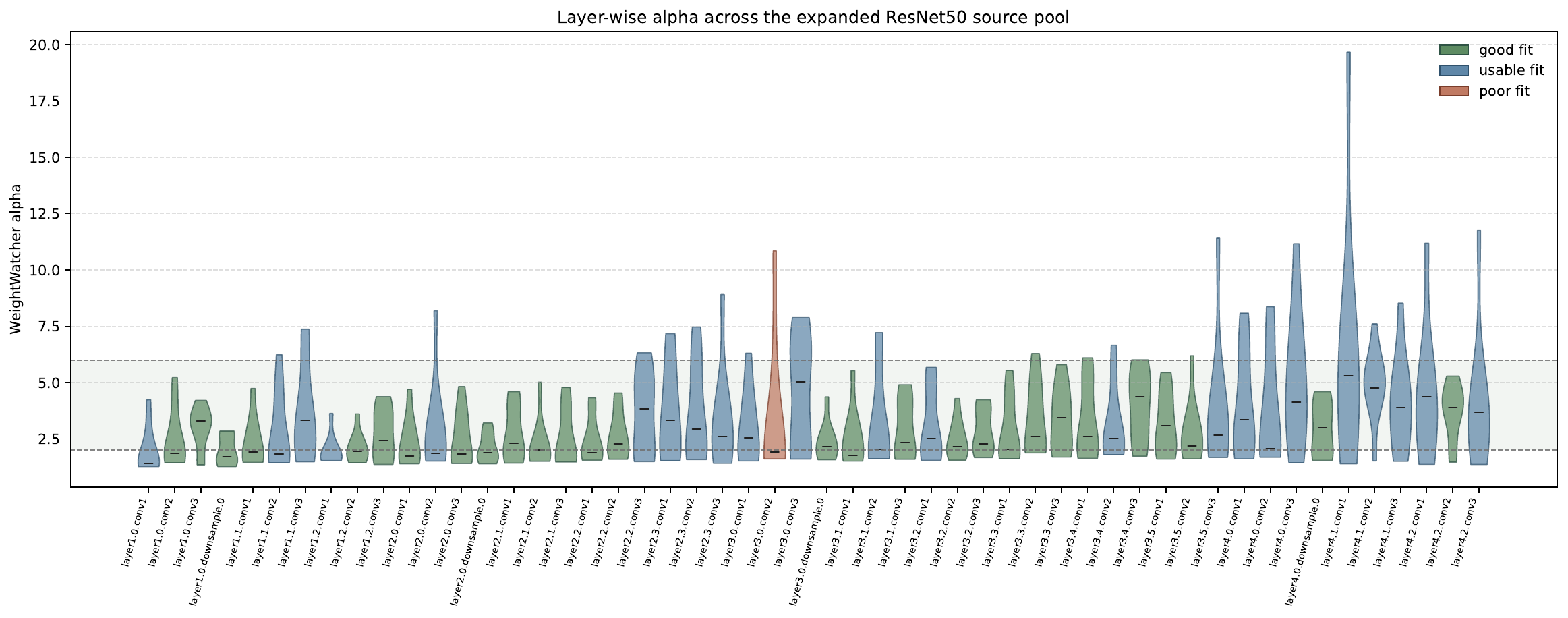}
  \caption{\footnotesize Layer-wise suitability diagnostics for the randomly initialized and trained from scratch ResNet-50 models. The strict metric reports $q_{\mathrm{all}}=0.0183$ reflecting the relatively poor quality of the universal space extractors due to the small number and diversity of models available. However, as our experiments have shown, our hypothesis still stands and and for universal space approximation still performs adequately.}
  \label{fig:resnet_scratch_alpha}
\end{figure}

\paragraph{Evaluation and calibration protocols.}
We evaluate the Universal Subspace under both \textbf{IID} and \textbf{OOD} settings, first through direct parameter reconstruction and then through lightweight coefficient adaptation. In the IID setting, the shared basis is constructed from the full model pool, including the target model; this tests whether models corresponding to already observed tasks lie within the recovered Universal Subspace. In the OOD setting, we perform leave-one-model-out construction, excluding the target model entirely when estimating the basis. This provides a stronger test of whether a model from a likely unseen, or future task also lies within the same subspace.

Table~\ref{tab:resnet_scratch} evaluates direct reconstruction on five target datasets, with \emph{no task-specific optimization or fine-tuning}. Despite the relatively crude approximation of the Universal Subspace in this small-model setting, both IID and OOD reconstructions recover models whose performance remains close to their independently fully trained counterparts. This provides direct evidence that independently learned model weights, including those excluded from subspace construction, can be well represented within the recovered linear subspace.

Table~\ref{tab:resnet_scratch_calibration} further evaluates CIFAR-10, CIFAR-100, and SVHN by optimizing the lightweight coefficients within the fixed Universal Subspace. This is equivalent to restricting optimization to the recovered linear subspace rather than updating the full model parameters. With this minimal adaptation, both IID and OOD models approach the performance of full training, despite the imperfect basis estimate. 
We adapt the subspace coefficients while keeping the shared bases fixed; OOD adaptation additionally trains the classifier head. This lightweight adaptation brings both IID and OOD models close to full-training performance.

\begin{table}[h]
    \centering
    \caption{\footnotesize
    \textbf{Direct reconstruction of scratch-trained ResNet-50 models.}
    Mean accuracy across five evaluated targets: ImageNet-1K, Places365,
    CIFAR-10, CIFAR-100, and SVHN. Models are reconstructed directly from
    the Universal Subspace without any task-specific fine-tuning.
    Variance Threshold denotes the fraction of spectral variance retained
    when constructing the subspace.}
    \label{tab:resnet_scratch}

    \small
    \setlength{\tabcolsep}{8pt}
    \begin{tabular}{lccc}
        \toprule
        \textbf{Protocol} &
        \textbf{Variance Threshold} &
        \textbf{Original Acc.} &
        \textbf{Reconstructed Acc.} \\
        \midrule
        IID & 60\% & 81.00 & 62.71 \\
        IID & 70\% & 81.00 & 71.56 \\
        IID & 80\% & 81.00 & 78.26 \\
        IID & 85\% & 81.00 & 79.86 \\
        IID & 90\% & 81.00 & 80.45 \\
        IID & 95\% & 81.00 & 80.85 \\
        \midrule
        OOD & 60\% & 81.00 & 56.69 \\
        OOD & 70\% & 81.00 & 69.97 \\
        OOD & 80\% & 81.00 & 77.91 \\
        OOD & 85\% & 81.00 & 79.80 \\
        OOD & 90\% & 81.00 & 80.41 \\
        OOD & 95\% & 81.00 & 80.85 \\
        \bottomrule
    \end{tabular}
\end{table}

\begin{table}[h]
    \centering
    \caption{\footnotesize
    \textbf{Lightweight coefficient adaptation within the Universal Subspace.}
    Mean accuracy across CIFAR-10, CIFAR-100, and SVHN.
    IID optimizes only subspace coefficients; OOD additionally adapts the classifier head.
    Coefficient parameters exclude the classifier head.}
    \label{tab:resnet_scratch_calibration}

    \footnotesize
    \setlength{\tabcolsep}{3.5pt}
    \renewcommand{\arraystretch}{1.05}

    \begin{tabular}{lccccc}
        \toprule
        \textbf{Protocol} &
        \shortstack{\textbf{Variance}\\\textbf{Threshold}} &
        \shortstack{\textbf{Original}\\\textbf{Acc.}} &
        \shortstack{\textbf{Reconstructed}\\\textbf{Acc.}} &
        \shortstack{\textbf{Adapted}\\\textbf{Acc.}} &
        \shortstack{\textbf{Coeff.}\\\textbf{Params}} \\
        \midrule
        IID & 60\% & 90.63 & 70.52 & 86.25 & 1.28M \\
        IID & 70\% & 90.63 & 80.98 & 87.63 & 1.64M \\
        IID & 80\% & 90.63 & 88.07 & 88.76 & 2.10M \\
        \midrule
        OOD & 60\% & 90.59 & 69.82 & 86.74 & 1.28M \\
        OOD & 70\% & 90.59 & 79.82 & 87.86 & 1.64M \\
        OOD & 80\% & 90.59 & 87.06 & 88.87 & 2.10M \\
        \bottomrule
    \end{tabular}
\end{table}

\subsection{Additional StyleGAN2 Universal-Subspace Experiments}
\label{sec:stylegan_apx}

We additionally test the Universal Subspace hypothesis on StyleGAN2, providing a substantially different setting from the discriminative models considered elsewhere in the paper. We use 20 publicly available StyleGAN2 checkpoints, spanning highly diverse domains including faces, animals, scenes, artwork, textures, maps, and other specialized datasets. Some models are trained from scratch, while a small number are fine-tuned.

This setting is particularly challenging because only 20 models are available for estimating the shared subspace. Consistent with this limitation, the full model set obtains a relatively low suitability score of
$q_{\mathrm{all}}=0.0572$. Applying a strict layer-quality filter would therefore discard most usable layers. We instead use a conservative threshold that retains 9 of the 13 common $3\times3$ synthesis convolution layers, yielding a retained-layer suitability score of $q_{\mathrm{arch}}=0.0907$. We emphasize that this remains an intentionally difficult setting and comparatively low-quality approximation of the Universal Subspace.

For each retained layer, we treat its kernels as spatial vectors, construct the corresponding shared spectral basis across the source models, and retain the smallest number of components explaining 90\% of the variance. The Universal Subspace model is obtained by replacing each retained target-model layer with its projection onto the corresponding shared spatial basis, while leaving non-retained layers unchanged. We follow the same IID/OOD protocols used in the preceding experiments: IID includes the target model when constructing the shared basis, whereas OOD excludes it entirely.

Despite the limited model pool and relatively weak estimated subspace, reconstruction remains remarkably close to the original models. Across the 18 models for which reconstruction could be evaluated, the mean change in Inception Score is only $-0.0041$ for IID and $-0.0040$ for OOD, while paired reconstruction achieves approximately 20\,dB PSNR in both settings. Thus, even for a complex adversarial generative model with a substantially different training mechanism, the learned weights remain well approximated by a shared low-dimensional subspace.

\begin{table}[h]
    \centering
    \caption{\footnotesize
    \textbf{StyleGAN2 reconstruction from the Universal Subspace.}
    Mean results across 18 evaluated generators. IID includes the target model
    during basis construction, whereas OOD excludes it.}
    \label{tab:stylegan_iid_ood}

    \footnotesize
    \setlength{\tabcolsep}{7pt}
    \begin{tabular}{lcccc}
        \toprule
        \textbf{Protocol} &
        \textbf{Original IS} &
        \textbf{Reconstructed IS} &
        \textbf{$\Delta$IS} &
        \textbf{PSNR (dB)} \\
        \midrule
        IID &
        $3.208 \pm 1.273$ &
        $3.204 \pm 1.167$ &
        $-0.0041 \pm 0.2620$ &
        $19.98 \pm 2.25$ \\
        OOD &
        $3.208 \pm 1.273$ &
        $3.204 \pm 1.166$ &
        $-0.0040 \pm 0.2626$ &
        $19.92 \pm 2.28$ \\
        \bottomrule
    \end{tabular}
\end{table}

\begin{figure}[h]
    \centering
    \includegraphics[width=\linewidth]{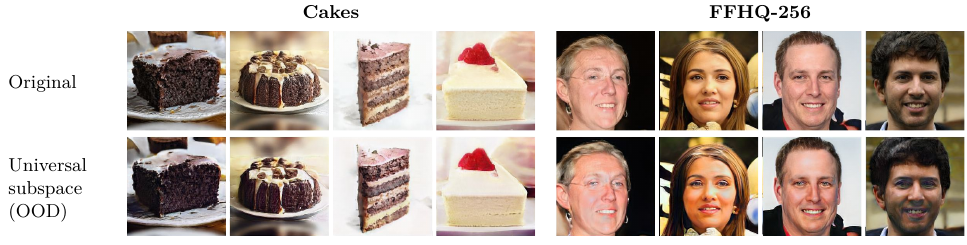}
    \caption{\footnotesize
    StyleGAN2 samples from the original models (top) and their Universal Subspace
    reconstructions (bottom), using matched latent seeds. Each target model is
    excluded when estimating its reconstruction basis (OOD).}
    \label{fig:stylegan_results}
\end{figure}

\subsection{Geometry Distributions: Scratch-Trained 3D Generation}
\label{sec:geomdist_apx}

We further test the Universal Subspace hypothesis on Geometry Distributions (GeomDist)~\citep{Zhang_2025_ICCV}, a 3D generative model that represents shapes as distributions over coordinates. This provides another substantially different setting from both discriminative networks and image generators. We use 20 models, each randomly initialized and trained independently from scratch for 1000 epochs on a different 3D primitive.

Despite the relatively small number of available models, the models obtain suitability scores of $q_{\mathrm{all}}=0.1501$ and $q_{\mathrm{arch}}=0.1977$, retaining 18 of the 22 compatible matrix-valued layers. For each retained layer, we construct the shared spectral basis across models and project the target weights onto this Universal Subspace.

As shown in Table~\ref{tab:geomdist_projection_apx}, the reconstructed models using the Universal subspace retain the geometry learned by the independently trained generative models: the mean pairwise Chamfer distance between samples from the original and reconstructed models is only $0.0250$ - the reconstructed models remain close to their original counterparts without any fine-tuning. Together with the qualitative reconstructions in Figure~\ref{fig:geomdist_reconstructions}, this extends the Universal Subspace hypothesis to independently trained 3D generative models with a fundamentally different architecture, training objective, and output domain. This experiment provides strong evidence to the generalizability of the Universal Subspace hypothesis across a wide range of model families and tasks, including discriminative, generative, and 3D models, especially in the case of classical weights like MLP or transformer-like models, which usually require huge amounts of data to train and are hard to analyze for the universal hypothesis.

\begin{table}[h]
    \centering
    \caption{\footnotesize
    \textbf{Geometry Distributions reconstruction from the Universal Subspace.}
    Reconstruction uses a 95\% explained-variance threshold with no fine-tuning.
    Chamfer distance (CD) is lower-is-better.}
    \label{tab:geomdist_projection_apx}

    \footnotesize
    \setlength{\tabcolsep}{8pt}
    \begin{tabular}{lcc}
        \toprule
        \textbf{Model} &
        \textbf{GT CD} &
        \textbf{Pair CD} \\
        \midrule
        Original &
        $0.0506 \pm 0.0070$ &
        -- \\
        Universal reconstruction &
        $0.0867 \pm 0.0802$ &
        $0.0250 \pm 0.0724$ \\
        \bottomrule
    \end{tabular}
\end{table}

\begin{figure}[h]
    \centering
    \includegraphics[width=0.90\linewidth]{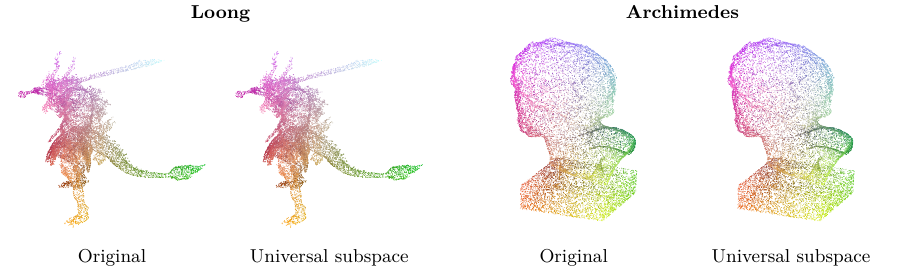}
    \caption{\footnotesize
    Geometry Distributions samples from the original models and their Universal
    Subspace reconstructions (IID), shown in pairs for Loong and Archimedes.
    Colors encode coordinates, not predicted texture.}
    \label{fig:geomdist_reconstructions}
\end{figure}

\subsection{Lower rank shared universal subspaces within low rank adaptation (LoRA) models}
\label{sec:lora_appx}
Spectral Decomposition is employed to extract the top \texttt{k} principal directions for each of the LoRA matrices \texttt{B} and \texttt{A}, which are concatenated across all available models. Subsequently, the top \texttt{k} principal directions are selected to define the low-rank subspace shared among the LoRA matrices. This process is conducted separately for each layer of the model to derive a low-rank approximated shared subspace for every individual layer. In practice, for every layer, the rank vectors of all available LoRA matrices are extracted and concatenated into a single matrix. This matrix is then normalized by subtracting the feature-wise mean from each vector, after which principal directions are extracted. The order-1 variant of our method (a single decomposed mode in \cref{alg:trunc-hosvd}) is mathematically equivalent to Principal Component Analysis (PCA), hence we can use \texttt{torch.pca\_lowrank} or \texttt{sklearn.decomposition.PCA} to extract the principal directions. The data matrix corresponding to a specific layer for 500 LoRA models is structured as $500r\times d$, where $r$ denotes the rank of each LoRA and $d$ specifies the dimension of each rank vector. Spectral decomposition is applied separately to the stored LoRA factors
$A$ and $B$ for each layer. For a given factor, rank vectors from all
available adapters are concatenated, zero-centered using the feature-wise
mean, and decomposed to extract the dominant shared directions. This yields
a layer-wise low-dimensional basis for the LoRA parameterization produced
by standard training.

We note that the LoRA factorization is not unique: for any invertible
$R$, $BA=(BR)(R^{-1}A)$. Our analysis therefore characterizes recurring
structure in the learned and stored LoRA parameterization, rather than a
gauge-invariant property of the equivalent full update $\Delta W=BA$.
A gauge-invariant alternative would perform the same analysis directly on
$\Delta W$.

\begin{figure}[htb]
  \centering
  \includegraphics[width=1.0\textwidth]{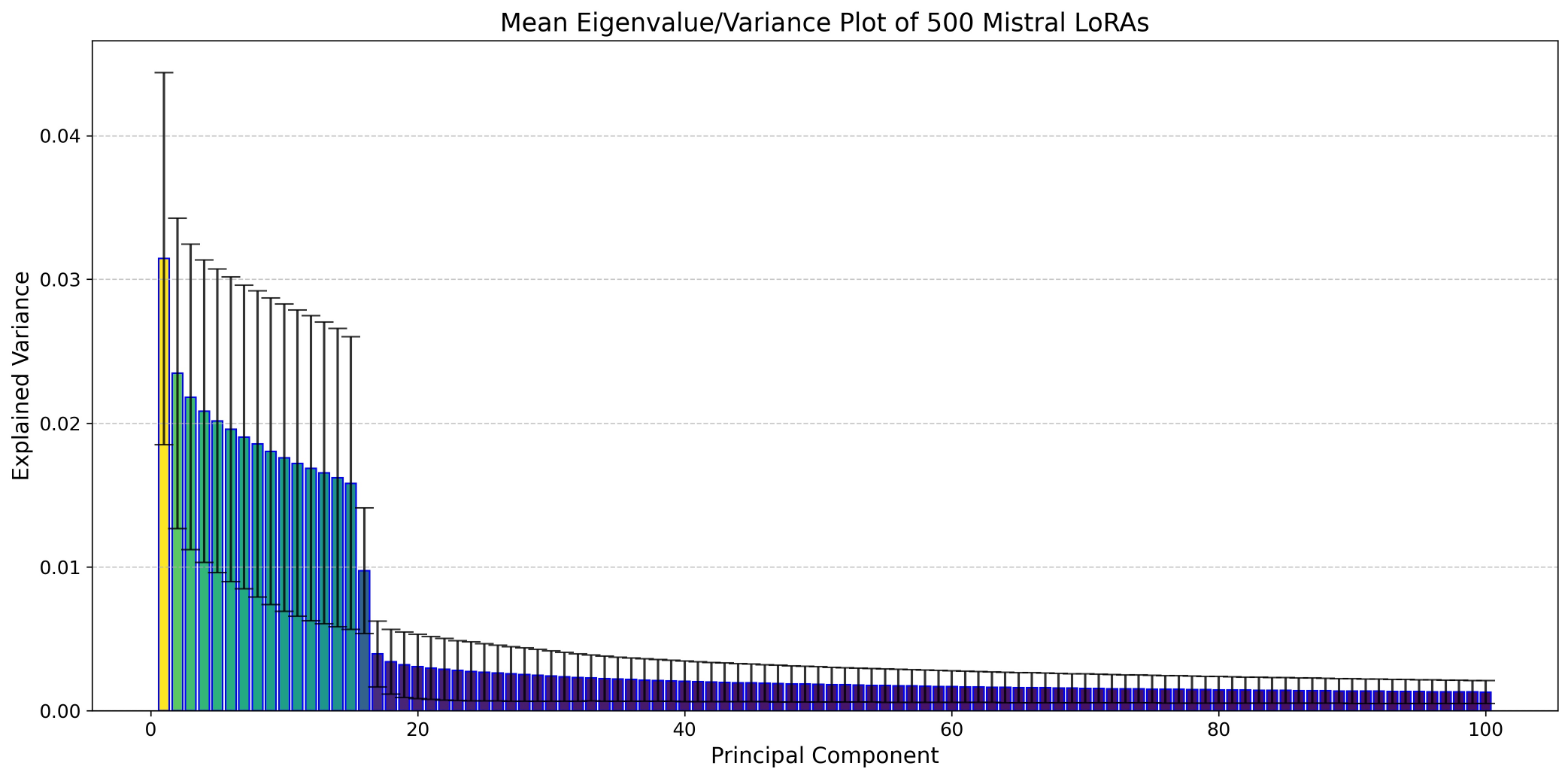}
  \caption{Spectral analysis of the Mistral-7B-Instruct-v0.2 model: Aggregated eigenvalue (scree) plot across 500 LoRA models and all layers. The plot demonstrates that the majority of the variance is consistently captured by the top few principal directions, indicating the presence of a shared low-dimensional universal subspace.}

  \label{fig:lora_main}
\end{figure}
\begin{figure}[htb]
  \centering
  \includegraphics[width=0.98\textwidth,height=0.23\textheight,keepaspectratio]{figures/k_proj.png}
  \includegraphics[width=0.98\textwidth,height=0.23\textheight,keepaspectratio]{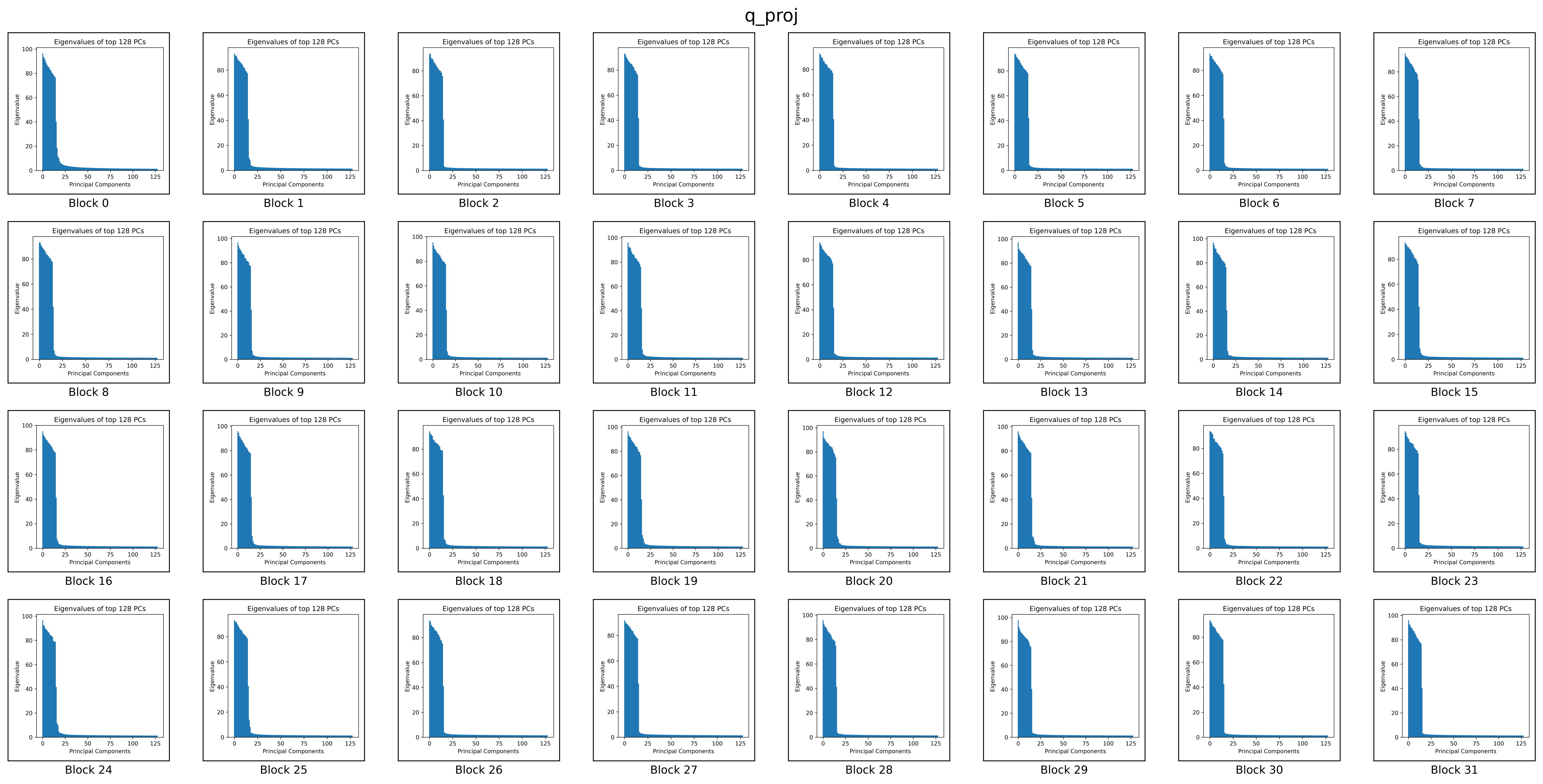}
  \includegraphics[width=0.98\textwidth,height=0.23\textheight,keepaspectratio]{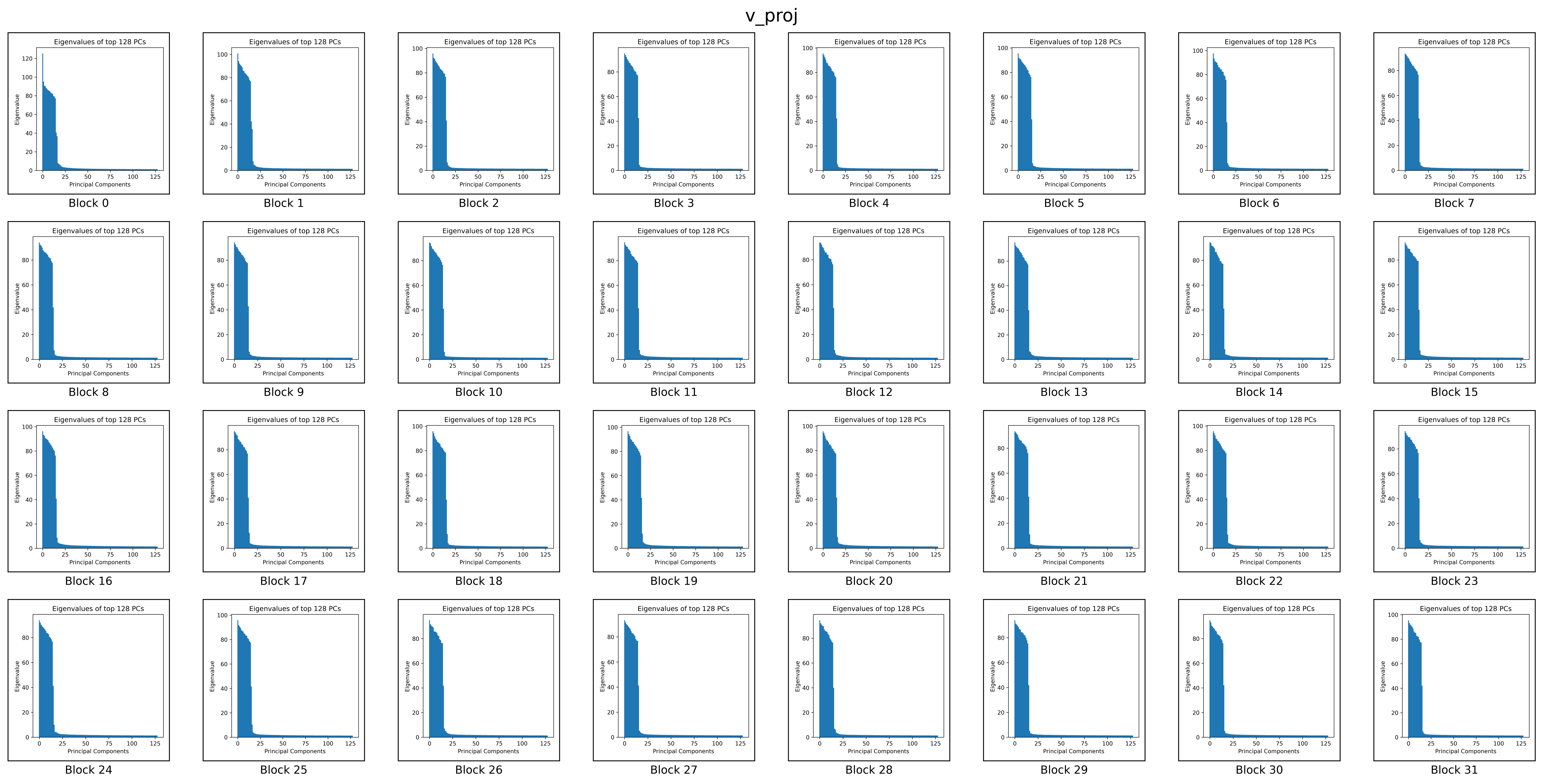}
  \caption{Layerwise Eigenvalue Plots of 500 Mistral-7B-Instruct-v0.2 models. Each layer has 3 sets of parameters - $k\_proj, q\_proj, v\_proj$}
  \label{tab:lora_all}
\end{figure}

\paragraph{Universal Mistral-7B/Lots of LoRAs experiment details} In our first experimental analysis, we use 500 LoRA models trained on distinct Natural Instructions~\citep{wang-etal-2022-super} using Mistral-7B-Instruct-v0.2~\citep{jiang2023mistral7b} as the base~\citep{brüelgabrielsson2024compressserveservingthousands}. Please refer to ~\cite{brüelgabrielsson2024compressserveservingthousands} for more details on how the LoRA models were trained. 
{\tiny
\begin{longtable}{|p{0.45\textwidth}|p{0.45\textwidth}|}
\caption{Models from HuggingFace for the Universal Mistral LoRA. Models in blue indicate the OOD models and the ones in red are the IID models used for evaluation.}
\label{tab:mistral_models} \\
\hline
\textcolor{blue}{Lots-of-LoRAs/Mistral-7B-Instruct-v0.2-4b-r16-task391} & \textcolor{blue}{Lots-of-LoRAs/Mistral-7B-Instruct-v0.2-4b-r16-task290} \\
\hline
\textcolor{blue}{Lots-of-LoRAs/Mistral-7B-Instruct-v0.2-4b-r16-task442} & \textcolor{blue}{Lots-of-LoRAs/Mistral-7B-Instruct-v0.2-4b-r16-task1598} \\
\hline
\textcolor{blue}{Lots-of-LoRAs/Mistral-7B-Instruct-v0.2-4b-r16-task039} &  \\
\hline
\textcolor{red}{Lots-of-LoRAs/Mistral-7B-Instruct-v0.2-4b-r16-task076} & \textcolor{red}{Lots-of-LoRAs/Mistral-7B-Instruct-v0.2-4b-r16-task627} \\
\hline
\textcolor{red}{Lots-of-LoRAs/Mistral-7B-Instruct-v0.2-4b-r16-task664} & \textcolor{red}{Lots-of-LoRAs/Mistral-7B-Instruct-v0.2-4b-r16-task819} \\
\hline
\textcolor{red}{Lots-of-LoRAs/Mistral-7B-Instruct-v0.2-4b-r16-task1631} &  \\
\hline
Lots-of-LoRAs/Mistral-7B-Instruct-v0.2-4b-r16-task190 & Lots-of-LoRAs/Mistral-7B-Instruct-v0.2-4b-r16-task1391 \\
\hline
Lots-of-LoRAs/Mistral-7B-Instruct-v0.2-4b-r16-task1342 & Lots-of-LoRAs/Mistral-7B-Instruct-v0.2-4b-r16-task620 \\
\hline
Lots-of-LoRAs/Mistral-7B-Instruct-v0.2-4b-r16-task769 & Lots-of-LoRAs/Mistral-7B-Instruct-v0.2-4b-r16-task1448 \\
\hline
Lots-of-LoRAs/Mistral-7B-Instruct-v0.2-4b-r16-task247 & Lots-of-LoRAs/Mistral-7B-Instruct-v0.2-4b-r16-task513 \\
\hline
Lots-of-LoRAs/Mistral-7B-Instruct-v0.2-4b-r16-task875 & Lots-of-LoRAs/Mistral-7B-Instruct-v0.2-4b-r16-task515 \\
\hline
Lots-of-LoRAs/Mistral-7B-Instruct-v0.2-4b-r16-task1534 & Lots-of-LoRAs/Mistral-7B-Instruct-v0.2-4b-r16-task1551 \\
\hline
Lots-of-LoRAs/Mistral-7B-Instruct-v0.2-4b-r16-task583 & Lots-of-LoRAs/Mistral-7B-Instruct-v0.2-4b-r16-task1431 \\
\hline
Lots-of-LoRAs/Mistral-7B-Instruct-v0.2-4b-r16-task270 & Lots-of-LoRAs/Mistral-7B-Instruct-v0.2-4b-r16-task1487 \\
\hline
Lots-of-LoRAs/Mistral-7B-Instruct-v0.2-4b-r16-task679 & Lots-of-LoRAs/Mistral-7B-Instruct-v0.2-4b-r16-task456 \\
\hline
Lots-of-LoRAs/Mistral-7B-Instruct-v0.2-4b-r16-task385 & Lots-of-LoRAs/Mistral-7B-Instruct-v0.2-4b-r16-task1607 \\
\hline
Lots-of-LoRAs/Mistral-7B-Instruct-v0.2-4b-r16-task278 & Lots-of-LoRAs/Mistral-7B-Instruct-v0.2-4b-r16-task022 \\
\hline
Lots-of-LoRAs/Mistral-7B-Instruct-v0.2-4b-r16-task210 & Lots-of-LoRAs/Mistral-7B-Instruct-v0.2-4b-r16-task137 \\
\hline
Lots-of-LoRAs/Mistral-7B-Instruct-v0.2-4b-r16-task574 & Lots-of-LoRAs/Mistral-7B-Instruct-v0.2-4b-r16-task629 \\
\hline
Lots-of-LoRAs/Mistral-7B-Instruct-v0.2-4b-r16-task1378 & Lots-of-LoRAs/Mistral-7B-Instruct-v0.2-4b-r16-task1194 \\
\hline
Lots-of-LoRAs/Mistral-7B-Instruct-v0.2-4b-r16-task1529 & Lots-of-LoRAs/Mistral-7B-Instruct-v0.2-4b-r16-task453 \\
\hline
Lots-of-LoRAs/Mistral-7B-Instruct-v0.2-4b-r16-task102 & Lots-of-LoRAs/Mistral-7B-Instruct-v0.2-4b-r16-task460 \\
\hline
Lots-of-LoRAs/Mistral-7B-Instruct-v0.2-4b-r16-task1204 & Lots-of-LoRAs/Mistral-7B-Instruct-v0.2-4b-r16-task1384 \\
\hline
Lots-of-LoRAs/Mistral-7B-Instruct-v0.2-4b-r16-task1572 & Lots-of-LoRAs/Mistral-7B-Instruct-v0.2-4b-r16-task699 \\
\hline
Lots-of-LoRAs/Mistral-7B-Instruct-v0.2-4b-r16-task1722 & Lots-of-LoRAs/Mistral-7B-Instruct-v0.2-4b-r16-task580 \\
\hline
Lots-of-LoRAs/Mistral-7B-Instruct-v0.2-4b-r16-task605 & Lots-of-LoRAs/Mistral-7B-Instruct-v0.2-4b-r16-task1152 \\
\hline
Lots-of-LoRAs/Mistral-7B-Instruct-v0.2-4b-r16-task1283 & Lots-of-LoRAs/Mistral-7B-Instruct-v0.2-4b-r16-task637 \\
\hline
Lots-of-LoRAs/Mistral-7B-Instruct-v0.2-4b-r16-task723 & Lots-of-LoRAs/Mistral-7B-Instruct-v0.2-4b-r16-task084 \\
\hline
Lots-of-LoRAs/Mistral-7B-Instruct-v0.2-4b-r16-task201 & Lots-of-LoRAs/Mistral-7B-Instruct-v0.2-4b-r16-task956 \\
\hline
Lots-of-LoRAs/Mistral-7B-Instruct-v0.2-4b-r16-task167 & Lots-of-LoRAs/Mistral-7B-Instruct-v0.2-4b-r16-task1192 \\
\hline
Lots-of-LoRAs/Mistral-7B-Instruct-v0.2-4b-r16-task300 & Lots-of-LoRAs/Mistral-7B-Instruct-v0.2-4b-r16-task1714 \\
\hline
Lots-of-LoRAs/Mistral-7B-Instruct-v0.2-4b-r16-task388 & Lots-of-LoRAs/Mistral-7B-Instruct-v0.2-4b-r16-task516 \\
\hline
Lots-of-LoRAs/Mistral-7B-Instruct-v0.2-4b-r16-task127 & Lots-of-LoRAs/Mistral-7B-Instruct-v0.2-4b-r16-task362 \\
\hline
Lots-of-LoRAs/Mistral-7B-Instruct-v0.2-4b-r16-task1158 & Lots-of-LoRAs/Mistral-7B-Instruct-v0.2-4b-r16-task322 \\
\hline
Lots-of-LoRAs/Mistral-7B-Instruct-v0.2-4b-r16-task697 & Lots-of-LoRAs/Mistral-7B-Instruct-v0.2-4b-r16-task1566 \\
\hline
Lots-of-LoRAs/Mistral-7B-Instruct-v0.2-4b-r16-task1451 & Lots-of-LoRAs/Mistral-7B-Instruct-v0.2-4b-r16-task1135 \\
\hline
Lots-of-LoRAs/Mistral-7B-Instruct-v0.2-4b-r16-task341 & Lots-of-LoRAs/Mistral-7B-Instruct-v0.2-4b-r16-task267 \\
\hline
Lots-of-LoRAs/Mistral-7B-Instruct-v0.2-4b-r16-task1720 & Lots-of-LoRAs/Mistral-7B-Instruct-v0.2-4b-r16-task1452 \\
\hline
Lots-of-LoRAs/Mistral-7B-Instruct-v0.2-4b-r16-task131 & Lots-of-LoRAs/Mistral-7B-Instruct-v0.2-4b-r16-task685 \\
\hline
Lots-of-LoRAs/Mistral-7B-Instruct-v0.2-4b-r16-task727 & Lots-of-LoRAs/Mistral-7B-Instruct-v0.2-4b-r16-task1590 \\
\hline
Lots-of-LoRAs/Mistral-7B-Instruct-v0.2-4b-r16-task1731 & Lots-of-LoRAs/Mistral-7B-Instruct-v0.2-4b-r16-task047 \\
\hline
Lots-of-LoRAs/Mistral-7B-Instruct-v0.2-4b-r16-task929 & Lots-of-LoRAs/Mistral-7B-Instruct-v0.2-4b-r16-task1592 \\
\hline
Lots-of-LoRAs/Mistral-7B-Instruct-v0.2-4b-r16-task1326 & Lots-of-LoRAs/Mistral-7B-Instruct-v0.2-4b-r16-task615 \\
\hline
Lots-of-LoRAs/Mistral-7B-Instruct-v0.2-4b-r16-task1216 & Lots-of-LoRAs/Mistral-7B-Instruct-v0.2-4b-r16-task689 \\
\hline
Lots-of-LoRAs/Mistral-7B-Instruct-v0.2-4b-r16-task1156 & Lots-of-LoRAs/Mistral-7B-Instruct-v0.2-4b-r16-task1657 \\
\hline
Lots-of-LoRAs/Mistral-7B-Instruct-v0.2-4b-r16-task833 & Lots-of-LoRAs/Mistral-7B-Instruct-v0.2-4b-r16-task1206 \\
\hline
Lots-of-LoRAs/Mistral-7B-Instruct-v0.2-4b-r16-task1151 & Lots-of-LoRAs/Mistral-7B-Instruct-v0.2-4b-r16-task244 \\
\hline
Lots-of-LoRAs/Mistral-7B-Instruct-v0.2-4b-r16-task1562 & Lots-of-LoRAs/Mistral-7B-Instruct-v0.2-4b-r16-task043 \\
\hline
Lots-of-LoRAs/Mistral-7B-Instruct-v0.2-4b-r16-task044 & Lots-of-LoRAs/Mistral-7B-Instruct-v0.2-4b-r16-task722 \\
\hline
Lots-of-LoRAs/Mistral-7B-Instruct-v0.2-4b-r16-task183 & Lots-of-LoRAs/Mistral-7B-Instruct-v0.2-4b-r16-task563 \\
\hline
Lots-of-LoRAs/Mistral-7B-Instruct-v0.2-4b-r16-task155 & Lots-of-LoRAs/Mistral-7B-Instruct-v0.2-4b-r16-task353 \\
\hline
Lots-of-LoRAs/Mistral-7B-Instruct-v0.2-4b-r16-task616 & Lots-of-LoRAs/Mistral-7B-Instruct-v0.2-4b-r16-task1724 \\
\hline
Lots-of-LoRAs/Mistral-7B-Instruct-v0.2-4b-r16-task288 & Lots-of-LoRAs/Mistral-7B-Instruct-v0.2-4b-r16-task092 \\
\hline
Lots-of-LoRAs/Mistral-7B-Instruct-v0.2-4b-r16-task707 & Lots-of-LoRAs/Mistral-7B-Instruct-v0.2-4b-r16-task577 \\
\hline
Lots-of-LoRAs/Mistral-7B-Instruct-v0.2-4b-r16-task742 & Lots-of-LoRAs/Mistral-7B-Instruct-v0.2-4b-r16-task706 \\
\hline
Lots-of-LoRAs/Mistral-7B-Instruct-v0.2-4b-r16-task1401 & Lots-of-LoRAs/Mistral-7B-Instruct-v0.2-4b-r16-task1393 \\
\hline
Lots-of-LoRAs/Mistral-7B-Instruct-v0.2-4b-r16-task1198 & Lots-of-LoRAs/Mistral-7B-Instruct-v0.2-4b-r16-task966 \\
\hline
Lots-of-LoRAs/Mistral-7B-Instruct-v0.2-4b-r16-task219 & Lots-of-LoRAs/Mistral-7B-Instruct-v0.2-4b-r16-task1211 \\
\hline
Lots-of-LoRAs/Mistral-7B-Instruct-v0.2-4b-r16-task050 & Lots-of-LoRAs/Mistral-7B-Instruct-v0.2-4b-r16-task494 \\
\hline
Lots-of-LoRAs/Mistral-7B-Instruct-v0.2-4b-r16-task1379 & Lots-of-LoRAs/Mistral-7B-Instruct-v0.2-4b-r16-task176 \\
\hline
Lots-of-LoRAs/Mistral-7B-Instruct-v0.2-4b-r16-task068 & Lots-of-LoRAs/Mistral-7B-Instruct-v0.2-4b-r16-task566 \\
\hline
Lots-of-LoRAs/Mistral-7B-Instruct-v0.2-4b-r16-task333 & Lots-of-LoRAs/Mistral-7B-Instruct-v0.2-4b-r16-task593 \\
\hline
Lots-of-LoRAs/Mistral-7B-Instruct-v0.2-4b-r16-task667 & Lots-of-LoRAs/Mistral-7B-Instruct-v0.2-4b-r16-task1670 \\
\hline
Lots-of-LoRAs/Mistral-7B-Instruct-v0.2-4b-r16-task733 & Lots-of-LoRAs/Mistral-7B-Instruct-v0.2-4b-r16-task472 \\
\hline
Lots-of-LoRAs/Mistral-7B-Instruct-v0.2-4b-r16-task1168 & Lots-of-LoRAs/Mistral-7B-Instruct-v0.2-4b-r16-task075 \\
\hline
Lots-of-LoRAs/Mistral-7B-Instruct-v0.2-4b-r16-task148 & Lots-of-LoRAs/Mistral-7B-Instruct-v0.2-4b-r16-task683 \\
\hline
Lots-of-LoRAs/Mistral-7B-Instruct-v0.2-4b-r16-task1315 & Lots-of-LoRAs/Mistral-7B-Instruct-v0.2-4b-r16-task121 \\
\hline
Lots-of-LoRAs/Mistral-7B-Instruct-v0.2-4b-r16-task370 & Lots-of-LoRAs/Mistral-7B-Instruct-v0.2-4b-r16-task856 \\
\hline
Lots-of-LoRAs/Mistral-7B-Instruct-v0.2-4b-r16-task891 & Lots-of-LoRAs/Mistral-7B-Instruct-v0.2-4b-r16-task140 \\
\hline
Lots-of-LoRAs/Mistral-7B-Instruct-v0.2-4b-r16-task609 & Lots-of-LoRAs/Mistral-7B-Instruct-v0.2-4b-r16-task344 \\
\hline
Lots-of-LoRAs/Mistral-7B-Instruct-v0.2-4b-r16-task1703 & Lots-of-LoRAs/Mistral-7B-Instruct-v0.2-4b-r16-task070 \\
\hline
Lots-of-LoRAs/Mistral-7B-Instruct-v0.2-4b-r16-task072 & Lots-of-LoRAs/Mistral-7B-Instruct-v0.2-4b-r16-task504 \\
\hline
Lots-of-LoRAs/Mistral-7B-Instruct-v0.2-4b-r16-task695 & Lots-of-LoRAs/Mistral-7B-Instruct-v0.2-4b-r16-task1434 \\
\hline
Lots-of-LoRAs/Mistral-7B-Instruct-v0.2-4b-r16-task095 & Lots-of-LoRAs/Mistral-7B-Instruct-v0.2-4b-r16-task346 \\
\hline
Lots-of-LoRAs/Mistral-7B-Instruct-v0.2-4b-r16-task274 & Lots-of-LoRAs/Mistral-7B-Instruct-v0.2-4b-r16-task1325 \\
\hline
Lots-of-LoRAs/Mistral-7B-Instruct-v0.2-4b-r16-task1190 & Lots-of-LoRAs/Mistral-7B-Instruct-v0.2-4b-r16-task568 \\
\hline
Lots-of-LoRAs/Mistral-7B-Instruct-v0.2-4b-r16-task1482 & Lots-of-LoRAs/Mistral-7B-Instruct-v0.2-4b-r16-task924 \\
\hline
Lots-of-LoRAs/Mistral-7B-Instruct-v0.2-4b-r16-task761 & Lots-of-LoRAs/Mistral-7B-Instruct-v0.2-4b-r16-task596 \\
\hline
Lots-of-LoRAs/Mistral-7B-Instruct-v0.2-4b-r16-task382 & Lots-of-LoRAs/Mistral-7B-Instruct-v0.2-4b-r16-task926 \\
\hline
Lots-of-LoRAs/Mistral-7B-Instruct-v0.2-4b-r16-task065 & Lots-of-LoRAs/Mistral-7B-Instruct-v0.2-4b-r16-task1421 \\
\hline
Lots-of-LoRAs/Mistral-7B-Instruct-v0.2-4b-r16-task323 & Lots-of-LoRAs/Mistral-7B-Instruct-v0.2-4b-r16-task1310 \\
\hline
Lots-of-LoRAs/Mistral-7B-Instruct-v0.2-4b-r16-task110 & Lots-of-LoRAs/Mistral-7B-Instruct-v0.2-4b-r16-task1288 \\
\hline
Lots-of-LoRAs/Mistral-7B-Instruct-v0.2-4b-r16-task1503 & Lots-of-LoRAs/Mistral-7B-Instruct-v0.2-4b-r16-task269 \\
\hline
Lots-of-LoRAs/Mistral-7B-Instruct-v0.2-4b-r16-task821 & Lots-of-LoRAs/Mistral-7B-Instruct-v0.2-4b-r16-task565 \\
\hline
Lots-of-LoRAs/Mistral-7B-Instruct-v0.2-4b-r16-task867 & Lots-of-LoRAs/Mistral-7B-Instruct-v0.2-4b-r16-task755 \\
\hline
Lots-of-LoRAs/Mistral-7B-Instruct-v0.2-4b-r16-task378 & Lots-of-LoRAs/Mistral-7B-Instruct-v0.2-4b-r16-task518 \\
\hline
Lots-of-LoRAs/Mistral-7B-Instruct-v0.2-4b-r16-task195 & Lots-of-LoRAs/Mistral-7B-Instruct-v0.2-4b-r16-task618 \\
\hline
Lots-of-LoRAs/Mistral-7B-Instruct-v0.2-4b-r16-task638 & Lots-of-LoRAs/Mistral-7B-Instruct-v0.2-4b-r16-task1217 \\
\hline
Lots-of-LoRAs/Mistral-7B-Instruct-v0.2-4b-r16-task118 & Lots-of-LoRAs/Mistral-7B-Instruct-v0.2-4b-r16-task1564 \\
\hline
Lots-of-LoRAs/Mistral-7B-Instruct-v0.2-4b-r16-task1429 & Lots-of-LoRAs/Mistral-7B-Instruct-v0.2-4b-r16-task687 \\
\hline
Lots-of-LoRAs/Mistral-7B-Instruct-v0.2-4b-r16-task640 & Lots-of-LoRAs/Mistral-7B-Instruct-v0.2-4b-r16-task1328 \\
\hline
Lots-of-LoRAs/Mistral-7B-Instruct-v0.2-4b-r16-task1311 & Lots-of-LoRAs/Mistral-7B-Instruct-v0.2-4b-r16-task285 \\
\hline
Lots-of-LoRAs/Mistral-7B-Instruct-v0.2-4b-r16-task1341 & Lots-of-LoRAs/Mistral-7B-Instruct-v0.2-4b-r16-task1603 \\
\hline
Lots-of-LoRAs/Mistral-7B-Instruct-v0.2-4b-r16-task162 & Lots-of-LoRAs/Mistral-7B-Instruct-v0.2-4b-r16-task063 \\
\hline
Lots-of-LoRAs/Mistral-7B-Instruct-v0.2-4b-r16-task686 & Lots-of-LoRAs/Mistral-7B-Instruct-v0.2-4b-r16-task1483 \\
\hline
Lots-of-LoRAs/Mistral-7B-Instruct-v0.2-4b-r16-task146 & Lots-of-LoRAs/Mistral-7B-Instruct-v0.2-4b-r16-task1729 \\
\hline
Lots-of-LoRAs/Mistral-7B-Instruct-v0.2-4b-r16-task852 & Lots-of-LoRAs/Mistral-7B-Instruct-v0.2-4b-r16-task1622 \\
\hline
Lots-of-LoRAs/Mistral-7B-Instruct-v0.2-4b-r16-task1704 & Lots-of-LoRAs/Mistral-7B-Instruct-v0.2-4b-r16-task145 \\
\hline
Lots-of-LoRAs/Mistral-7B-Instruct-v0.2-4b-r16-task648 & Lots-of-LoRAs/Mistral-7B-Instruct-v0.2-4b-r16-task151 \\
\hline
Lots-of-LoRAs/Mistral-7B-Instruct-v0.2-4b-r16-task1212 & Lots-of-LoRAs/Mistral-7B-Instruct-v0.2-4b-r16-task045 \\
\hline
Lots-of-LoRAs/Mistral-7B-Instruct-v0.2-4b-r16-task893 & Lots-of-LoRAs/Mistral-7B-Instruct-v0.2-4b-r16-task936 \\
\hline
Lots-of-LoRAs/Mistral-7B-Instruct-v0.2-4b-r16-task1425 & Lots-of-LoRAs/Mistral-7B-Instruct-v0.2-4b-r16-task1533 \\
\hline
Lots-of-LoRAs/Mistral-7B-Instruct-v0.2-4b-r16-task093 & Lots-of-LoRAs/Mistral-7B-Instruct-v0.2-4b-r16-task389 \\
\hline
Lots-of-LoRAs/Mistral-7B-Instruct-v0.2-4b-r16-task670 & Lots-of-LoRAs/Mistral-7B-Instruct-v0.2-4b-r16-task461 \\
\hline
Lots-of-LoRAs/Mistral-7B-Instruct-v0.2-4b-r16-task113 & Lots-of-LoRAs/Mistral-7B-Instruct-v0.2-4b-r16-task066 \\
\hline
Lots-of-LoRAs/Mistral-7B-Instruct-v0.2-4b-r16-task497 & Lots-of-LoRAs/Mistral-7B-Instruct-v0.2-4b-r16-task343 \\
\hline
Lots-of-LoRAs/Mistral-7B-Instruct-v0.2-4b-r16-task908 & Lots-of-LoRAs/Mistral-7B-Instruct-v0.2-4b-r16-task1606 \\
\hline
Lots-of-LoRAs/Mistral-7B-Instruct-v0.2-4b-r16-task381 & Lots-of-LoRAs/Mistral-7B-Instruct-v0.2-4b-r16-task355 \\
\hline
Lots-of-LoRAs/Mistral-7B-Instruct-v0.2-4b-r16-task850 & Lots-of-LoRAs/Mistral-7B-Instruct-v0.2-4b-r16-task1557 \\
\hline
Lots-of-LoRAs/Mistral-7B-Instruct-v0.2-4b-r16-task356 & Lots-of-LoRAs/Mistral-7B-Instruct-v0.2-4b-r16-task1394 \\
\hline
Lots-of-LoRAs/Mistral-7B-Instruct-v0.2-4b-r16-task1428 & Lots-of-LoRAs/Mistral-7B-Instruct-v0.2-4b-r16-task087 \\
\hline
Lots-of-LoRAs/Mistral-7B-Instruct-v0.2-4b-r16-task582 & Lots-of-LoRAs/Mistral-7B-Instruct-v0.2-4b-r16-task770 \\
\hline
Lots-of-LoRAs/Mistral-7B-Instruct-v0.2-4b-r16-task035 & Lots-of-LoRAs/Mistral-7B-Instruct-v0.2-4b-r16-task1317 \\
\hline
Lots-of-LoRAs/Mistral-7B-Instruct-v0.2-4b-r16-task1203 & Lots-of-LoRAs/Mistral-7B-Instruct-v0.2-4b-r16-task1444 \\
\hline
Lots-of-LoRAs/Mistral-7B-Instruct-v0.2-4b-r16-task1585 & Lots-of-LoRAs/Mistral-7B-Instruct-v0.2-4b-r16-task1508 \\
\hline
Lots-of-LoRAs/Mistral-7B-Instruct-v0.2-4b-r16-task740 & Lots-of-LoRAs/Mistral-7B-Instruct-v0.2-4b-r16-task429 \\
\hline
Lots-of-LoRAs/Mistral-7B-Instruct-v0.2-4b-r16-task1427 & Lots-of-LoRAs/Mistral-7B-Instruct-v0.2-4b-r16-task328 \\
\hline
Lots-of-LoRAs/Mistral-7B-Instruct-v0.2-4b-r16-task955 & Lots-of-LoRAs/Mistral-7B-Instruct-v0.2-4b-r16-task130 \\
\hline
Lots-of-LoRAs/Mistral-7B-Instruct-v0.2-4b-r16-task161 & Lots-of-LoRAs/Mistral-7B-Instruct-v0.2-4b-r16-task507 \\
\hline
Lots-of-LoRAs/Mistral-7B-Instruct-v0.2-4b-r16-task1502 & Lots-of-LoRAs/Mistral-7B-Instruct-v0.2-4b-r16-task505 \\
\hline
Lots-of-LoRAs/Mistral-7B-Instruct-v0.2-4b-r16-task633 & Lots-of-LoRAs/Mistral-7B-Instruct-v0.2-4b-r16-task1645 \\
\hline
Lots-of-LoRAs/Mistral-7B-Instruct-v0.2-4b-r16-task1486 & Lots-of-LoRAs/Mistral-7B-Instruct-v0.2-4b-r16-task1146 \\
\hline
Lots-of-LoRAs/Mistral-7B-Instruct-v0.2-4b-r16-task1380 & Lots-of-LoRAs/Mistral-7B-Instruct-v0.2-4b-r16-task1088 \\
\hline
Lots-of-LoRAs/Mistral-7B-Instruct-v0.2-4b-r16-task033 & Lots-of-LoRAs/Mistral-7B-Instruct-v0.2-4b-r16-task085 \\
\hline
Lots-of-LoRAs/Mistral-7B-Instruct-v0.2-4b-r16-task1294 & Lots-of-LoRAs/Mistral-7B-Instruct-v0.2-4b-r16-task080 \\
\hline
Lots-of-LoRAs/Mistral-7B-Instruct-v0.2-4b-r16-task489 & Lots-of-LoRAs/Mistral-7B-Instruct-v0.2-4b-r16-task1721 \\
\hline
Lots-of-LoRAs/Mistral-7B-Instruct-v0.2-4b-r16-task1713 & Lots-of-LoRAs/Mistral-7B-Instruct-v0.2-4b-r16-task721 \\
\hline
Lots-of-LoRAs/Mistral-7B-Instruct-v0.2-4b-r16-task1403 & Lots-of-LoRAs/Mistral-7B-Instruct-v0.2-4b-r16-task746 \\
\hline
Lots-of-LoRAs/Mistral-7B-Instruct-v0.2-4b-r16-task728 & Lots-of-LoRAs/Mistral-7B-Instruct-v0.2-4b-r16-task889 \\
\hline
Lots-of-LoRAs/Mistral-7B-Instruct-v0.2-4b-r16-task1583 & Lots-of-LoRAs/Mistral-7B-Instruct-v0.2-4b-r16-task1665 \\
\hline
Lots-of-LoRAs/Mistral-7B-Instruct-v0.2-4b-r16-task708 & Lots-of-LoRAs/Mistral-7B-Instruct-v0.2-4b-r16-task1419 \\
\hline
Lots-of-LoRAs/Mistral-7B-Instruct-v0.2-4b-r16-task963 & Lots-of-LoRAs/Mistral-7B-Instruct-v0.2-4b-r16-task454 \\
\hline
Lots-of-LoRAs/Mistral-7B-Instruct-v0.2-4b-r16-task308 & Lots-of-LoRAs/Mistral-7B-Instruct-v0.2-4b-r16-task828 \\
\hline
Lots-of-LoRAs/Mistral-7B-Instruct-v0.2-4b-r16-task579 & Lots-of-LoRAs/Mistral-7B-Instruct-v0.2-4b-r16-task753 \\
\hline
Lots-of-LoRAs/Mistral-7B-Instruct-v0.2-4b-r16-task1404 & Lots-of-LoRAs/Mistral-7B-Instruct-v0.2-4b-r16-task1201 \\
\hline
Lots-of-LoRAs/Mistral-7B-Instruct-v0.2-4b-r16-task901 & Lots-of-LoRAs/Mistral-7B-Instruct-v0.2-4b-r16-task1567 \\
\hline
Lots-of-LoRAs/Mistral-7B-Instruct-v0.2-4b-r16-task1319 & Lots-of-LoRAs/Mistral-7B-Instruct-v0.2-4b-r16-task858 \\
\hline
Lots-of-LoRAs/Mistral-7B-Instruct-v0.2-4b-r16-task1200 & Lots-of-LoRAs/Mistral-7B-Instruct-v0.2-4b-r16-task492 \\
\hline
Lots-of-LoRAs/Mistral-7B-Instruct-v0.2-4b-r16-task675 & Lots-of-LoRAs/Mistral-7B-Instruct-v0.2-4b-r16-task094 \\
\hline
Lots-of-LoRAs/Mistral-7B-Instruct-v0.2-4b-r16-task1506 & Lots-of-LoRAs/Mistral-7B-Instruct-v0.2-4b-r16-task694 \\
\hline
Lots-of-LoRAs/Mistral-7B-Instruct-v0.2-4b-r16-task614 & Lots-of-LoRAs/Mistral-7B-Instruct-v0.2-4b-r16-task1390 \\
\hline
Lots-of-LoRAs/Mistral-7B-Instruct-v0.2-4b-r16-task1355 & Lots-of-LoRAs/Mistral-7B-Instruct-v0.2-4b-r16-task714 \\
\hline
Lots-of-LoRAs/Mistral-7B-Instruct-v0.2-4b-r16-task457 & Lots-of-LoRAs/Mistral-7B-Instruct-v0.2-4b-r16-task1565 \\
\hline
Lots-of-LoRAs/Mistral-7B-Instruct-v0.2-4b-r16-task834 & Lots-of-LoRAs/Mistral-7B-Instruct-v0.2-4b-r16-task642 \\
\hline
Lots-of-LoRAs/Mistral-7B-Instruct-v0.2-4b-r16-task732 & Lots-of-LoRAs/Mistral-7B-Instruct-v0.2-4b-r16-task1605 \\
\hline
Lots-of-LoRAs/Mistral-7B-Instruct-v0.2-4b-r16-task326 & Lots-of-LoRAs/Mistral-7B-Instruct-v0.2-4b-r16-task1292 \\
\hline
Lots-of-LoRAs/Mistral-7B-Instruct-v0.2-4b-r16-task716 & Lots-of-LoRAs/Mistral-7B-Instruct-v0.2-4b-r16-task1479 \\
\hline
Lots-of-LoRAs/Mistral-7B-Instruct-v0.2-4b-r16-task1147 & Lots-of-LoRAs/Mistral-7B-Instruct-v0.2-4b-r16-task153 \\
\hline
Lots-of-LoRAs/Mistral-7B-Instruct-v0.2-4b-r16-task1495 & Lots-of-LoRAs/Mistral-7B-Instruct-v0.2-4b-r16-task1196 \\
\hline
Lots-of-LoRAs/Mistral-7B-Instruct-v0.2-4b-r16-task1489 & Lots-of-LoRAs/Mistral-7B-Instruct-v0.2-4b-r16-task294 \\
\hline
Lots-of-LoRAs/Mistral-7B-Instruct-v0.2-4b-r16-task157 & Lots-of-LoRAs/Mistral-7B-Instruct-v0.2-4b-r16-task147 \\
\hline
Lots-of-LoRAs/Mistral-7B-Instruct-v0.2-4b-r16-task1197 & Lots-of-LoRAs/Mistral-7B-Instruct-v0.2-4b-r16-task754 \\
\hline
Lots-of-LoRAs/Mistral-7B-Instruct-v0.2-4b-r16-task1599 & Lots-of-LoRAs/Mistral-7B-Instruct-v0.2-4b-r16-task1420 \\
\hline
Lots-of-LoRAs/Mistral-7B-Instruct-v0.2-4b-r16-task1285 & Lots-of-LoRAs/Mistral-7B-Instruct-v0.2-4b-r16-task587 \\
\hline
Lots-of-LoRAs/Mistral-7B-Instruct-v0.2-4b-r16-task1338 & Lots-of-LoRAs/Mistral-7B-Instruct-v0.2-4b-r16-task116 \\
\hline
Lots-of-LoRAs/Mistral-7B-Instruct-v0.2-4b-r16-task713 & Lots-of-LoRAs/Mistral-7B-Instruct-v0.2-4b-r16-task431 \\
\hline
Lots-of-LoRAs/Mistral-7B-Instruct-v0.2-4b-r16-task067 & Lots-of-LoRAs/Mistral-7B-Instruct-v0.2-4b-r16-task934 \\
\hline
Lots-of-LoRAs/Mistral-7B-Instruct-v0.2-4b-r16-task617 & Lots-of-LoRAs/Mistral-7B-Instruct-v0.2-4b-r16-task696 \\
\hline
Lots-of-LoRAs/Mistral-7B-Instruct-v0.2-4b-r16-task846 & Lots-of-LoRAs/Mistral-7B-Instruct-v0.2-4b-r16-task933 \\
\hline
Lots-of-LoRAs/Mistral-7B-Instruct-v0.2-4b-r16-task865 & Lots-of-LoRAs/Mistral-7B-Instruct-v0.2-4b-r16-task671 \\
\hline
Lots-of-LoRAs/Mistral-7B-Instruct-v0.2-4b-r16-task1398 & Lots-of-LoRAs/Mistral-7B-Instruct-v0.2-4b-r16-task1518 \\
\hline
Lots-of-LoRAs/Mistral-7B-Instruct-v0.2-4b-r16-task628 & Lots-of-LoRAs/Mistral-7B-Instruct-v0.2-4b-r16-task1286 \\
\hline
Lots-of-LoRAs/Mistral-7B-Instruct-v0.2-4b-r16-task1596 & Lots-of-LoRAs/Mistral-7B-Instruct-v0.2-4b-r16-task298 \\
\hline
Lots-of-LoRAs/Mistral-7B-Instruct-v0.2-4b-r16-task645 & Lots-of-LoRAs/Mistral-7B-Instruct-v0.2-4b-r16-task903 \\
\hline
Lots-of-LoRAs/Mistral-7B-Instruct-v0.2-4b-r16-task594 & Lots-of-LoRAs/Mistral-7B-Instruct-v0.2-4b-r16-task413 \\
\hline
Lots-of-LoRAs/Mistral-7B-Instruct-v0.2-4b-r16-task719 & Lots-of-LoRAs/Mistral-7B-Instruct-v0.2-4b-r16-task672 \\
\hline
Lots-of-LoRAs/Mistral-7B-Instruct-v0.2-4b-r16-task1418 & Lots-of-LoRAs/Mistral-7B-Instruct-v0.2-4b-r16-task475 \\
\hline
Lots-of-LoRAs/Mistral-7B-Instruct-v0.2-4b-r16-task357 & Lots-of-LoRAs/Mistral-7B-Instruct-v0.2-4b-r16-task1387 \\
\hline
Lots-of-LoRAs/Mistral-7B-Instruct-v0.2-4b-r16-task1711 & Lots-of-LoRAs/Mistral-7B-Instruct-v0.2-4b-r16-task304 \\
\hline
Lots-of-LoRAs/Mistral-7B-Instruct-v0.2-4b-r16-task750 & Lots-of-LoRAs/Mistral-7B-Instruct-v0.2-4b-r16-task1320 \\
\hline
Lots-of-LoRAs/Mistral-7B-Instruct-v0.2-4b-r16-task034 & Lots-of-LoRAs/Mistral-7B-Instruct-v0.2-4b-r16-task692 \\
\hline
Lots-of-LoRAs/Mistral-7B-Instruct-v0.2-4b-r16-task1406 & Lots-of-LoRAs/Mistral-7B-Instruct-v0.2-4b-r16-task119 \\
\hline
Lots-of-LoRAs/Mistral-7B-Instruct-v0.2-4b-r16-task211 & Lots-of-LoRAs/Mistral-7B-Instruct-v0.2-4b-r16-task363 \\
\hline
Lots-of-LoRAs/Mistral-7B-Instruct-v0.2-4b-r16-task1087 & Lots-of-LoRAs/Mistral-7B-Instruct-v0.2-4b-r16-task083 \\
\hline
Lots-of-LoRAs/Mistral-7B-Instruct-v0.2-4b-r16-task1385 & Lots-of-LoRAs/Mistral-7B-Instruct-v0.2-4b-r16-task1308 \\
\hline
Lots-of-LoRAs/Mistral-7B-Instruct-v0.2-4b-r16-task1656 & Lots-of-LoRAs/Mistral-7B-Instruct-v0.2-4b-r16-task892 \\
\hline
Lots-of-LoRAs/Mistral-7B-Instruct-v0.2-4b-r16-task499 & Lots-of-LoRAs/Mistral-7B-Instruct-v0.2-4b-r16-task1409 \\
\hline
Lots-of-LoRAs/Mistral-7B-Instruct-v0.2-4b-r16-task1207 & Lots-of-LoRAs/Mistral-7B-Instruct-v0.2-4b-r16-task564 \\
\hline
Lots-of-LoRAs/Mistral-7B-Instruct-v0.2-4b-r16-task325 & Lots-of-LoRAs/Mistral-7B-Instruct-v0.2-4b-r16-task206 \\
\hline
Lots-of-LoRAs/Mistral-7B-Instruct-v0.2-4b-r16-task890 & Lots-of-LoRAs/Mistral-7B-Instruct-v0.2-4b-r16-task1520 \\
\hline
Lots-of-LoRAs/Mistral-7B-Instruct-v0.2-4b-r16-task703 & Lots-of-LoRAs/Mistral-7B-Instruct-v0.2-4b-r16-task318 \\
\hline
Lots-of-LoRAs/Mistral-7B-Instruct-v0.2-4b-r16-task366 & Lots-of-LoRAs/Mistral-7B-Instruct-v0.2-4b-r16-task335 \\
\hline
Lots-of-LoRAs/Mistral-7B-Instruct-v0.2-4b-r16-task600 & Lots-of-LoRAs/Mistral-7B-Instruct-v0.2-4b-r16-task477 \\
\hline
Lots-of-LoRAs/Mistral-7B-Instruct-v0.2-4b-r16-task138 & Lots-of-LoRAs/Mistral-7B-Instruct-v0.2-4b-r16-task291 \\
\hline
Lots-of-LoRAs/Mistral-7B-Instruct-v0.2-4b-r16-task074 & Lots-of-LoRAs/Mistral-7B-Instruct-v0.2-4b-r16-task1389 \\
\hline
Lots-of-LoRAs/Mistral-7B-Instruct-v0.2-4b-r16-task192 & Lots-of-LoRAs/Mistral-7B-Instruct-v0.2-4b-r16-task316 \\
\hline
Lots-of-LoRAs/Mistral-7B-Instruct-v0.2-4b-r16-task1609 & Lots-of-LoRAs/Mistral-7B-Instruct-v0.2-4b-r16-task296 \\
\hline
Lots-of-LoRAs/Mistral-7B-Instruct-v0.2-4b-r16-task666 & Lots-of-LoRAs/Mistral-7B-Instruct-v0.2-4b-r16-task1582 \\
\hline
Lots-of-LoRAs/Mistral-7B-Instruct-v0.2-4b-r16-task1728 & Lots-of-LoRAs/Mistral-7B-Instruct-v0.2-4b-r16-task701 \\
\hline
Lots-of-LoRAs/Mistral-7B-Instruct-v0.2-4b-r16-task275 & Lots-of-LoRAs/Mistral-7B-Instruct-v0.2-4b-r16-task107 \\
\hline
Lots-of-LoRAs/Mistral-7B-Instruct-v0.2-4b-r16-task079 & Lots-of-LoRAs/Mistral-7B-Instruct-v0.2-4b-r16-task1157 \\
\hline
Lots-of-LoRAs/Mistral-7B-Instruct-v0.2-4b-r16-task1167 & Lots-of-LoRAs/Mistral-7B-Instruct-v0.2-4b-r16-task403 \\
\hline
Lots-of-LoRAs/Mistral-7B-Instruct-v0.2-4b-r16-task359 & Lots-of-LoRAs/Mistral-7B-Instruct-v0.2-4b-r16-task517 \\
\hline
Lots-of-LoRAs/Mistral-7B-Instruct-v0.2-4b-r16-task351 & Lots-of-LoRAs/Mistral-7B-Instruct-v0.2-4b-r16-task964 \\
\hline
Lots-of-LoRAs/Mistral-7B-Instruct-v0.2-4b-r16-task904 & Lots-of-LoRAs/Mistral-7B-Instruct-v0.2-4b-r16-task1148 \\
\hline
Lots-of-LoRAs/Mistral-7B-Instruct-v0.2-4b-r16-task879 & Lots-of-LoRAs/Mistral-7B-Instruct-v0.2-4b-r16-task636 \\
\hline
Lots-of-LoRAs/Mistral-7B-Instruct-v0.2-4b-r16-task1509 & Lots-of-LoRAs/Mistral-7B-Instruct-v0.2-4b-r16-task207 \\
\hline
Lots-of-LoRAs/Mistral-7B-Instruct-v0.2-4b-r16-task228 & Lots-of-LoRAs/Mistral-7B-Instruct-v0.2-4b-r16-task209 \\
\hline
Lots-of-LoRAs/Mistral-7B-Instruct-v0.2-4b-r16-task128 & Lots-of-LoRAs/Mistral-7B-Instruct-v0.2-4b-r16-task710 \\
\hline
Lots-of-LoRAs/Mistral-7B-Instruct-v0.2-4b-r16-task1322 & Lots-of-LoRAs/Mistral-7B-Instruct-v0.2-4b-r16-task163 \\
\hline
Lots-of-LoRAs/Mistral-7B-Instruct-v0.2-4b-r16-task178 & Lots-of-LoRAs/Mistral-7B-Instruct-v0.2-4b-r16-task089 \\
\hline
Lots-of-LoRAs/Mistral-7B-Instruct-v0.2-4b-r16-task700 & Lots-of-LoRAs/Mistral-7B-Instruct-v0.2-4b-r16-task1581 \\
\hline
Lots-of-LoRAs/Mistral-7B-Instruct-v0.2-4b-r16-task927 & Lots-of-LoRAs/Mistral-7B-Instruct-v0.2-4b-r16-task101 \\
\hline
Lots-of-LoRAs/Mistral-7B-Instruct-v0.2-4b-r16-task123 & Lots-of-LoRAs/Mistral-7B-Instruct-v0.2-4b-r16-task1321 \\
\hline
Lots-of-LoRAs/Mistral-7B-Instruct-v0.2-4b-r16-task550 & Lots-of-LoRAs/Mistral-7B-Instruct-v0.2-4b-r16-task129 \\
\hline
Lots-of-LoRAs/Mistral-7B-Instruct-v0.2-4b-r16-task393 & Lots-of-LoRAs/Mistral-7B-Instruct-v0.2-4b-r16-task1214 \\
\hline
Lots-of-LoRAs/Mistral-7B-Instruct-v0.2-4b-r16-task277 & Lots-of-LoRAs/Mistral-7B-Instruct-v0.2-4b-r16-task1447 \\
\hline
Lots-of-LoRAs/Mistral-7B-Instruct-v0.2-4b-r16-task324 & Lots-of-LoRAs/Mistral-7B-Instruct-v0.2-4b-r16-task455 \\
\hline
Lots-of-LoRAs/Mistral-7B-Instruct-v0.2-4b-r16-task725 & Lots-of-LoRAs/Mistral-7B-Instruct-v0.2-4b-r16-task365 \\
\hline
Lots-of-LoRAs/Mistral-7B-Instruct-v0.2-4b-r16-task1316 & Lots-of-LoRAs/Mistral-7B-Instruct-v0.2-4b-r16-task1199 \\
\hline
Lots-of-LoRAs/Mistral-7B-Instruct-v0.2-4b-r16-task717 & Lots-of-LoRAs/Mistral-7B-Instruct-v0.2-4b-r16-task245 \\
\hline
Lots-of-LoRAs/Mistral-7B-Instruct-v0.2-4b-r16-task874 & Lots-of-LoRAs/Mistral-7B-Instruct-v0.2-4b-r16-task925 \\
\hline
Lots-of-LoRAs/Mistral-7B-Instruct-v0.2-4b-r16-task380 & Lots-of-LoRAs/Mistral-7B-Instruct-v0.2-4b-r16-task1712 \\
\hline
Lots-of-LoRAs/Mistral-7B-Instruct-v0.2-4b-r16-task1504 & Lots-of-LoRAs/Mistral-7B-Instruct-v0.2-4b-r16-task619 \\
\hline
Lots-of-LoRAs/Mistral-7B-Instruct-v0.2-4b-r16-task590 & Lots-of-LoRAs/Mistral-7B-Instruct-v0.2-4b-r16-task1186 \\
\hline
Lots-of-LoRAs/Mistral-7B-Instruct-v0.2-4b-r16-task736 & Lots-of-LoRAs/Mistral-7B-Instruct-v0.2-4b-r16-task069 \\
\hline
Lots-of-LoRAs/Mistral-7B-Instruct-v0.2-4b-r16-task377 & Lots-of-LoRAs/Mistral-7B-Instruct-v0.2-4b-r16-task181 \\
\hline
Lots-of-LoRAs/Mistral-7B-Instruct-v0.2-4b-r16-task859 & Lots-of-LoRAs/Mistral-7B-Instruct-v0.2-4b-r16-task144 \\
\hline
Lots-of-LoRAs/Mistral-7B-Instruct-v0.2-4b-r16-task632 & Lots-of-LoRAs/Mistral-7B-Instruct-v0.2-4b-r16-task641 \\
\hline
Lots-of-LoRAs/Mistral-7B-Instruct-v0.2-4b-r16-task064 & Lots-of-LoRAs/Mistral-7B-Instruct-v0.2-4b-r16-task630 \\
\hline
Lots-of-LoRAs/Mistral-7B-Instruct-v0.2-4b-r16-task1154 & Lots-of-LoRAs/Mistral-7B-Instruct-v0.2-4b-r16-task390 \\
\hline
Lots-of-LoRAs/Mistral-7B-Instruct-v0.2-4b-r16-task1188 & Lots-of-LoRAs/Mistral-7B-Instruct-v0.2-4b-r16-task625 \\
\hline
Lots-of-LoRAs/Mistral-7B-Instruct-v0.2-4b-r16-task607 & Lots-of-LoRAs/Mistral-7B-Instruct-v0.2-4b-r16-task495 \\
\hline
Lots-of-LoRAs/Mistral-7B-Instruct-v0.2-4b-r16-task1189 & Lots-of-LoRAs/Mistral-7B-Instruct-v0.2-4b-r16-task398 \\
\hline
Lots-of-LoRAs/Mistral-7B-Instruct-v0.2-4b-r16-task108 & Lots-of-LoRAs/Mistral-7B-Instruct-v0.2-4b-r16-task1347 \\
\hline
Lots-of-LoRAs/Mistral-7B-Instruct-v0.2-4b-r16-task1541 & Lots-of-LoRAs/Mistral-7B-Instruct-v0.2-4b-r16-task202 \\
\hline
Lots-of-LoRAs/Mistral-7B-Instruct-v0.2-4b-r16-task1723 & Lots-of-LoRAs/Mistral-7B-Instruct-v0.2-4b-r16-task1669 \\
\hline
Lots-of-LoRAs/Mistral-7B-Instruct-v0.2-4b-r16-task1089 & Lots-of-LoRAs/Mistral-7B-Instruct-v0.2-4b-r16-task1584 \\
\hline
Lots-of-LoRAs/Mistral-7B-Instruct-v0.2-4b-r16-task081 & Lots-of-LoRAs/Mistral-7B-Instruct-v0.2-4b-r16-task329 \\
\hline
Lots-of-LoRAs/Mistral-7B-Instruct-v0.2-4b-r16-task691 & Lots-of-LoRAs/Mistral-7B-Instruct-v0.2-4b-r16-task588 \\
\hline
Lots-of-LoRAs/Mistral-7B-Instruct-v0.2-4b-r16-task1593 & Lots-of-LoRAs/Mistral-7B-Instruct-v0.2-4b-r16-task724 \\
\hline
Lots-of-LoRAs/Mistral-7B-Instruct-v0.2-4b-r16-task149 & Lots-of-LoRAs/Mistral-7B-Instruct-v0.2-4b-r16-task1449 \\
\hline
Lots-of-LoRAs/Mistral-7B-Instruct-v0.2-4b-r16-task1313 & Lots-of-LoRAs/Mistral-7B-Instruct-v0.2-4b-r16-task1453 \\
\hline
Lots-of-LoRAs/Mistral-7B-Instruct-v0.2-4b-r16-task905 & Lots-of-LoRAs/Mistral-7B-Instruct-v0.2-4b-r16-task704 \\
\hline
Lots-of-LoRAs/Mistral-7B-Instruct-v0.2-4b-r16-task585 & Lots-of-LoRAs/Mistral-7B-Instruct-v0.2-4b-r16-task1209 \\
\hline
Lots-of-LoRAs/Mistral-7B-Instruct-v0.2-4b-r16-task249 & Lots-of-LoRAs/Mistral-7B-Instruct-v0.2-4b-r16-task1386 \\
\hline
Lots-of-LoRAs/Mistral-7B-Instruct-v0.2-4b-r16-task1400 & Lots-of-LoRAs/Mistral-7B-Instruct-v0.2-4b-r16-task751 \\
\hline
Lots-of-LoRAs/Mistral-7B-Instruct-v0.2-4b-r16-task1332 & Lots-of-LoRAs/Mistral-7B-Instruct-v0.2-4b-r16-task674 \\
\hline
Lots-of-LoRAs/Mistral-7B-Instruct-v0.2-4b-r16-task379 & Lots-of-LoRAs/Mistral-7B-Instruct-v0.2-4b-r16-task243 \\
\hline
Lots-of-LoRAs/Mistral-7B-Instruct-v0.2-4b-r16-task1318 & Lots-of-LoRAs/Mistral-7B-Instruct-v0.2-4b-r16-task428 \\
\hline
Lots-of-LoRAs/Mistral-7B-Instruct-v0.2-4b-r16-task488 & Lots-of-LoRAs/Mistral-7B-Instruct-v0.2-4b-r16-task705 \\
\hline
Lots-of-LoRAs/Mistral-7B-Instruct-v0.2-4b-r16-task698 & Lots-of-LoRAs/Mistral-7B-Instruct-v0.2-4b-r16-task1601 \\
\hline
Lots-of-LoRAs/Mistral-7B-Instruct-v0.2-4b-r16-task861 & Lots-of-LoRAs/Mistral-7B-Instruct-v0.2-4b-r16-task1510 \\
\hline
Lots-of-LoRAs/Mistral-7B-Instruct-v0.2-4b-r16-task077 & Lots-of-LoRAs/Mistral-7B-Instruct-v0.2-4b-r16-task509 \\
\hline
Lots-of-LoRAs/Mistral-7B-Instruct-v0.2-4b-r16-task734 & Lots-of-LoRAs/Mistral-7B-Instruct-v0.2-4b-r16-task720 \\
\hline
Lots-of-LoRAs/Mistral-7B-Instruct-v0.2-4b-r16-task1210 & Lots-of-LoRAs/Mistral-7B-Instruct-v0.2-4b-r16-task284 \\
\hline
Lots-of-LoRAs/Mistral-7B-Instruct-v0.2-4b-r16-task584 & Lots-of-LoRAs/Mistral-7B-Instruct-v0.2-4b-r16-task105 \\
\hline
Lots-of-LoRAs/Mistral-7B-Instruct-v0.2-4b-r16-task330 & Lots-of-LoRAs/Mistral-7B-Instruct-v0.2-4b-r16-task923 \\
\hline
Lots-of-LoRAs/Mistral-7B-Instruct-v0.2-4b-r16-task319 & Lots-of-LoRAs/Mistral-7B-Instruct-v0.2-4b-r16-task400 \\
\hline
Lots-of-LoRAs/Mistral-7B-Instruct-v0.2-4b-r16-task246 & Lots-of-LoRAs/Mistral-7B-Instruct-v0.2-4b-r16-task726 \\
\hline
Lots-of-LoRAs/Mistral-7B-Instruct-v0.2-4b-r16-task1568 & Lots-of-LoRAs/Mistral-7B-Instruct-v0.2-4b-r16-task1442 \\
\hline
Lots-of-LoRAs/Mistral-7B-Instruct-v0.2-4b-r16-task1640 &  Lots-of-LoRAs/Mistral-7B-Instruct-v0.2-4b-r16-task280\\
\hline
\end{longtable}}

~\autoref{fig:lora_main} presents the aggregated results across all layers, with error bars representing the standard deviation. For reference, the eigenvalue (scree) plot from ~\autoref{fig:short-b} is also reproduced in ~\autoref{tab:lora_all}. This plot depicts the proportion of variance explained by each principal component, computed across all weight matrices and layers from 500 independently trained Mistral models. The concentration of variance within the top \(k\) components reveals the presence of a consistent low-dimensional subspace, offering strong empirical support for the universal subspace hypothesis.

The individual plots provide spectral analysis results for the key, query, and value matrices from all 31 layers of all 500 Mistral models. For clarity, only the top 128 principal directions are visualized, representing a subset of the full component basis. This truncation mitigates the visual distortion caused by the long tail of near-zero eigenvalues beyond the universal subspace, which would otherwise dominate the graph without contributing meaningful information.

To test subspace expressiveness, we reconstruct LoRA weights for both 5 seen (IID) and unseen (OOD) tasks by projecting them into the universal subspace. As shown in \autoref{fig:lola_perf}, the reconstructed models retain high performance in both cases. In contrast, projection into the residual \textit{Secondary Subspace} leads to a sharp performance drop, underscoring the importance of the principal subspace. Our method is also 19$\times$ more memory efficient, as it eliminates the need to store all 500 LoRAs.

{\tiny  
\begin{longtable}{|l|l|l|l|}
\caption{Models from HuggingFace used for the Universal Stable Diffusion-XL subspace extraction} \label{tab:diffusion_data} \\
\hline
alphonse-mucha-style & directors-coen-brothers-style & larry-carlson-style & rene-magritte-style \\ \hline
beeple-mike-winkelmann-style & director-sergei-eisenstein-style & lascaux & richard-corben-style \\ \hline
character-design & director-sofia-coppola-style & laurel-burch-style & richard-dadd-style \\ \hline
director-christopher-nolan-style & director-terrence-malick-style & lawrence-alma-tadema-style & richard-hescox-style \\ \hline
director-lars-von-trier-style & director-tim-burton-style & leonid-afremov-style & richard-scarry-style \\ \hline
director-ridley-scott-style & director-wes-anderson-style & leonora-carrington-style & robert-adams-style \\ \hline
director-stanley-kubrick-style & director-wong-kar-wai-style & levitating-cube & robert-crumb-style \\ \hline
director-zhang-yimou-style & director-yorgos-lanthimos-style & liam-wong-style & robert-rauschenberg-style \\ \hline
olafur-eliasson-style & dixit-card-generator & lotte-reiniger-style & rodney-matthews-style \\ \hline
origami & dressed-animals & louis-comfort-tiffany-style & roger-ballen-style \\ \hline
simone-martini-style & dripping-art & lovis-corinth-style & roger-deakins-style \\ \hline
studio-ghibli-style & edward-gorey-style & lucas-cranach-style & romare-bearden-style \\ \hline
ukiyo-e-art & elizabeth-gadd-style & luc-schuiten-style & ryoji-ikeda-style \\ \hline
wu-guanzhong-style & erik-johansson-style & lyonel-feininger-style & sacha-goldberger-style \\ \hline
1987-action-figure-playset-packaging & erik-madigan-heck-style & made-of-iridescent-foil & salomon-van-ruysdael-style \\ \hline
aardman-animations-style & euan-uglow-style & makoto-shinkai-style & sam-spratt-style \\ \hline
akos-major-style & felipe-pantone-style & marc-silvestri-style & sandy-skoglund-style \\ \hline
albumen-print & filip-hodas-style & marianna-rothen-style & santiago-caruso-style \\ \hline
alec-soth-style & folk-art & maria-sibylla-merian-style & shaun-tan-style \\ \hline
alejandro-jodorowsky-style & gabriel-pacheco-style & mark-catesby-style & shepard-fairey-style \\ \hline
alessandro-gottardo-style & gemma-correll-style & mark-ryden-style & sidney-nolan-style \\ \hline
alex-andreev-style & george-condo-style & martin-whatson-style & simon-stalenhag-style \\ \hline
alex-gross-style & gilbert-garcin-style & mary-cassatt-style & skottie-young-style \\ \hline
alfred-augustus-glendening-style & gregory-crewdson-style & maurice-de-vlaminck-style & sofonisba-anguissola-style \\ \hline
alex-pardee-style & gustave-dore-style & maurice-prendergast-style & sophie-gengembre-anderson-style \\ \hline
alternate-realities & hasui-kawase-style & maxfield-parrish-style & stained-glass-portrait \\ \hline
ando-fuchs-style & hiroshi-nagai-style & maxime-maufra-style & stanley-donwood-style \\ \hline
andre-derain-style & infrared-photos & mike-mignola-style & stephan-martiniere-style \\ \hline
andrei-tarkovsky-style & isometric-cutaway & mikhail-vrubel-style & stephen-gammell-style \\ \hline
andrew-wyeth-style & ivan-bilibin-style & moebius-jean-giraud-style & stop-motion-animation \\ \hline
angus-mckie-style & james-c-christensen-style & movie-poster & surreal-collage \\ \hline
anna-maria-garthwaite-style & james-jean-style & moving-meditations & surreal-harmony \\ \hline
atey-ghailan-style & james-r-eads-style & nadav-kander-style & surreal-plate \\ \hline
audrey-kawasaki-style & james-turrell-style & natalia-goncharova-style & syd-mead-style \\ \hline
avant-garde-fashion & jan-brueghel-style & n-c-wyeth-style & synthwave-t-shirt \\ \hline
banksy-style & jan-svankmajer-style & needlepoint & teamlab-style \\ \hline
bas-relief & jan-van-eyck-style & neon-night & terry-gilliam-style \\ \hline
century-botanical-illustration & jan-van-goyen-style & nicolas-poussin-style & thomas-cole-style \\ \hline
christopher-balaskas-style & j-c-leyendecker-style & noah-bradley-style & thomas-kinkade-style \\ \hline
christopher-ryan-mckenney-style & jean-baptiste-camille-corot-style & ohara-koson-style & thomas-moran-style \\ \hline
clay-animation & jean-baptiste-monge-style & okuda-san-miguel-style & thomas-schaller-style \\ \hline
color-palette & jean-baptiste-simeon-chardin-style & olly-moss-style & tim-walker-style \\ \hline
craig-mullins-style & jean-metzinger-style & op-art & tintoretto-style \\ \hline
crocheted & jean-michel-basquiat-style & parralel-dimensions & todd-hido-style \\ \hline
daniel-arsham-style & jessie-willcox-smith-style & pascal-campion-style & tove-jansson-style \\ \hline
dark-fantasy & jim-mahfood-style & paul-gustav-fischer-style & tracie-grimwood-style \\ \hline
dave-mckean-style & john-albert-bauer-style & paul-laffoley-style & vasily-vereshchagin-style \\ \hline
diorama & john-berkey-style & paul-signac-style & vertical-landscapes \\ \hline
director-agnes-varda-style & john-blanche-style & peter-doig-style & victor-brauner-style \\ \hline
death-stranding & john-constable-style & peter-paul-rubens-style & victor-moscoso-style \\ \hline
director-akira-kurosawa-style & john-everett-millais-style & philippe-druillet-style & video-installation \\ \hline
director-andrei-zvyagintsev-style & john-harris-style & photographer-elena-helfrecht-style & vintage-postage-stamps \\ \hline
director-bong-joon-ho-style & john-james-audubon-style & photographer-flora-borsi-style & weegee-style \\ \hline
director-darren-aronofsky-style & john-kenn-mortensen-style & photographer-maren-klemp-style & wendy-froud-style \\ \hline
director-david-fincher-style & john-martin-style & photographer-martin-kimbell-style & will-eisner-style \\ \hline
director-david-lynch-style & john-singer-sargent-style & photographer-reuben-wu-style & willem-haenraets-style \\ \hline
cute-animals & john-singleton-copley-style & pierre-auguste-renoir-style & willem-van-aelst-style \\ \hline
ben-aronson-style & john-william-waterhouse-style & pierre-bonnard-style & william-langson-lathrop-style \\ \hline
director-emir-kusturica-style & joseph-wright-of-derby-style & pieter-claesz-style & william-mctaggart-style \\ \hline
director-gaspar-noe-style & josh-agle-style & punk-collage & william-merritt-chase-style \\ \hline
director-jean-pierre-jeunet-style & josh-kirby-style & quentin-blake-style & winslow-homer-style \\ \hline
director-krzysztof-kieslowski-style & jules-bastien-lepage-style & raimonds-staprans-style & worthington-whittredge-style \\ \hline
director-martin-scorsese-style & kate-greenaway-style & ralph-bakshi-style & yaacov-agam-style \\ \hline
director-nicolas-winding-refn-style & kay-nielsen-style & ralph-steadman-style & yoh-nagao-style \\ \hline
director-park-chan-wook-style & kilian-eng-style & randolph-caldecott-style & yves-klein-style \\ \hline
director-pedro-almodovar-style & kirigami & ray-caesar-style & zanele-muholi-style \\ \hline
director-quentin-tarantino-style & konstantin-korovin-style & remedios-varo-style &  \\ \hline
\end{longtable}
}

\FloatBarrier

\paragraph{Universal SDXL experiment details} Our second experiment involves the complex and multimodal task of Text-to-Image generation using the Stable Diffusion-XL model~\cite{sdxl}. We extract our low rank universal subspace from publicly available LoRA models on HuggingFace repository~\cite{von-platen-etal-2022-diffusers} - \autoref{tab:diffusion_data} lists all the SDXL models that we used to extract the Universal Subspace. As can be seen in \autoref{tab:diffusion_data}, the models range wildly in styles on which they were finetuned. The fact that all these diverse models can be represented by a single low rank universal subspace model provides strong empirical support for our hypothesis. We use top 16 components and 30 denoising steps. For each experiment model shown in \autoref{tab:clip_sdxl} and \autoref{fig:diffusion}, that LoRA model is reconstructed using a universal subspace created using rest of the available LoRA adapters, essentially confirming the generalization capability of this subspace.

We then use this single SDXL universal subspace to generate images with similar styles to evaluate whether this subspace is capable of doing so, by projecting randomly chosen LoRA models into this subspace. \autoref{fig:diffusion} shows that our universal subspace matches the visual quality and style nuances of individual LoRAs, resulting in significant memory savings. \autoref{tab:clip_sdxl} shows quantitative results for our Universal subspace in terms of CLIP scores, where interestingly we can see that our Universal Subspace outperforms the individual LoRA models. The small improvement may reflect a denoising effect from low-rank spectral
projection, although we do not isolate its cause here - similar to what was observed by~\cite{sharma_laser_2023}.
The styles used in \cref{tab:clip_sdxl}, which are in \cref{tab:diffusion_data} are (from Style 1 to Style 10) Ukiyo-e Style,  Todd Hildo Style , Olly Moss Style , Needlepoint Style , Studio Ghibli Style, Surreal Harmony Style , Dressed Animal Style , Lascaux Cave Art Style , Kirigami Style , Yaacov Agam Style.
\subsection{Low rank shared universal subspaces in classical weights}
\label{sec:classical_apx}
\begin{figure}[htb]
  \centering
  \includegraphics[width=0.9\textwidth]{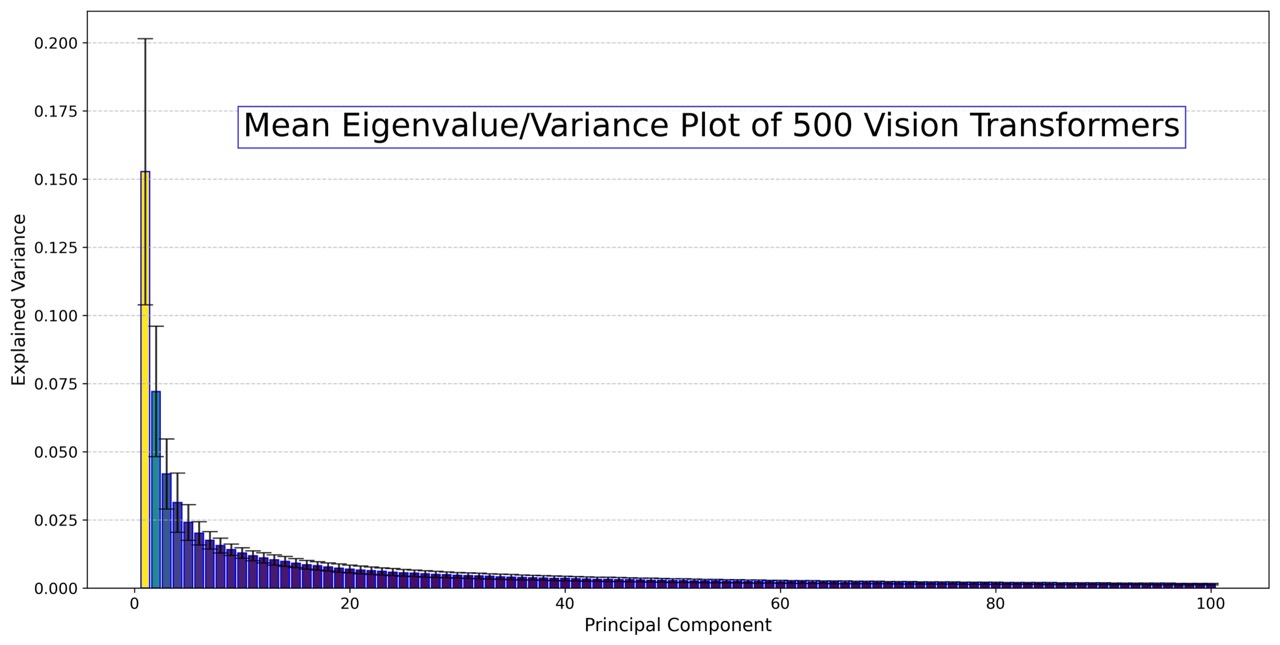}
  \caption{Spectral analysis of the Vision Transformer (ViT-base-patch16-224) model: Aggregated eigenvalue (scree) plot across $~500$ ViT models and all layers. The plot demonstrates that the majority of the variance is consistently captured by the top 16 principal directions, indicating the presence of a shared low-dimensional universal subspace.}
  \label{fig:vit_all}
\end{figure}

We show that this universality does extend beyond adapter models to conventional weights. Convolutional weight parameters have been shown, in limited scope, to match Gabor-like filters~\citep{Krizhevsky2012ImageNet}. Therefore, our analysis extends to these kinds of parameters as well. However, there are a few practical differences between the low rank adapter and classical weight subspace analysis. The classical weight subspace analysis is more computationally expensive relative to the LoRA one due to high dimensionality of the parameters, but in effect, same. Additionally, the number of sufficiently well trained models is understandably fewer than LoRA models. Further, there is also higher variance in terms of model quality in the classical weights as it is harder to optimize these models as compared to LoRA which often are optimized from a good initialization point (the pretrained base model). An outcome of this is that the universal subspace approximation that we obtain from the publicly available pretrained models are noisier than their LoRA counterparts. Despite this, our universal subspace hypothesis remains validated.

To further support our universal subspace hypothesis, we extend our analysis beyond adapter models to standard full-rank weights. We also include convolutional parameters, although they are functionally equivalent to fully connected layers under certain conditions~\citep{ma2017equivalencefullyconnectedlayer}, and their learned representations (e.g., Gabor-like filters) have been studied, in limited scope, in prior work~\citep{Krizhevsky2012ImageNet}. Consequently, our analysis generalizes naturally to convolutional weights as well.

There are, however, practical differences between the subspace analysis of full-rank model weights and that of low-rank adapters. First, analyzing conventional weight matrices is significantly more computationally intensive because of their higher dimensionality. Second, the availability of a large number of independently and sufficiently well-trained models is limited compared to LoRA models. Third, the classical weight models exhibit greater variance in model quality when trained from scratch, as it requires good initialization, optimization, and a sufficient amount and quality of data, unlike LoRA which builds upon a strong pretrained base.

As a result, subspaces estimated from classical weights tend to be noisier
than their LoRA counterparts. Despite this, we consistently observe dominant
shared spectral structure, providing further empirical support for the
Universal Subspace hypothesis.

\paragraph{Transformer suitability diagnostics.}
For the VIT analysis, we compute $\alpha$ directly from the raw weight matrices using the procedure in \cref{sec:metric_apx}. We estimate $\alpha$ for every available model and layer, then compute $q_l$, $q_{\mathrm{all}}$, and $q_{\mathrm{arch}}$ from those estimates. The retained layers are the ones used to calculate $q_{\mathrm{arch}}$.

\begin{table}[h]
    \centering
    \caption{\footnotesize
    Suitability diagnostics for classical-transformer model families.
    Retained layers satisfy the statistical and $q_l$ criteria in
    \cref{sec:metric_apx}.}
    \label{tab:transformer_alpha_apx}

    \small
    \setlength{\tabcolsep}{6pt}
    \begin{tabular}{lccccc}
        \toprule
        \textbf{Family} &
        \textbf{Models} &
        \textbf{Mean $\alpha$} &
        \textbf{$q_{\mathrm{all}}$} &
        \textbf{$q_{\mathrm{arch}}$} &
        \textbf{Retained Layers} \\
        \midrule
        ViT          & 464 & 4.7066 & 0.0518 & 0.3545 & 18 / 40 \\
        GPT-2        & 177 & 4.0951 & 0.0725 & 0.3059 & 23 / 49 \\
        LLaMA-3-8B   & 50  & 5.2391 & 0.0940 & 0.3290 & 116 / 224 \\
        Flan-T5 GLUE & 196 & 3.9136 & 0.0224 & 0.4850 & 179 / 216 \\
        \bottomrule
    \end{tabular}
\end{table}

\begin{figure}[htb]
  \centering
  \begin{subfigure}{0.48\textwidth}
    \includegraphics[width=\linewidth]{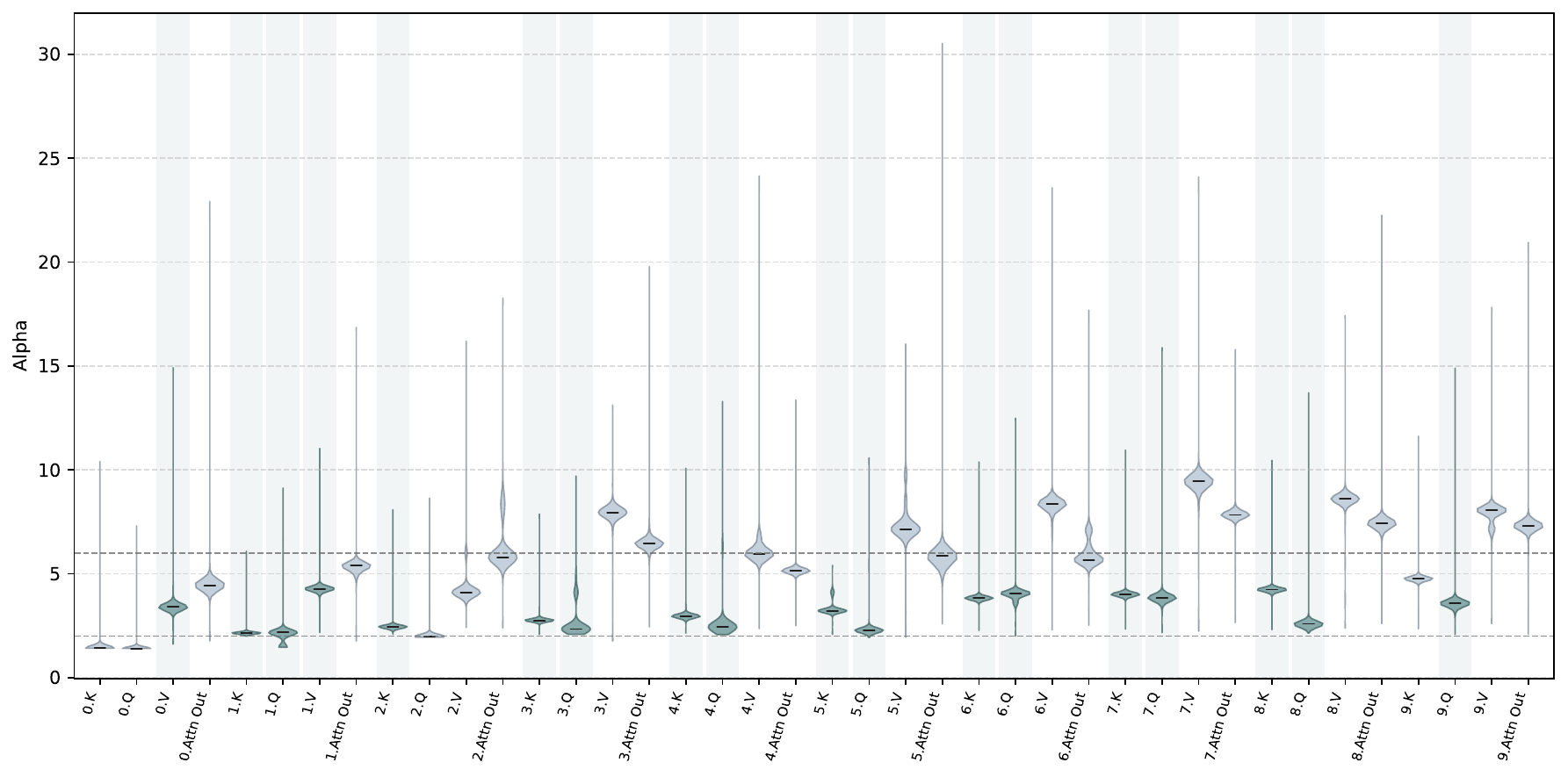}
    \caption{ViT}
  \end{subfigure}
  \hfill
  \begin{subfigure}{0.48\textwidth}
    \includegraphics[width=\linewidth]{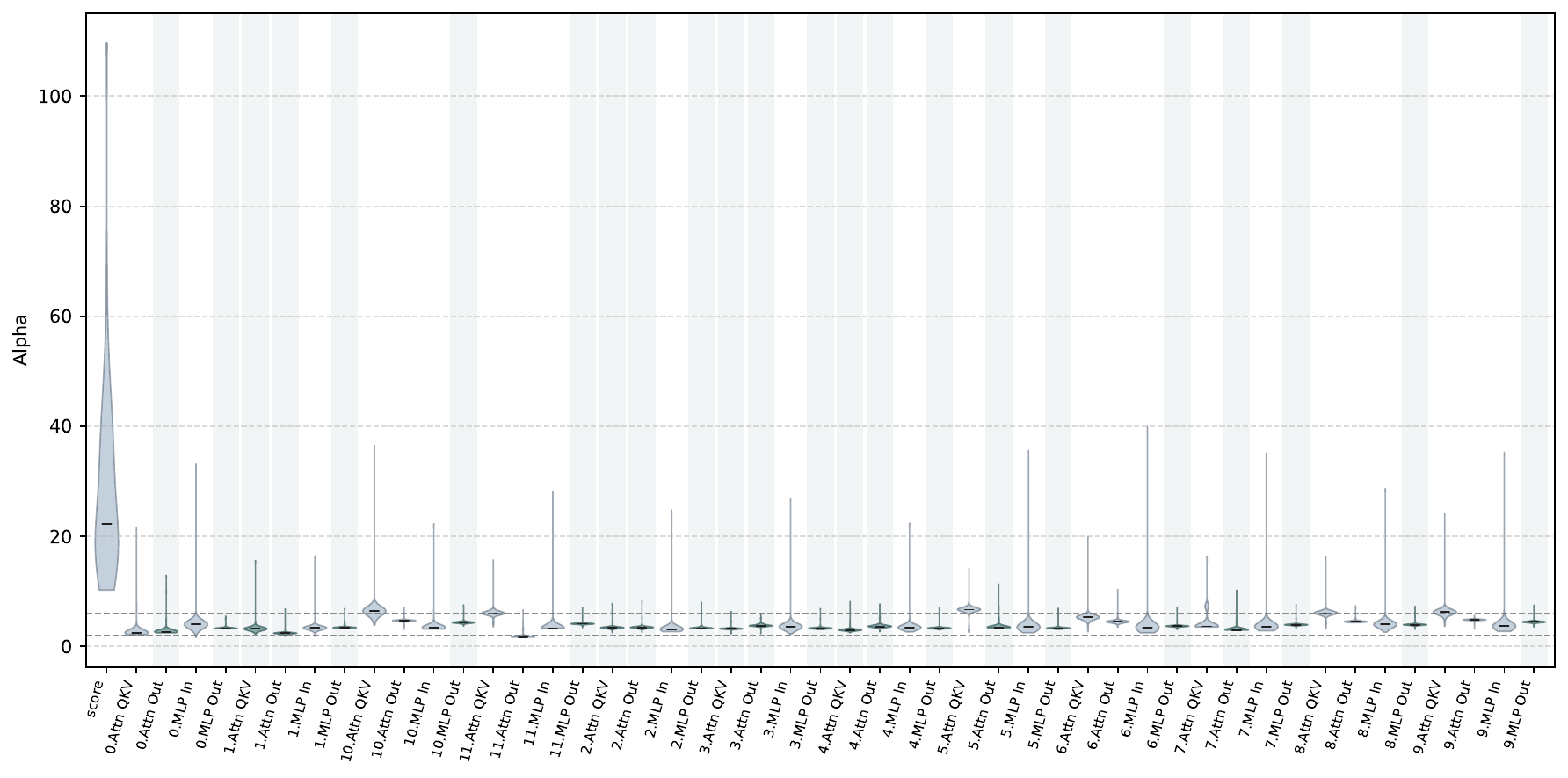}
    \caption{GPT-2}
  \end{subfigure}
  \begin{subfigure}{0.48\textwidth}
    \includegraphics[width=\linewidth]{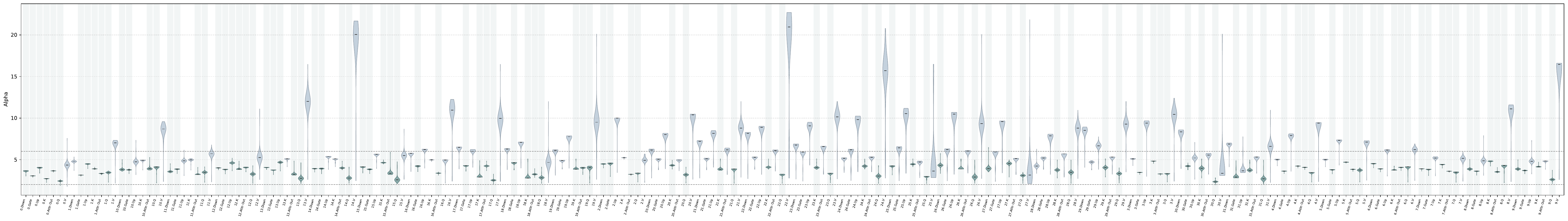}
    \caption{LLaMA-3-8B}
  \end{subfigure}
  \hfill
  \begin{subfigure}{0.48\textwidth}
    \includegraphics[width=\linewidth]{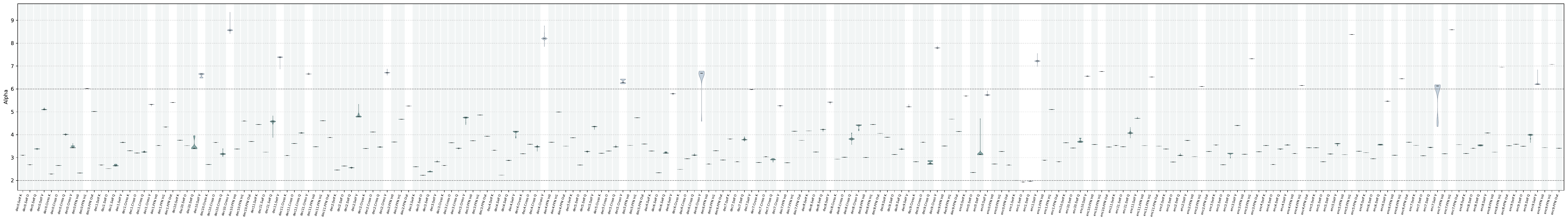}
    \caption{Flan-T5 GLUE}
  \end{subfigure}
  \caption{\footnotesize Per-layer alpha diagnostics for classical-transformer families. Each violin aggregates valid raw-weight alpha estimates over all available models for a fixed layer; retained layers are summarized numerically in \autoref{tab:transformer_alpha_apx}.}
  \label{fig:transformer_alpha_layers_apx}
\end{figure}

\begin{figure}[h]
  \centering
  \fbox{\includegraphics[width=1.0\textwidth]{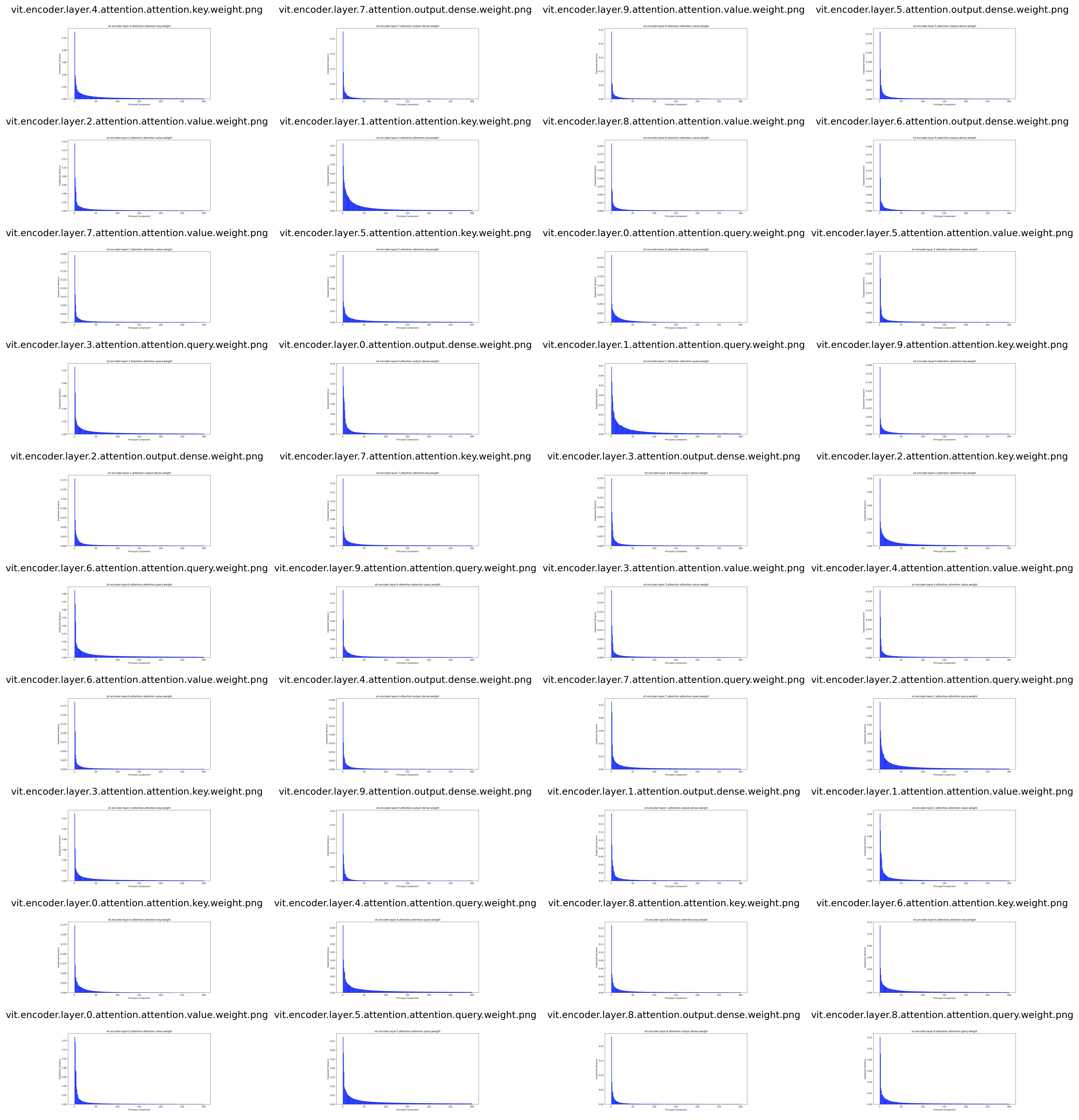}}
  \caption{Layerwise Eigenvalue Plots of 500 ViT models.}
  \label{fig:vit_ind}
\end{figure}

\paragraph{Universal ViT experiment details} We collect $\sim$500 pretrained ViT models from HuggingFace, shown in \autoref{tab:vit_models}, spanning very diverse domains - many of which would be considered orthogonal to one another in terms of domain generalization. These models have been trained with varying losses, optimizers, and initializations. These models were used as-is, without curation or access to training data, to reflect real-world variability. \autoref{fig:vit_all} shows the summarized scree plot for all relevant layers of ViT (sans first and last layers due to differences in shape and tasks) for all $\sim500$ ViT models showing that the majority of variance is captured by the top 16 principal directions, revealing a highly compressible, shared subspace across layers. Only the top 100 components are visualized for clarity, although the available subspace is significantly larger, underlying the sparsity of this universal subspace. We observe this for layerwise analysis in \autoref{fig:vit_ind} as well. For the experimental results presented in \autoref{tab:vit_Res}, we randomly choose 4-5 IID and 4-5 OOD models from \autoref{tab:vit_models} for which evaluation dataset is available, and reconstruct these model weights by projecting them into our universal subspace. For the OOD case, we ensure that the models being evaluated are not present in the subset used for creating the universal subspace approximation. As seen from the results, our extremely sparse subspace model performs competitively compared to the fully trained versions. It is likely that with more careful choice of principal directions per layer would allow for at par or even better performance.
{\tiny  
\begin{longtable}{|p{0.4\textwidth}|p{0.4\textwidth}|}
\caption{Finetuned Models from HuggingFace used for the Universal Vision Transformer subspace extraction (vit-base-patch16-224)}
\label{tab:vit_models} \\
\hline
0.50-200Train-100Test-vit-base & 2025-01-21-16-13-04-vit-base-patch16-224 \\
\hline
2025-02-05-14-22-36-vit-base-patch16-224 & 21BAI1229 \\
\hline
Accomodation\_room\_classification & adam\_VitB-p16-224-1e-4-batch\_16\_epoch\_4\_classes\_24 \\
\hline
age\_face\_detection\_base & AIvisionGuard-v2 \\
\hline
alea & amns \\
\hline
AnimeCharacterClassifierMark1 & autotrain-48ci8-roib9 \\
\hline
autotrain-8oqr6-image0807-20 & autotrain-ap-pass-fail-v1 \\
\hline
autotrain-g2g80-iwcfm & autotrain-google-vit-13epoch \\
\hline
autotrain-ht4es-gbvmt & autotrain-image-classifier-cats-and-dogs \\
\hline
autotrain-pknu0-o76h9 & autotrain-s0sds-erede \\
\hline
autotrain-test-image-classification & autotrain-vit-base-patch16-224-fog-or-smog-classification \\
\hline
beauty-ornot & beer-classifier \\
\hline
bg-classif & bigger-chord-finetuned \\
\hline
brain-tumor-44 & ButterflyClasifModel \\
\hline
camera-type & Caracam \\
\hline
cards-vit-base-patch16-224-finetuned-v1 & carmodel \\
\hline
cats123 & cats-dogs-2024 \\
\hline
cats-dogs-classification & CheXpert-ViT-U-MultiClass \\
\hline
CheXpert-ViT-U-SelfTrained & chord-final-model \\
\hline
chord\_ViT-finetuned & cifar10-lt \\
\hline
city\_multiclass\_classification & clasificador\_masas \\
\hline
corals\_binary\_classification & custom \\
\hline
detect\_meme & dog-breeds-classification \\
\hline
dog-cat-demo-20240815 & dog-cats-model \\
\hline
dummy\_classification\_model & dvm-cars-vit-first-5k \\
\hline
ecg-image-multilabel-classification & emotion \\
\hline
EmotionAgeModel & emotion\_model \\
\hline
emotion-recognition & emotion\_recognition \\
\hline
emotion\_recognition\_results & emotion-vit \\
\hline
face\_age\_detection\_base\_v2 & face\_age\_detection\_base\_v3\_weighted \\
\hline
final-run & finetune-cats \\
\hline
fine-tuned & finetuned-amazon \\
\hline
fine-tuned-augmented & finetuned-bin \\
\hline
finetuned-cifar10 & finetuned-indian-food \\
\hline
fine-tuned-model & finetuned\_model \\
\hline
Fine-Tuned\_Model & Fine-Tuned\_Model2 \\
\hline
Fine-Tuned\_Model3 & Fine-Tuned\_Model3\_Transfer\_learning \\
\hline
finetune-vit-base-patch16-224 & finetune\_vit\_base\_patch16\_224\_1epoch \\
\hline
Flowers & food \\
\hline
food-101-finetuned-model & Freshness-Fruit\_Vegies \\
\hline
frost-vision-v2-google\_vit-base-patch16-224 & frost-vision-v2-google\_vit-base-patch16-224-v2024-11-09 \\
\hline
frost-vision-v2-google\_vit-base-patch16-224-v2024-11-11 & frost-vision-v2-google\_vit-base-patch16-224-v2024-11-14 \\
\hline
fruit\_classification & fruits-360-16-7 \\
\hline
ft\_stable\_diffusion & gender \\
\hline
giecom-vit-model-clasification-waste & google-vit-base-patch16-224-batch32-lr0.0005-standford-dogs \\
\hline
google-vit-base-patch16-224-batch32-lr0.005-standford-dogs & google-vit-base-patch16-224-batch32-lr5e-05-standford-dogs \\
\hline
google-vit-base-patch16-224-batch64-lr0.005-standford-dogs & google-vit-base-patch16-224-OrganicAndInorganicWaste-classification \\
\hline
google-vit-base-patch16-224-Waste-O-I-classification & hf\_vit\_format\_hap\_pretrained\_256\_128 \\
\hline
Human-Action-Recognition-VIT-Base-patch16-224 & human-actions \\
\hline
image-classification & image\_classification \\
\hline
image\_strawbery-peach\_classifier & isa-vit\_model \\
\hline
lixg\_food\_model001 & Maggi-Parle-G\_Classifier \\
\hline
mammals\_multiclass\_classification & MemeDetector \\
\hline
model & Model \\
\hline
model-vit-base-finetuned & MRI\_vit \\
\hline
my\_chest\_xray\_model & myclass \\
\hline
my\_classification & MyPetModel \\
\hline
out & outputs \\
\hline
PagesClassificationModel & physiotheraphy-E2 \\
\hline
plant\_disease\_detection-beans & pokemon\_classification \\
\hline
pokemon\_model & pokemon-vit \\
\hline
recaptcha & recycled\_waste\_classification \\
\hline
results & rmsprop\_VitB-p16-224-1e-4-batch\_16\_epoch\_4\_classes\_24 \\
\hline
rmsprop\_VitB-p16-224-2e-4-batch\_16\_epoch\_4\_classes\_24 & road-conditions \\
\hline
rose\_recognition & rotated2 \\
\hline
Ruster & S1\_M1\_R1\_vit\_42498800 \\
\hline
S1\_M1\_R1\_vit\_42509509 & S1\_M1\_R1\_ViT\_42616100 \\
\hline
S1\_M1\_R2\_vit\_42498972 & S1\_M1\_R2\_ViT\_42618476 \\
\hline
S1\_M1\_R3\_vit\_42499444 & S1\_M1\_R3\_ViT\_42618486 \\
\hline
S2\_M1\_R1\_vit\_42499480 & S2\_M1\_R1\_ViT\_42618522 \\
\hline
S2\_M1\_R2\_vit\_42499499 & S2\_M1\_R2\_ViT\_42618530 \\
\hline
S2\_M1\_R3\_vit\_42499514 & S2\_M1\_R3\_ViT\_42618549 \\
\hline
S5\_M1\_fold1\_vit\_42499955 & S5\_M1\_fold1\_ViT\_42618571 \\
\hline
S5\_M1\_fold2\_vit\_42499968 & S5\_M1\_fold2\_ViT\_42618583 \\
\hline
S5\_M1\_fold3\_vit\_42499983 & S5\_M1\_fold3\_ViT\_42618589 \\
\hline
S5\_M1\_fold4\_vit\_42499997 & S5\_M1\_fold4\_ViT\_42618593 \\
\hline
S5\_M1\_fold5\_vit\_42500027 & S5\_M1\_fold5\_ViT\_42621111 \\
\hline
Screenshots\_detection\_to\_classification & sign-lan-model \\
\hline
square\_run\_32\_batch & square\_run\_age\_gender \\
\hline
square\_run\_first\_vote\_full\_pic\_50 & square\_run\_first\_vote\_full\_pic\_50\_age\_gender \\
\hline
square\_run\_first\_vote\_full\_pic\_75 & square\_run\_first\_vote\_full\_pic\_75\_age\_gender \\
\hline
square\_run\_second\_vote & square\_run\_second\_vote\_full\_pic\_50 \\
\hline
square\_run\_second\_vote\_full\_pic\_50\_age\_gender & square\_run\_second\_vote\_full\_pic\_75 \\
\hline
square\_run\_second\_vote\_full\_pic\_75\_age\_gender & square\_run\_second\_vote\_full\_pic\_age\_gender \\
\hline
square\_run\_second\_vote\_full\_pic\_stratified & square\_run\_square\_run\_first\_vote\_full\_pic\_25 \\
\hline
square\_run\_square\_run\_first\_vote\_full\_pic\_25\_age & square\_run\_square\_run\_first\_vote\_full\_pic \\
\hline
square\_run\_square\_run\_first\_vote\_full\_pic\_25 & square\_run\_square\_run\_second\_vote\_full\_pic\_25 \\
\hline
square\_run\_square\_run\_second\_vote..nder & square\_run\_with\_16\_batch\_size \\
\hline
square\_run\_with\_actual\_16\_batch\_size & stool-condition-classification \\
\hline
swaddling-classifier & swin-tiny-patch4-window7-224-finetuned-eurosat-kornia \\
\hline
tarread & telidermai \\
\hline
test-cifar-10 & traffic-levels-image-classification \\
\hline
Train-Augmentation-vit-base & trainer\_output \\
\hline
Train-Test-Augmentation-V3D-vit-base & UL\_base\_classification \\
\hline
UL\_bedroom\_classification & UL\_exterior\_classification \\
\hline
UL\_interior\_classification & vehicle\_multiclass\_classification \\
\hline
ViT\_ASVspoof\_DF & vit-augmentation \\
\hline
vit-b16-plant\_village & vit\_base \\
\hline
vit-base-1e-4-15ep & vit-base-1e-4-20ep \\
\hline
vit-base-1e-4-randaug & vit-base-1stGen-Pokemon-Images \\
\hline
vit-base-25ep & Vit-Base-30VN \\
\hline
vit-base-3e-5-randaug & vit-base-5e-4 \\
\hline
vit-base-add-2-decay & vit-base-augment \\
\hline
vit-base-batch-32 & vit-base-beans \\
\hline
vit-base-brain-mri & vit-base-cat\_or\_dog \\
\hline
vit-base-change-arg & vit-base-cocoa \\
\hline
ViT-Base-Document-Classifier & vit-base-fashion \\
\hline
vit-base-finetuned-cephalometric & vit-base-food101 \\
\hline
vit-base-fruits-360 & vit-base-hate-meme \\
\hline
vit-base-nationality & vit-base-org-plot \\
\hline
vit-base-oxford-brain-tumor & vit-base-oxford-brain-tumor\_try\_stuff \\
\hline
vit-base-oxford-brain-tumor\_x-ray & vit-base-oxford-iiit-pets \\
\hline
vit-base-oxford-pets-krasuluk & vit-base-patch16-224 \\
\hline
vit-base-patch16-224-13\_model & vit-base-patch16-224-30-vit \\
\hline
vit-base-patch16-224-9models & vit-base-patch16-224-abhi1-finetuned \\
\hline
vit-base-patch16-224\_augmented-v2\_fft & vit-base-patch16-224\_augmented-v2\_tl \\
\hline
vit-base-patch16-224-blur\_vs\_clean & vit-base-patch16-224-brand \\
\hline
vit-base-patch16-224-classifier & vit-base-patch16-224-clothes-filter \\
\hline
vit-base-patch16-224-cl-v1 & vit-base-patch16-224-crochets-clothes-classification \\
\hline
vit-base-patch16-224-Diastar & vit-base-patch16-224-Diastarallclasses \\
\hline
vit-base-patch16-224-dmae-va-U & vit-base-patch16-224-dmae-va-U5-100-iN \\
\hline
vit-base-patch16-224-dmae-va-U5-10-45-5e-05 & vit-base-patch16-224-dmae-va-U5-20-45-5e-05 \\
\hline
vit-base-patch16-224-dmae-va-U5-40-45-5e-05 & vit-base-patch16-224-dmae-va-U5-42B \\
\hline
vit-base-patch16-224-dmae-va-U5-42C & vit-base-patch16-224-dmae-va-U5-42D \\
\hline
vit-base-patch16-224-ethos & vit-base-patch16-224-ethos-25 \\
\hline
vit-base-patch16-224-ethos-8 & vit-base-patch16-224-ethos-data \\
\hline
vit-base-patch16-224-ethosrealdata & vit-base-patch16-224-fatigue \\
\hline
vit-base-patch16-224-finalterm & vit-base-patch16-224-finetuned \\
\hline
vit-base-patch16-224-finetuned-barkley & vit-base-patch16-224-finetuned-brain-tumor-classification \\
\hline
vit-base-patch16-224-finetuned-Brain-Tumor-Classification & vit-base-patch16-224-finetuned-cassava-leaf-disease \\
\hline
vit-base-patch16-224-finetuned-cedar & vit-base-patch16-224-finetuned-cifar10 \\
\hline
vit-base-patch16-224-finetuned-combinedSpiders & vit-base-patch16-224-finetuned-context-classifier \\
\hline
vit-base-patch16-224-finetuned-covid\_ct\_set\_full & vit-base-patch16-224-finetuned-covid\_ct\_set\_resumed \\
\hline
vit-base-patch16-224-finetuned-crochets-clothes & vit-base-patch16-224-finetuned-dangerousSpiders \\
\hline
vit-base-patch16-224-finetuned-eurosat & vit-base-patch16-224-finetuned-feature-maps-v3 \\
\hline
vit-base-patch16-224-finetuned-feature-map-v2 & vit-base-patch16-224-finetuned-fibre \\
\hline
vit-base-patch16-224-finetuned-flower & vit-base-patch16-224-finetuned-flower-classify \\
\hline
vit-base-patch16-224-finetuned-flowers & vit-base-patch16-224-finetuned-food101 \\
\hline
vit-base-patch16-224-finetuned-food102 & vit-base-patch16-224-finetuned-foveated-features \\
\hline
vit-base-patch16-224-finetuned-foveated-features-v2 & vit-base-patch16-224-finetuned-galaxy10-decals \\
\hline
vit-base-patch16-224-finetuned-hateful-meme-restructured & vit-base-patch16-224-finetuned-hateful-meme-restructured-balanced \\
\hline
vit-base-patch16-224-finetuned-imagegpt & vit-base-patch16-224-finetuned-ind-17-imbalanced-aadhaarmask \\
\hline
vit-base-patch16-224-finetuned-ind-17-imbalanced-aadhaarmask-new-parameter & vit-base-patch16-224-finetuned-landscape-test \\
\hline
vit-base-patch16-224-finetuned-lora-oxford-pets & vit-base-patch16-224-finetuned-masked-hateful-meme-restructured \\
\hline
vit-base-patch16-224-finetuned-noh & vit-base-patch16-224-finetuned-original-images \\
\hline
vit-base-patch16-224-finetuned-pneumonia-detection & vit-base-patch16-224-finetuned-polyterrasse \\
\hline
vit-base-patch16-224-finetuned-skin & vit\_base\_patch16\_224-finetuned-SkinDisease \\
\hline
vit-base-patch16-224-finetuned-teeth\_dataset & vit-base-patch16-224-finetuned-trash-classifications-albumentations \\
\hline
vit-base-patch16-224-finetuned-turquoise & vit-base-patch16-224-finetuned-Visual-Emotional \\
\hline
vit-base-patch16-224-finetuned-vit & vit-base-patch16-224-finetune\_test \\
\hline
vit-base-patch16-224-food101-16-7 & vit-base-patch16-224-food101-24-12 \\
\hline
vit-base-patch16-224-for-pre\_evaluation & vit-base-patch16-224-fruits-360-16-7 \\
\hline
vit-base-patch16-224-high-vit & vit-base-patch16-224-jvadlamudi2 \\
\hline
vit-base-patch16-224-masaratti & vit-base-patch16-224-mascotas \\
\hline
vit-base-patch16-224-mascotas-DA & vit-base-patch16-224-MSC-dmae \\
\hline
vit-base-patch16-224-newly-trained & vit-base-patch16-224-oxford-pets-classification \\
\hline
vit-base-patch16-224-perros-y-gatos & vit-base-patch16-224-pure-ViT \\
\hline
vit-base-patch16-224-R1-10 & vit-base-patch16-224-R1-40 \\
\hline
vit-base-patch16-224-Rado\_5 & vit-base-patch16-224\_rice-disease-02 \\
\hline
vit-base-patch16-224\_rice-leaf-disease-augmented\_fft & vit-base-patch16-224\_rice-leaf-disease-augmented\_tl \\
\hline
vit-base-patch16-224\_rice-leaf-disease-augmented-v4\_fft & vit-base-patch16-224\_rice-leaf-disease-augmented-v4\_tl \\
\hline
vit-base-patch16-224\_rice-leaf-disease-augmented-v4\_v5\_fft & vit-base-patch16-224\_rice-leaf-disease-augmented-v4\_v5\_pft \\
\hline
vit-base-patch16-224-rotated-dungeons-v101 & vit-base-patch16-224-rotated-dungeons-v103 \\
\hline
vit-base-patch16-224-RU2-10 & vit-base-patch16-224-RU2-40 \\
\hline
vit-base-patch16-224-RU3-10 & vit-base-patch16-224-RU3-40 \\
\hline
vit-base-patch16-224-RU4-10 & vit-base-patch16-224-RU4-40 \\
\hline
vit-base-patch16-224-RU5-10 & vit-base-patch16-224-RU5-10-8 \\
\hline
vit-base-patch16-224-RU5-40 & vit-base-patch16-224-RU9-24 \\
\hline
vit-base-patch16-224-RX1-24 & vit-base-patch16-224-RX2-12 \\
\hline
vit-base-patch16-224-RXL1-24 & vit-base-patch16-224-type \\
\hline
vit-base-patch16-224-U6-10 & vit-base-patch16-224-U7-10 \\
\hline
vit-base-patch16-224-U8-10 & vit-base-patch16-224-U8-10b \\
\hline
vit-base-patch16-224-U8-10c & vit-base-patch16-224-U8-40 \\
\hline
vit-base-patch16-224-U8-40b & vit-base-patch16-224-U8-40c \\
\hline
vit-base-patch16-224-U8-40d & vit-base-patch16-224-ve-b-U10-12 \\
\hline
vit-base-patch16-224-ve-b-U10-24 & vit-base-patch16-224-ve-b-U10-40 \\
\hline
vit-base-patch16-224-ve-U10-12 & vit-base-patch16-224-ve-U10-24 \\
\hline
vit-base-patch16-224-ve-U11-12 & vit-base-patch16-224-ve-U11-b-24 \\
\hline
vit-base-patch16-224-ve-U11-b-40 & vit-base-patch16-224-ve-U11-b-80 \\
\hline
vit-base-patch16-224-ve-U12-b-24 & vit-base-patch16-224-ve-U12-b-80 \\
\hline
vit-base-patch16-224-ve-U13-b-120 & vit-base-patch16-224-ve-U13-b-24 \\
\hline
vit-base-patch16-224-ve-U13-b-80 & vit-base-patch16-224-ve-U13b-80R \\
\hline
vit-base-patch16-224-ve-U13b-80RX & vit-base-patch16-224-ve-U13b-80RX1 \\
\hline
vit-base-patch16-224-ve-U13b-80RX3 & vit-base-patch16-224-ve-U13b-R \\
\hline
vit-base-patch16-224-ve-U14-b-24 & vit-base-patch16-224-ve-U15-b-80 \\
\hline
vit-base-patch16-224-ve-U16-b-80 & vit-base-patch16-224-ve-Ub \\
\hline
vit-base-patch16-224-vit & vit-base-patch16-224-vit-base-patch16-224-vit-base-patch16-224-dogORnot \\
\hline
vit-base-pets & vit-base-PICAI \\
\hline
vit-base-seed-1e-4 & vit-base-seed-3e-4 \\
\hline
vit-base-travel-document-classification & vit-base-v1-eval-epoch-maxgrad-decay-cosine \\
\hline
vit-beans-classifier & vit-beta1-0.85 \\
\hline
vit-beta1-0.88 & vit-beta1-0.95 \\
\hline
vit-beta2-0.99 & vit-beta2-0.995 \\
\hline
vit-beta2-0.9995 & vit-bird \\
\hline
ViT\_bloodmnist & ViT\_bloodmnist\_std\_0 \\
\hline
ViT\_bloodmnist\_std\_15 & ViT\_bloodmnist\_std\_30 \\
\hline
ViT\_bloodmnist\_std\_45 & ViT\_bloodmnist\_std\_60 \\
\hline
ViT\_breastmnist & ViT\_breastmnist\_std\_0 \\
\hline
ViT\_breastmnist\_std\_15 & ViT\_breastmnist\_std\_30 \\
\hline
ViT\_breastmnist\_std\_45 & ViT\_breastmnist\_std\_60 \\
\hline
VIT-cats-vs-dogs & vit-cifar10-fine-tuned \\
\hline
vit-class-weight & vit-cxr4 \\
\hline
vit-demo & ViT\_dog\_food \\
\hline
vit-dropout-0.2 & vit-dropout-0.3 \\
\hline
vit-dropout-0.4 & vit-dropout-0.5 \\
\hline
vit-ds-processed & vit-emotion-model \\
\hline
vit-epsilon-1e-7 & vit-epsilon-1e-9 \\
\hline
vit-epsilon-5e-9 & vit-face-project-piyush \\
\hline
vit-fine-tune-classification-cats-vs-dogs & vit-finetuned-1 \\
\hline
vit-food-classification-chrisis2 & vit-geometric-shapes-base \\
\hline
vit-google-model-30-classes & vit\_google\_vehicle\_classification\_model \\
\hline
vit-historical-page & vit\_Liveness\_detection\_v1.0 \\
\hline
vit-lr-0.0001 & vit-lr-0.001 \\
\hline
vit-lr-0.01 & vit-lr-cosine-restarts \\
\hline
vit-lr-cosine-warm-restarts & vit-lr-cosine-warmup \\
\hline
vit-lr-exponential & vit-lr-inverse-sqrt \\
\hline
vit-lr-linear & vit-lr-poly \\
\hline
vit-lr-reduce-plateau & vit-lr-step \\
\hline
vit-mae-base-finetuned-eurosat & vit-molecul \\
\hline
vit-ori-dataset-exp & vit-plant-classification \\
\hline
vit-plantnet300k & vit-plants \\
\hline
vit-real-fake-classification-v1 & vit-real-fake-classification-v2 \\
\hline
vit-real-fake-classification-v3 & vit-real-fake-classification-v4 \\
\hline
vit-skin-demo-v1 & vit-skin-demo-v2 \\
\hline
vit-skin-demo-v3 & vit-skin-demo-v4 \\
\hline
vit-skin-demo-v5 & vit-spam \\
\hline
vit-sports-cls & vit-transfer-learning \\
\hline
vit\_transformer\_eye\_disease & vit\_tumor\_classifier \\
\hline
vit-vit & vit-vit-base-patch16-224-finetuned-chest\\
\hline
vit-weight-decay-1e-2 & vit-weight-decay-1e-3 \\
\hline
vit-weight-decay-1e-4 & vit-weight-decay-1e-5 \\
\hline
wmc\_v2\_vit\_base\_wm811k\_cls\_contra\_learning\_0916 & wmc\_v2\_vit\_base\_wm811k\_cls\\
\hline
wmc-wmk811-v0-vit-special\_map\_det\_0917 & WS800\_ViT\_42820348 \\
\hline
WS800\_ViT\_42895082 & xraynewww \\
\hline
yet-another-amber-mines & zdravJEM\_CV\_BERT \\
\hline
\end{longtable}}

\FloatBarrier
\paragraph{Universal LLaMA3-8B Experiment Details}  
To further stress-test our universal subspace hypothesis on classical weight matrices, we extract a shared subspace from approximately 50 finetuned LLaMA3 models, each with 8 billion parameters. These models were obtained from publicly available repositories on HuggingFace. Due to their scale, we do not apply any model selection or filtering, and instead include the entire available set - we may also use compatible layers from slightly different model architectures.

As shown in \autoref{fig:llama_all}, which presents the aggregated scree plot across all layers and all 50 models, the principal variance is concentrated in the top few components—consistent with the emergence of a low-rank universal subspace. For reference, the plot displays only the top 300 components, which represent a small fraction of the full rank, highlighting the inherently low-dimensional structure.

The models included in this analysis span a diverse range of domains, including medical applications, multilingual dialogue systems, and general-purpose assistants, as listed in \autoref{tab:llama_data}. To the best of our knowledge, this is the first work to demonstrate that such a large and heterogeneous collection of high-capacity language models can be jointly represented within a single low-rank subspace.

The layerwise spectral analysis, shown in \autoref{fig:llama_ind}, corroborates this finding: across all layers, the majority of eigenvalues fall below a threshold of \( < 0.001 \), indicating that most directions in parameter space contribute negligibly to variation across models. The plots are cropped to show only the leading components due to the large number of total dimensions. We recommend zooming in for clearer visualization.


\begin{figure}[h]
  \centering
  \includegraphics[width=1.0\textwidth]{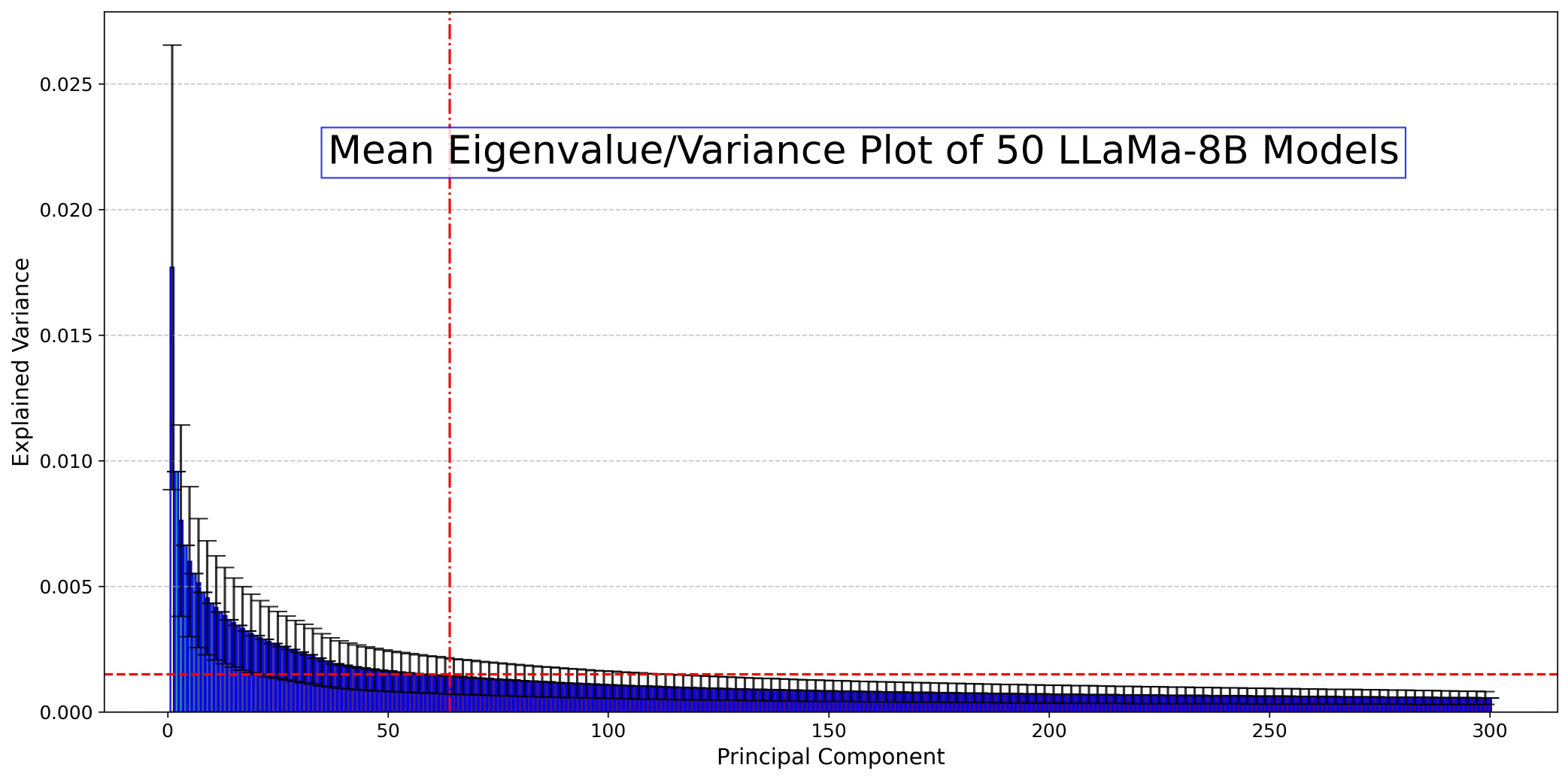}
  \caption{Spectral analysis of 50 LLaMA-3-8B model: Aggregated eigenvalue (scree) plot across $50$ LlaMa-8B models and all layers. The plot demonstrates that the majority of the variance is consistently captured by few top principal directions, indicating the presence of a shared low-dimensional universal subspace.}
  \label{fig:llama_all}
\end{figure}

\begin{figure}[htb]
  \centering
  \fbox{\includegraphics[width=0.9\textwidth,height=0.82\textheight,keepaspectratio]{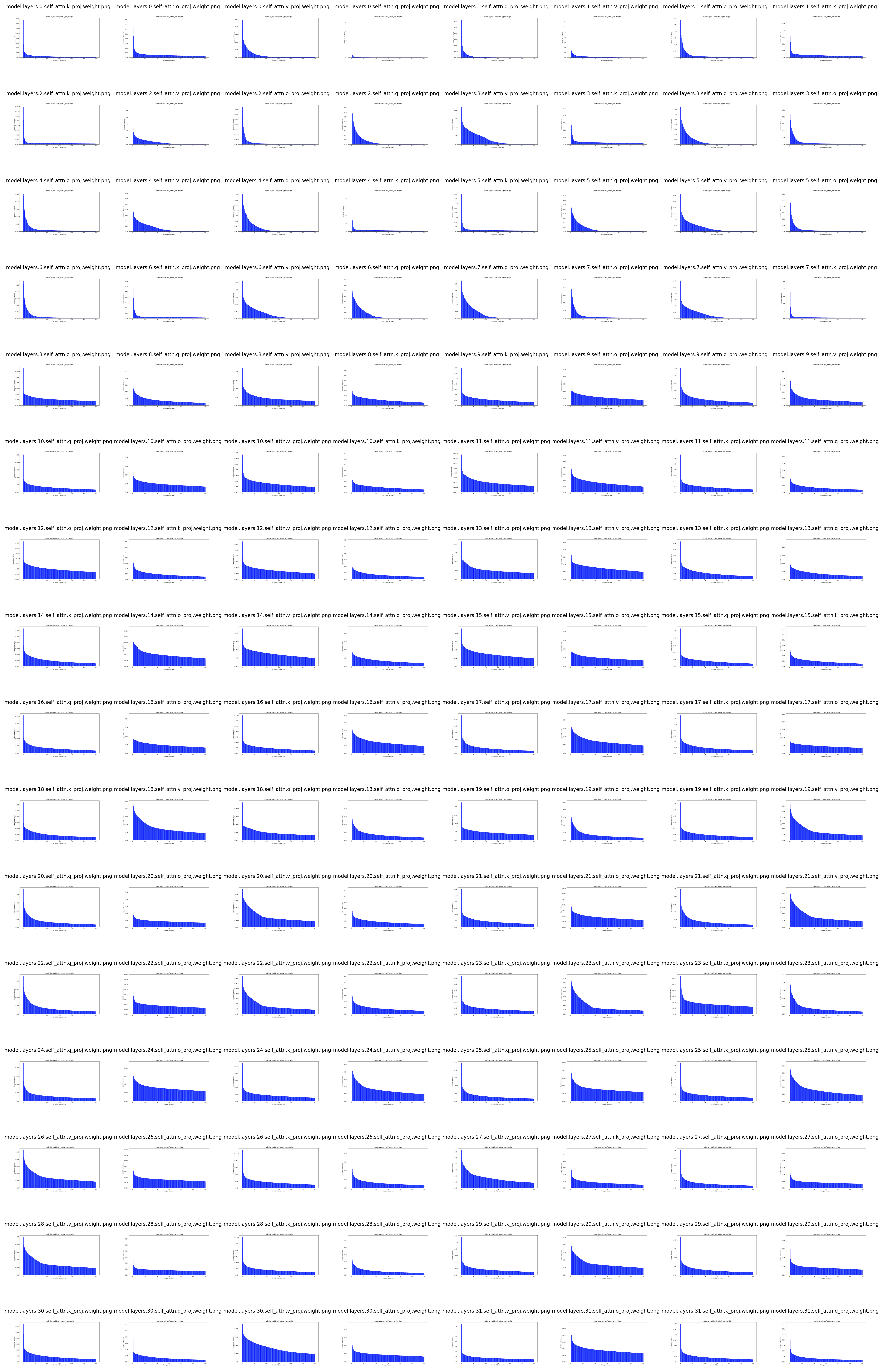}}
  \caption{Layerwise Scree Plots for 50 LLaMA-3-8B Models. For enhanced clarity, each subplot presents a truncated view of the total possible principal directions. These plots consistently demonstrate that the dominant information, as represented by explained variance, resides within a small number of leading principal directions for all models. Components beyond this initial set are characterized by eigenvalues approaching zero, signifying their redundancy for the universal subspace.}
  \label{fig:llama_ind}
\end{figure}
\begin{table}[h]
\caption{Models from HuggingFace used for the Universal LlaMa3-8B subspace extraction}
\label{tab:llama_data}
\resizebox{\textwidth}{!}{%
\begin{tabular}{|l|l|l|l|}
\hline
Meta-Llama-3-8B-Instruct-Jailbroken & Llama-3-13B-Instruct & large\_crafting\_sft\_success & suzume-llama-3-8B-multilingual \\
\hline
summary-llama3-8b-f16-full & Llama-3-13B-Instruct-v0.1 & Llama-3-8B-ProLong-64k-Base & LLaMAntino-3-ANITA-8B-Inst-DPO-ITA \\
\hline
ai-medical-model-32bit & filtered\_crafting\_train\_data\_shorter\_length & Llama-3-portuguese-Tom-cat-8b-instruct & Llama-3-MAAL-8B-Instruct-v0.1 \\
\hline
Human-Like-LLama3-8B-Instruct & LLaMA-3-8B-Instruct-TR-DPO & CabraLlama3-8b & chartgpt-llama3 \\
\hline
KoLlama-3-8B-Instruct & honeypot-llama3-8B & Llama-SEA-LION-v2-8B & TR \\
\hline
Llama3-8B-Instruct-Turkish-Finetuned & Llama-3-15B-Instruct-zeroed & Llama-3-8B-Instruct-TAR-Bio-v2 & Bio-Medical-Llama-3-8B \\
\hline
filtered\_construction\_train\_data & shisa-v1-llama3-8b & REFUEL-Llama-3-Armo-iter\_1 & llama3-instrucTrans-enko-8b \\
\hline
Llama-3-8B-Instruct-Ja & llama3-passthrough-chat & RoLlama3-8b-Instruct & Lloro-SQL \\
\hline
Summary\_L3\_1000steps\_1e7rate\_SFT2 & CyberSentinel & Meta-Llama-3-8B-Instruct-function-calling-json-mode & MARS \\
\hline
Llama-3-8B-Instruct-Finance-RAG & LLaMA3-Instruct-8B-FR-Spec & Llama-3-8B-Japanese-Instruct & Llama3-8B-Chinese-Chat \\
\hline
llama-3-chinese-8b-instruct-v2 & Athene-RM-8B & Llama-3-OffsetBias-RM-8B & large\_cooking\_sft\_success \\
\hline
suzume-llama-3-8B-japanese & llama-3-chinese-8b-instruct-v3 & Waktaverse-Llama-3-KO-8B-Instruct & llama-3-8b-gpt-4o-ru1.0 \\
\hline
Llama-3-Aplite-Instruct-4x8B-MoE & Llama-3-8B-Instruct-DPO-v0.3 &  &  \\
\hline
\end{tabular}
}
\end{table}


\paragraph{Universal GPT-2 Experiment Details}
We analyze 177 public GPT-2 checkpoints over 49 matrix layers, listed in \autoref{tab:gpt2_models}. These checkpoints cover multilingual, summarization, classification, and domain-specialized finetunes, making them a heterogeneous full-weight language-model pool. We use about 200 models for \cref{fig:teaser}.
{\tiny
\begin{longtable}{|p{0.45\textwidth}|p{0.45\textwidth}|}
\caption{Models from HuggingFace used for the GPT-2 subspace extraction}
\label{tab:gpt2_models} \\
\hline
openai-community/gpt2 & pierreguillou/gpt2-small-portuguese \\
\hline
ai-forever/rugpt3small\_based\_on\_gpt2 & uer/gpt2-chinese-cluecorpussmall \\
\hline
axtonyao/gpt2-fp16-tflite & Xenova/gpt2 \\
\hline
tensorblock/rakeshkiriyath\_gpt2Medium\_text\_to\_sql-GGUF & 13on/gpt2-wishes \\
\hline
13onn/gpt2-wishes-2 & 13onnn/gpt2-wish \\
\hline
3koozy/gpt2-HxH & 850886470/xxy\_gpt2\_chinese \\
\hline
AAli/gpt2-wikitext2 & ARCYVILK/gpt2-bot \\
\hline
Akash7897/gpt2-wikitext2 & Aleksandar1932/gpt2-country \\
\hline
Aleksandar1932/gpt2-hip-hop & Aleksandar1932/gpt2-pop \\
\hline
Aleksandar1932/gpt2-rock-124439808 & Aleksandar1932/gpt2-soul \\
\hline
Aleksandar1932/gpt2-spanish-classics & Azaghast/GPT2-SCP-ContainmentProcedures \\
\hline
Azaghast/GPT2-SCP-Descriptions & Azaghast/GPT2-SCP-Miscellaneous \\
\hline
BenDavis71/GPT-2-Finetuning-AIRaid & BhanuSama/gpt2-finetuned-xsum \\
\hline
BigSalmon/GPT2HardandEasy & Chaewon/mnmt\_decoder\_en\_gpt2 \\
\hline
Chakita/gpt2\_mwp & CoffeeAddict93/gpt2-call-of-the-wild \\
\hline
CoffeeAddict93/gpt2-modest-proposal & ComCom/gpt2 \\
\hline
ComCom-Dev/gpt2-bible-test & DHBaek/gpt2-stackoverflow-question-contents-generator \\
\hline
DeepESP/gpt2-spanish & Enes3774/gpt2 \\
\hline
Eren/gpt-2-small-the-office & Eunooeh/mnmt\_gpt2 \\
\hline
Fabby/gpt2-english-light-novel-titles & Ferch423/gpt2-small-portuguese-wikipediabio \\
\hline
Finnish-NLP/gpt2-finnish & Forest/gpt2-fanfic \\
\hline
GHP/gpt2-fine-tuned & Galuh/id-journal-gpt2 \\
\hline
Gantenbein/ADDI-CH-GPT2 & Gantenbein/ADDI-DE-GPT2 \\
\hline
Gantenbein/ADDI-FI-GPT2 & Gantenbein/ADDI-FR-GPT2 \\
\hline
Gantenbein/ADDI-IT-GPT2 & Graphcore/gpt2-medium-ipu \\
\hline
Graphcore/gpt2-small-ipu & GroNLP/gpt2-small-dutch-embeddings \\
\hline
GroNLP/gpt2-small-dutch & GroNLP/gpt2-small-italian-embeddings \\
\hline
GroNLP/gpt2-small-italian & HamidRezaAttar/gpt2-product-description-generator \\
\hline
HooshvareLab/gpt2-fa-comment & HooshvareLab/gpt2-fa-poetry \\
\hline
HooshvareLab/gpt2-fa & Iacopo/Shakespear-GPT2 \\
\hline
Inkdrop/gpt2-property-classifier & JaidevShriram/gpt2-wikitext2 \\
\hline
JasonCheung/gpt2 & Jipski/gpt2-Flo-BasBoettcher-Chefkoch \\
\hline
Jipski/gpt2-Flo-BasBoettcher & Jipski/gpt2-FloSolo \\
\hline
Kup/gpt2-wikitext2 & Langame/gpt2-starter-2 \\
\hline
Langame/gpt2-starter & Langame/gpt2-waiting \\
\hline
Luciano/gpt2-small-portuguese-finetuned-peticoes & Luciano/gpt2-small-portuguese-finetuned-tcu-acordaos \\
\hline
MYX4567/gpt2-wikitext2 & Madhour/gpt2-eli5 \\
\hline
Mamatha/agri-gpt2 & Mary222/GPT2\_RU\_GAME \\
\hline
Mary222/GPT2\_standard & Masdfasdg/GPT2-Prompt \\
\hline
Micky/gpt2-wikitext2 & MickyMike/0-GPT2SP-appceleratorstudio \\
\hline
MickyMike/0-GPT2SP-aptanastudio & MickyMike/0-GPT2SP-bamboo \\
\hline
MickyMike/0-GPT2SP-clover & MickyMike/0-GPT2SP-datamanagement \\
\hline
MickyMike/0-GPT2SP-duracloud & MickyMike/0-GPT2SP-jirasoftware \\
\hline
MickyMike/0-GPT2SP-mesos & MickyMike/0-GPT2SP-moodle \\
\hline
MickyMike/0-GPT2SP-mule & MickyMike/0-GPT2SP-mulestudio \\
\hline
MickyMike/0-GPT2SP-springxd & MickyMike/0-GPT2SP-talenddataquality \\
\hline
MickyMike/0-GPT2SP-talendesb & MickyMike/0-GPT2SP-titanium \\
\hline
MickyMike/0-GPT2SP-usergrid & MickyMike/00-GPT2SP-appceleratorstudio-aptanastudio \\
\hline
MickyMike/00-GPT2SP-appceleratorstudio-titanium & MickyMike/00-GPT2SP-aptanastudio-titanium \\
\hline
MickyMike/00-GPT2SP-mesos-usergrid & MickyMike/00-GPT2SP-mule-mulestudio \\
\hline
MickyMike/00-GPT2SP-mulestudio-mule & MickyMike/00-GPT2SP-titanium-appceleratorstudio \\
\hline
MickyMike/00-GPT2SP-usergrid-mesos & MickyMike/000-GPT2SP-appceleratorstudio-mule \\
\hline
MickyMike/000-GPT2SP-appceleratorstudio-mulestudio & MickyMike/000-GPT2SP-clover-usergrid \\
\hline
MickyMike/000-GPT2SP-mule-titanium & MickyMike/000-GPT2SP-mulestudio-titanium \\
\hline
MickyMike/000-GPT2SP-talenddataquality-appceleratorstudio & MickyMike/000-GPT2SP-talenddataquality-aptanastudio \\
\hline
MickyMike/000-GPT2SP-talendesb-mesos & MickyMike/1-GPT2SP-appceleratorstudio \\
\hline
MickyMike/1-GPT2SP-aptanastudio & MickyMike/1-GPT2SP-bamboo \\
\hline
MickyMike/1-GPT2SP-clover & MickyMike/1-GPT2SP-datamanagement \\
\hline
MickyMike/1-GPT2SP-duracloud & MickyMike/1-GPT2SP-jirasoftware \\
\hline
MickyMike/1-GPT2SP-mesos & MickyMike/1-GPT2SP-moodle \\
\hline
MickyMike/1-GPT2SP-mule & MickyMike/1-GPT2SP-mulestudio \\
\hline
MickyMike/1-GPT2SP-springxd & MickyMike/1-GPT2SP-talenddataquality \\
\hline
MickyMike/1-GPT2SP-talendesb & MickyMike/1-GPT2SP-titanium \\
\hline
MickyMike/1-GPT2SP-usergrid & MickyMike/11-GPT2SP-appceleratorstudio-aptanastudio \\
\hline
MickyMike/11-GPT2SP-appceleratorstudio-titanium & MickyMike/11-GPT2SP-aptanastudio-titanium \\
\hline
MickyMike/11-GPT2SP-mesos-usergrid & MickyMike/11-GPT2SP-mule-mulestudio \\
\hline
MickyMike/11-GPT2SP-mulestudio-mule & MickyMike/11-GPT2SP-titanium-appceleratorstudio \\
\hline
MickyMike/11-GPT2SP-usergrid-mesos & MickyMike/111-GPT2SP-appceleratorstudio-mule \\
\hline
MickyMike/111-GPT2SP-appceleratorstudio-mulestudio & MickyMike/111-GPT2SP-clover-usergrid \\
\hline
MickyMike/111-GPT2SP-mule-titanium & MickyMike/111-GPT2SP-mulestudio-titanium \\
\hline
MickyMike/111-GPT2SP-talenddataquality-appceleratorstudio & MickyMike/111-GPT2SP-talenddataquality-aptanastudio \\
\hline
MickyMike/111-GPT2SP-talendesb-mesos & MickyMike/2-GPT2SP-appceleratorstudio \\
\hline
MickyMike/2-GPT2SP-aptanastudio & MickyMike/2-GPT2SP-bamboo \\
\hline
MickyMike/2-GPT2SP-clover & MickyMike/2-GPT2SP-datamanagement \\
\hline
MickyMike/2-GPT2SP-duracloud & MickyMike/2-GPT2SP-jirasoftware \\
\hline
MickyMike/2-GPT2SP-mesos & MickyMike/2-GPT2SP-moodle \\
\hline
MickyMike/2-GPT2SP-mule & MickyMike/2-GPT2SP-mulestudio \\
\hline
MickyMike/2-GPT2SP-springxd & MickyMike/2-GPT2SP-talenddataquality \\
\hline
MickyMike/2-GPT2SP-talendesb & MickyMike/2-GPT2SP-titanium \\
\hline
MickyMike/2-GPT2SP-usergrid & MickyMike/22-GPT2SP-appceleratorstudio-aptanastudio \\
\hline
MickyMike/22-GPT2SP-appceleratorstudio-titanium & MickyMike/22-GPT2SP-aptanastudio-titanium \\
\hline
MickyMike/22-GPT2SP-mesos-usergrid & MickyMike/22-GPT2SP-mule-mulestudio \\
\hline
MickyMike/22-GPT2SP-mulestudio-mule & MickyMike/22-GPT2SP-titanium-appceleratorstudio \\
\hline
MickyMike/22-GPT2SP-usergrid-mesos & MickyMike/222-GPT2SP-appceleratorstudio-mule \\
\hline
MickyMike/222-GPT2SP-appceleratorstudio-mulestudio & MickyMike/222-GPT2SP-clover-usergrid \\
\hline
MickyMike/222-GPT2SP-mule-titanium & MickyMike/222-GPT2SP-mulestudio-titanium \\
\hline
MickyMike/222-GPT2SP-talenddataquality-appceleratorstudio & MickyMike/222-GPT2SP-talenddataquality-aptanastudio \\
\hline
MickyMike/222-GPT2SP-talendesb-mesos & MickyMike/6-GPT2SP-appceleratorstudio \\
\hline
MickyMike/6-GPT2SP-aptanastudio & MickyMike/6-GPT2SP-bamboo \\
\hline
MickyMike/6-GPT2SP-clover & MickyMike/6-GPT2SP-datamanagement \\
\hline
MickyMike/6-GPT2SP-duracloud & MickyMike/6-GPT2SP-jirasoftware \\
\hline
MickyMike/6-GPT2SP-mesos & MickyMike/6-GPT2SP-moodle \\
\hline
MickyMike/6-GPT2SP-mule & MickyMike/6-GPT2SP-mulestudio \\
\hline
MickyMike/6-GPT2SP-springxd & MickyMike/6-GPT2SP-talenddataquality \\
\hline
MickyMike/6-GPT2SP-talendesb & MickyMike/6-GPT2SP-titanium \\
\hline
MickyMike/6-GPT2SP-usergrid & MickyMike/66-GPT2SP-appceleratorstudio-aptanastudio \\
\hline
MickyMike/66-GPT2SP-appceleratorstudio-titanium & MickyMike/66-GPT2SP-aptanastudio-titanium \\
\hline
MickyMike/66-GPT2SP-mesos-usergrid & MickyMike/66-GPT2SP-mule-mulestudio \\
\hline
MickyMike/66-GPT2SP-mulestudio-mule & MickyMike/66-GPT2SP-titanium-appceleratorstudio \\
\hline
MickyMike/66-GPT2SP-usergrid-mesos & MickyMike/666-GPT2SP-appceleratorstudio-mule \\
\hline
\end{longtable}}

\paragraph{Universal Flan-T5 Experiment Details} We collected Flan-T5 models fine-tuned on individual datasets from the GLUE~\citep{glue} benchmark. 
The alpha diagnostic is strongest among the transformer families, retaining 179 layers with $q_{\mathrm{all}}=0.0224$ and $q_{\mathrm{arch}}=0.4850$. The 90\% explained-variance setting yields 85.13\% savings on analyzed matrix tensors, showing that the shared subspace remains compact even across task-specialized sequence-to-sequence finetunes. The suitability diagnostics use 196 Flan-T5 models, while the subspace analysis in \cref{fig:teaser} uses the eight models listed below.

{\tiny  
\begin{longtable}{|p{0.4\textwidth}|p{0.4\textwidth}|}
\caption{Finetuned Flan-T5 Models from HuggingFace used for the Universal Flan-T5 subspace extraction}
\label{tab:flan_models} \\
\hline
tanganke/flan-t5-base\_glue-cola & tanganke/flan-t5-base\_glue-mnli \\
\hline
tanganke/flan-t5-base\_glue-mrpc & tanganke/flan-t5-base\_glue-qnli \\
\hline
tanganke/flan-t5-base\_glue-rte & tanganke/flan-t5-base\_glue-qqp \\
\hline
tanganke/flan-t5-base\_glue-sst2 & tanganke/flan-t5-base\_glue-stsb \\
\hline
\end{longtable}}

\subsection{Ablating number of models and subspace effectiveness}
Although this is implicitly addressed through our large-scale experiments (500 ViTs, 500 Mistral-7B and 300 Stable Diffusion LoRAs, 50 LLaMA3-8B, 177 GPT-2s, Flan-T5, and ResNet50 models) in all Figures and Tables, which demonstrate consistent behavior at different scales. \cref{thm:twolevel} provides insight into the saturation dynamics: the finite-task estimation error decreases at the standard $O(T^{-1/2})$ rate, up to logarithmic factors, where $T$ is the number of available tasks. Thus, increasing the number of models improves the estimated shared subspace, with diminishing returns as $T$ grows. In practice, the minimum number of models per architecture needed to achieve saturation point depends on the quality of the trained models, the diversity of data they have been trained on, and on the architecture itself. Ablating these would require access to all the data for all the models, and very careful training on every training for each data, and then running permutation with all possible combinations of models. All of this is out of reach for most researchers simply due to time, data and compute constraints. We, however, do provide an initial ablation here. For LoRA models shown in \cref{tab:mistral_models}, we choose 9 random (OOD) tasks (39	,190	,280	,290	,391	,442	,1342	,1391	,1598) and extract the Universal Subspace from rest of the tasks, sampled randomly for increasing number of models. The coefficients for OOD tasks are analytically reconstructed to effectively evaluate the universal subspace created from varying number of models. \cref{tab:ablate_nmodel} shows that the adequate principal components are quickly extracted, and increasing the number of models has diminishing returns.

\begin{table}
    \centering
    \caption{Lots of LoRAs (Mistral-7B) OOD evaluation per increasing number of models used to extract Universal Subspace}
    \label{tab:ablate_nmodel}
    \small
    \begin{tabular}{@{}lcc@{}}
        \toprule
        \textbf{Method} & \textbf{Model Number} & \textbf{Rouge-L Score} \\
        \midrule
        Normal Model & -  & 73.7 \\
        Universal model & 50  & 55.8 \\
        Universal model & 150  & 66.1 \\
        Universal model & 250  & 71.9 \\
        Universal model & 450  & 72.3 \\
        \bottomrule
    \end{tabular}
\end{table}

\FloatBarrier
\clearpage

\section{Appendix C}\label{sec:appendix_c}
\begin{table}[h]
    \centering
    \caption{\footnotesize
    Per-task results for eight ViT-B/32 models, each fine-tuned with LoRA on a
    different image-classification dataset. ``Finetuned'' reports the absolute
    accuracy of each task-specific model. For merging methods, we report
    per-task accuracy normalized by the corresponding fine-tuned model, together
    with the average across tasks.}
    \label{tab:merge_accuracies}

    \resizebox{\textwidth}{!}{%
    \begin{tabular}{lccccccccc}
        \toprule
        \textbf{Method} &
        \multicolumn{8}{c}{\textbf{Datasets}} &
        \textbf{Avg.} \\
        \cmidrule(lr){2-9}
        &
        \textbf{Cars} &
        \textbf{DTD} &
        \textbf{EuroSAT} &
        \textbf{GTSRB} &
        \textbf{MNIST} &
        \textbf{RESISC45} &
        \textbf{SUN397} &
        \textbf{SVHN} & \\
        \midrule

        \multicolumn{10}{c}{\textbf{Per-Task Absolute Accuracies (\%)}} \\
        \cmidrule(lr){1-10}
        Finetuned
        & 74.0 & 58.3 & 99.0 & 92.7 & 99.3 & 88.4 & 64.5 & 96.2 & 84.1 \\
        \midrule

        \multicolumn{10}{c}{
        \textbf{Combined-Model Accuracy Normalized by Finetuned Accuracy (\%)}} \\
        \cmidrule(lr){1-10}

        RegMean
        & 80.2 & 71.3 & 37.9 & 47.3 & 43.1 & 70.5 & 93.9 & 43.0 & 60.9 \\

        TA
        & 82.0 & 73.6 & 48.8 & 42.1 & 53.1 & 71.5 & 97.5 & 41.2 & 63.7 \\

        TIES
        & 82.2 & 72.8 & 50.0 & 36.8 & 56.8 & 69.4 & 96.9 & 44.6 & 63.7 \\

        DARE-TIES
        & 81.4 & 74.5 & 50.8 & 39.2 & 55.0 & 70.7 & 97.6 & 40.1 & 63.7 \\

        KnOTS-TIES
        & 82.7 & 73.7 & 49.3 & 48.9 & 68.9 & 70.9 & 95.5 & 53.8 & 68.0 \\

        KnOTS-DARE-TIES
        & 81.8 & 75.9 & 50.7 & 40.3 & 53.2 & 70.2 & 97.9 & 41.0 & 63.9 \\

        \textbf{Ours}
        & \textbf{88.1}
        & \textbf{82.3}
        & \textbf{65.9}
        & \textbf{61.3}
        & \textbf{88.3}
        & \textbf{98.1}
        & \textbf{98.5}
        & \textbf{85.1}
        & \textbf{83.5} \\

        \bottomrule
    \end{tabular}%
    }
\end{table}

\subsection{Related Work}
\label{sec:related_work}
Several lines of prior research support the core intuition behind our universal subspace hypothesis, though they do not provide a unified, scalable framework for identifying and leveraging such subspaces across architectures, tasks, and modalities. The Neural Tangent Kernel framework reinforces this idea, in the infinite-width limit, training dynamics are governed by a limiting kernel determined by the architecture and initialization distribution~\citep{jacot2018ntk}. This result implies that training is implicitly constrained to a shared function space, suggesting the existence of low-dimensional structures that generalize across tasks. Complementing this, works in mechanistic interpretability has uncovered modular and recurring patterns that consistently re-emerge in independently trained models~\citep{olah2020zoom, chughtai2023toymodeluniversalityreverse}, supporting the notion of structural universality in network representations. 
The closest related work is~\citep{kaushik2025eigenloraxrecyclingadaptersprincipal}, which investigates principal subspaces in low rank adapters for highly parameter efficient finetuning, and is an application of the universal subspace hypothesis presented in this work. Relatedly, \citet{han2026flame} report that the functional energy of trained mixture-of-experts layers concentrates in a sharply low-rank subspace despite near-full-rank weights, and use this to add new tasks as low-rank slices of a fixed expert pool for continual multimodal multi-task learning. 
 
While \cite{Mao_2024} study low-dimensional structure in prediction space, we examine shared subspaces in model weights. Importantly, in contrast to the abstract, weak, or speculative notions of universality presented in earlier works like~\cite{chughtai2023toymodeluniversalityreverse, olah2020zoom} - which focus primarily on representations and are arguably easier to demonstrate due to direct data dependency - ours is the first work to provide concrete evidence and a clear universal hypothesis at the \textit{neural parameter} level.

Empirical studies further strengthen this perspective. The lottery ticket hypothesis~\citep{frankle2019lottery} demonstrates that overparameterized networks contain sparse subnetworks capable of matching full-model performance, implying that task-relevant information resides in a small, structured subset of weights. Similarly, mode connectivity studies~\citep{garipov2018mode} reveal that seemingly isolated optima in parameter space are often connected by low-loss paths. In convolutional models, Krizhevsky et al.~\citep{Krizhevsky2012ImageNet} famously observed that early layers consistently learn Gabor-like filters, indicating a universal inductive bias in early representations. More recent works~\citep{guth24, guth2024universalityneuralencodingscnns} extends this observation to deeper layers, showing that certain eigendirections of trained convolutional layers recur across networks trained on different datasets.

While these studies are suggestive of shared structures in neural representations or parameters, they remain limited in their focus, application and analysis.
Our work fills this critical gap by presenting a principled and empirically validated method for discovering and utilizing universal parametric subspaces that span across architectures, tasks, and modalities. By conducting large-scale spectral analyses of over large number of diverse architectures, models and tasks, we demonstrate that a small number of principal directions consistently capture the majority of task-relevant variation. We then operationalize these findings by developing a practical framework for reusing these subspaces for parameter-efficient finetuning, task adaptation, and model merging, achieving competitive performance while dramatically reducing memory and compute requirements.


\subsection{Model Merging and Reuse via Universal Subspaces}
We present this experiment as an illustration of model merging
through shared subspaces and task-specific reuse. We extract common
layer-wise bases and retain lightweight coefficients for each task,
computed by projection. Evaluation uses the corresponding task's
coefficients; these are not averaged or combined. This shared-basis
representation is related to
\citet{brüelgabrielsson2024compressserveservingthousands}.
Unlike the weight-merging baselines, our approach retains task-specific
coefficients rather than producing one fixed set of weights for all tasks. Despite retaining task-specific coefficients, our shared representation can be smaller than a baseline merged model or even a single original model. A lightweight gating network or parameter could be trained to select the appropriate coefficients at inference, similar to the task-specific routers that select low-rank expert slices in \citet{han2026flame}; we leave this extension to future work.

\subsection{Universal Subspace Analysis}
\label{sec:analysis_apx}
Similar methodology is followed for subspace analysis for both LoRA and classical weight models. The same spectral methodology applies to both LoRA parameters and conventional
weight tensors after aligning parameter shapes across models. We nevertheless
evaluate both settings independently because their parameterizations and
training mechanisms differ substantially. However, in order to solidify our universal subspace hypothesis, we conduct extensive experiments for both types of models. LoRA is chosen because of the recent spurt in the availability of LoRA models trained on diverse kinds of datasets and models. \hl{We do this universal subspace analysis on all weight parameters in every neural network layer except the first (or few initial) and last neural network layer. This is because these layers may differ across models due to differences in input shapes and types, loss functions, and the tasks being trained.} We also focus our analysis on linear/fully-connected and matrix weights, as well as convolutional weights, as the analysis done on these are straightforward and the results observed can be extended to other types of neural parameters~\citep{ma2017equivalencefullyconnectedlayer}.

\paragraph{Secondary Subspace} refers to the residual subspace that remains after removing the top \( k \) principal directions associated with the low-rank universal subspace. This subspace is orthogonal to the universal subspace and serves as a control for evaluating the uniqueness and effectiveness of the learned shared subspace. To make computation tractable when the residual subspace is high-dimensional, we focus on the top components beyond rank \( k \), as computing a full SVD is often impractical. This approximation is justified, since the lower components typically capture noise, which has been shown to degrade performance~\citep{sharma_laser_2023}.

\paragraph{How to choose top $k$ components?} As shown in all eigenvalue (scree) plots, a trivial way to choose is a simple visual inspection, since we can see a discontinuity in the spectral analysis. Another way is to define a threshold on the explained variance, all components whose explained variance is close to zero <.01 are considered secondary subspace, and can be discarded. Singular-value thresholds derived for matrix denoising can also motivate heuristic rank selection for the matrix unfoldings used in HOSVD~\citep{gavish2014optimalhardthresholdsingular}.

\subsection{Finding universal subspaces and applying them to future tasks}
\label{sec:newtask_apx}
In this section, we present two tasks, GLUE~\citep{glue} and Image Classification. For each experiment, the joint subspace is created using all other models in subset. For Image Classification, we use $k=4$ and train only 8 epochs using learning rate of 1e-4. Importantly, only the coefficients are trained for the experiment. It is important to note that our shared subspace model performs quite well despite using very few (4-5) models to extract the subspace. For GLUE, we use 16-32 components for our subspace, with learning rate of 4e-4, batch size of 64, and 30-80 epochs for each task. In addition, it is likely that our model might perform similarly or better if trained longer or with optimized hyperparameters.
\paragraph{Compute Resources} We conduct all our experiments using a single A5000 GPU, and a CPU with 8 workers. For the universal subspace extraction, all calculation can be done on the CPU. However, GPU would increase the speed of calculation as the layerwise subspace extraction can be parallelized. Storage comparisons use same precision and for Universal Subspace model variants count the shared bases once together with all model-specific coefficients, excluding task-specific first and last layers from both representations. 

\subsection{Discussion and Broader Impact}
Our findings suggest that deep neural networks trained across diverse tasks and modalities systematically converge to shared, low-dimensional subspaces within their parameter space. The existence of such universal subspaces challenges conventional assumptions about the independence and diversity of model and task-specific finetuning trajectories. Instead, it highlights a powerful regularity in the way deep models encode task-specific knowledge - one that can be exploited for significantly improved training and deployment efficiency. By leveraging these subspaces, we demonstrate that models can be adapted to new tasks by learning only a small number of coefficients, rather than retraining or storing full sets of weights. This facilitates more robust multi-task learning, model merging, and scalable fine-tuning, with theoretical guarantees and empirical validation across multiple architectures.

The broader societal impact of this work is substantial. Our approach enables large-scale models to be reused and extended with dramatically reduced computational overhead, addressing both the financial and environmental costs associated with training and deploying deep learning systems. This contributes directly to the goals of sustainable and accessible AI. By lowering the hardware and energy requirements for adaptation and inference, we empower under-resourced researchers, institutions, and communities to build upon state-of-the-art models without needing extensive compute infrastructure. Furthermore, by supporting modular model design and data-free model merging, our work lays the foundation for more interpretable, maintainable, and equitable AI systems.

\subsection{Limitations and Future Work}
\label{sec:limitation}
Although we provide conclusive results towards the existence and utility of universal shared subspaces, the current analysis has scope for future research, such as limited interpretability of the shared subspace and the corresponding directions. While it is a critical area of research, it is extremely cumbersome to demonstrate interpretability of the principal directions for each layer of the network. To the best of our knowledge we are not aware of any other literature that performs such an in-depth analysis of the weight space of large models across diverse tasks, data modalities and model architectures. The current approach to approximating a universal subspace relies on pretrained task-specific models (predictors) for tasks, which may not be readily available for new tasks. An interesting direction for future research would be to explore model independent methods for learning a universal shared subspace, potentially derived directly from data. Furthermore, the conditions proposed in \cite{ortiz-jimenez2023taskarithmetic} relate task arithmetic to localized kernel eigenfunctions. How these conditions interact with our shared weight-subspace formulation is unclear and warrants further investigation. As a result, performing task arithmetic within the current framework of a shared universal subspace is non-trivial and warrants further investigation. Our analysis also assumes that tasks are drawn i.i.d.\ from a fixed task distribution. Whether a shared subspace remains adequate under distribution shift, for example when novel classes appear while the distribution of known classes also changes~\citep{chaudhari2026openset}, is an open question.


\end{document}